%% file: neurips_2026.tex
\documentclass{article}

\PassOptionsToPackage{numbers, compress}{natbib}
\PassOptionsToPackage{table}{xcolor}
\usepackage[preprint]{neurips_2026}
\usepackage{amsmath}
\usepackage{makecell}
\usepackage{pifont} \newcommand{\cmark}{\ding{51}} \newcommand{\xmark}{\ding{55}}
\newcommand{\pmark}{{\footnotesize$\sim$}}
\usepackage{amsthm}
   \usepackage{thmtools}
   \usepackage{thm-restate}
\usepackage{amssymb}

   \newtheorem{corollary}{Corollary}
   \newtheorem{lemma}{Lemma}
   \newtheorem{theorem}{Theorem}
   
\usepackage{enumitem}
\usepackage{subcaption}
\usepackage{CJKutf8} %
\newcommand{\zh}[1]{\begin{CJK}{UTF8}{gbsn}#1\end{CJK}}
\usepackage{booktabs, multirow, adjustbox, xcolor,natbib}
\usepackage[utf8]{inputenc} %
\usepackage[T1]{fontenc}    %
\usepackage{hyperref}       %
\usepackage{url}            %
\usepackage{amsfonts}       %
\usepackage{nicefrac}       %
\usepackage{microtype}      %
\usepackage{xcolor}         %
\usepackage{wrapfig}

\title{Steering Without Breaking: Mechanistically Informed Interventions for Discrete Diffusion Language Models}

\author{%
  Hanhan Zhou$^{*}$\\
  AWS AI Labs\\
  Santa Clara, CA 95054 \\
  \texttt{hanhanz@amazon.com} \\
  \And
  Shamik Roy$^{*}$ \\
  AWS AI Labs \\
  Santa Clara, CA 95054 \\
  \texttt{royshami@amazon.com} \\
  \AND
  Rashmi Gangadharaiah \\
  AWS AI Labs \\
  Santa Clara, CA 95054 \\
  \texttt{rgangad@amazon.com} \\
}

\begin{document}

\maketitle
\begingroup
\renewcommand\thefootnote{*}
\footnotetext{Equal contribution.}
\endgroup

\begin{abstract}
Discrete diffusion language models (DLMs) generate text by iteratively denoising all positions in parallel, offering an alternative to autoregressive models. Controlled generation methods for DLMs, imported from autoregressive models, apply uniform intervention at every denoising step. We show this uniform schedule is inefficient and degrades quality, and the damage compounds when multiple attributes are steered jointly. To diagnose the failure, we train sparse autoencoders on four DLMs (124M-8B parameters) and find that different attributes commit on distinct schedules, varying in timing, sharpness, and magnitude. For instance, topic commits within the first 2\% of denoising on MDLM, whereas sentiment emerges gradually over 20\% of the process. 
Motivated by these profiles, we propose an adaptive scheduling mechanism that concentrates intervention where each attribute is actively forming. An idealized allocation analysis predicts that attributes with more sharply concentrated emergence benefit more from adaptive scheduling, a prediction we confirm empirically. Across seven single- and multi-attribute steering tasks on four DLMs, adaptive steering consistently improves the control-quality balance over uniform and interval-restricted baselines.
\end{abstract}

\input{1_intro}

\input{3_related_works}

\input{4_method}
\input{5_interpretability_2}

\input{6_experiments}

\input{7_conclusion_limitations}

{
\small

\bibliography{neurips_bib}
\bibliographystyle{unsrtnat}
}

\newpage
\input{appendix}

\newpage

\end{document}

%% file: 1_intro.tex
\section{Introduction}

Discrete diffusion language models (DLMs) generate text by iteratively denoising corrupted token sequences and have recently become competitive with autoregressive models~\citep{austin2021structured, sahoo2024mdlm, lou2024sedd, dream2025, nie2025llada}. Unlike autoregressive generation, where each token is produced once in fixed left-to-right order, DLMs refine all positions in parallel across hundreds of steps, creating a temporal trajectory where semantic content is progressively committed. This raises a fundamental question: how do DLMs organize semantic attributes across denoising, and can that temporal structure be exploited for controlled generation?

Sparse autoencoders (SAEs) decompose dense activations into sparse, interpretable features~\citep{cunningham2023sparse, bricken2023monosemanticity, templeton2024scaling} and have been used to steer autoregressive models~\citep{goyal2025breaking, yeon2025gsae, cho2025corrsteer, arad2025saes}. Concurrent work explores applying SAEs to DLMs~\citep{wang2026dlmscope} and steering DLMs via reference-sequence alignment~\citep{avrahami2026ilrr}. Other approaches from autoregressive models, including contrastive vectors~\citep{zou2023representation, rimsky-etal-2024-steering} and probes~\citep{belinkov-2022-probing}, can also be adapted to DLMs. All existing methods, however, apply the same intervention at every denoising step, treating the trajectory as temporally uniform. This degrades generation quality in both single- and multi-attribute control, inflating perplexity and collapsing diversity~\citep{gu2022improving}, while increasing cross-attribute interference, the unintended shifting of non-target attributes~\citep{gu2022distributional, oozeer2025ksteering, shafran2026directions}, as we verify in \S\ref{sec:experiments}. Yet how attributes form and commit across denoising steps in DLMs remains underexplored, leaving open whether temporally informed intervention can avoid this quality cost.

\begin{figure}[t!]
\centering
\makebox[\textwidth][c]{%
  \includegraphics[width=1\textwidth]{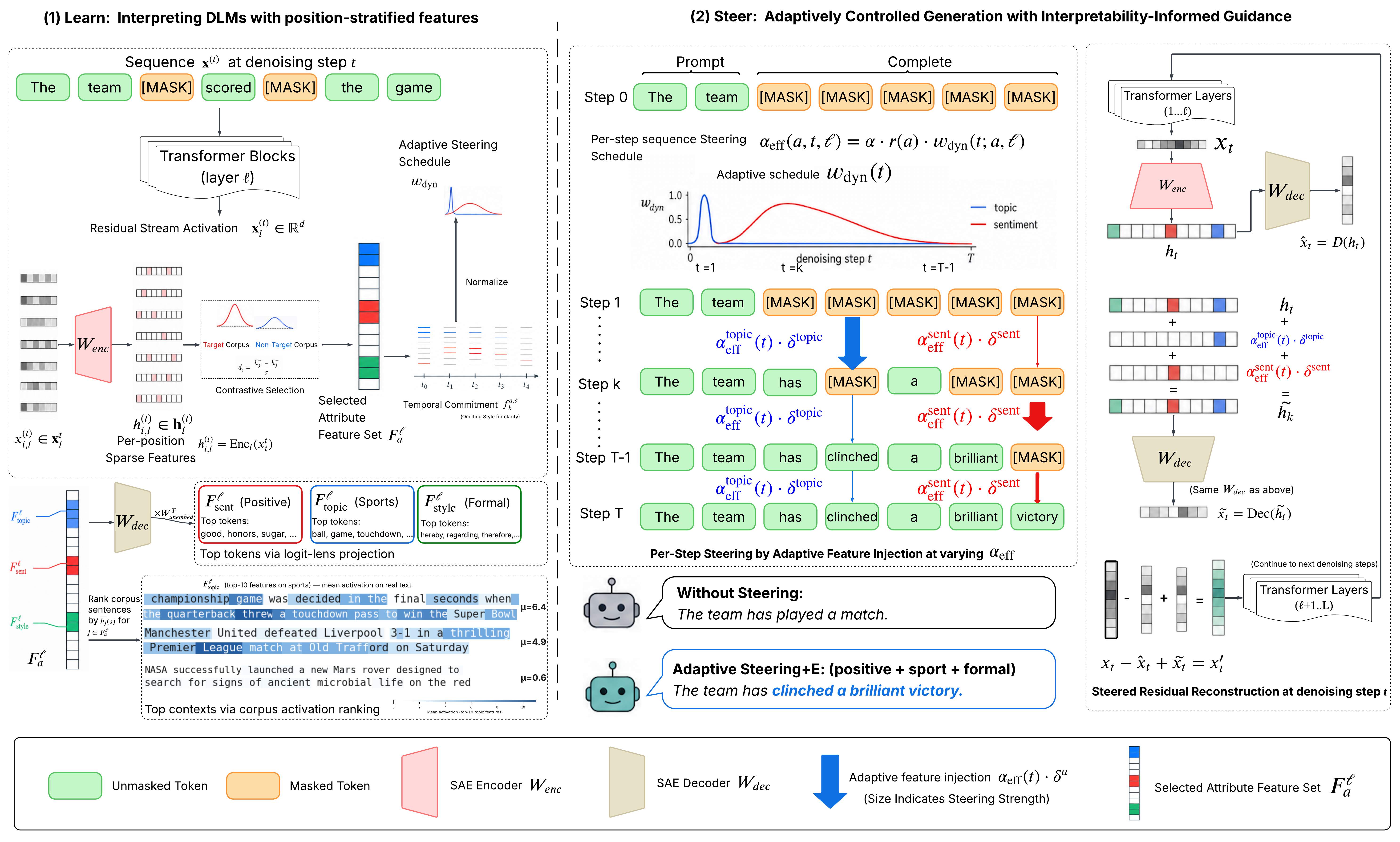}
}
\caption{\textbf{Method overview.} (1) SAEs trained per-layer decompose the residual stream into interpretable features. Contrastive selection identifies attribute-relevant feature sets $F_a^\ell$, whose temporal commitment across denoising steps defines an adaptive schedule $w_{\mathrm{dyn}}(t)$. (2) At each step $t$, a sparse contrastive shift $\alpha_{\mathrm{eff}}(a,t,\ell)\,\delta^{(a)}$ is applied to selected features and decoded with residual correction, enabling multi-attribute composition. All notations and terms introduced are summarized in Table \ref{tab:notation}.}
\label{fig:intro_figure}
\end{figure}

We take an interpretability-first approach to address this gap (Figure~\ref{fig:intro_figure}). We train SAEs on four DLMs spanning two training objectives, three architectures, and a 60$\times$ scale range (MDLM~\citep{sahoo2024mdlm}, SEDD~\citep{lou2024sedd}, DREAM~\citep{dream2025}, and LLaDA~\citep{nie2025llada}) and use the resulting features to characterize how attributes evolve across the denoising trajectory. Our analysis (\S\ref{sec:interpretation}) identifies three properties that directly inform steering: \emph{commitment timing}, where attributes commit at different denoising steps; \emph{commitment sharpness}, where emergence may be concentrated in a narrow window or spread across the trajectory; and \emph{effect-size asymmetry}, where attributes carry different per-feature discriminative strength. These findings explain the quality cost of uniform steering: a constant intervention spends budget on steps where the target attribute has already committed or has not yet begun to emerge. From these empirical profiles, we derive \emph{adaptive steering}, a schedule that concentrates intervention where each attribute is actively forming and leaves the rest of the trajectory undisturbed. An idealized budgeted-allocation analysis shows that the optimal schedule should concentrate intervention where emergence is sharpest, motivating our design (\S\ref{sec:steering}). For multi-attribute composition, we steer through disjoint SAE feature sets and rebalance per-attribute strengths via an effectiveness-ratio calibration so weaker attributes are not drowned out.

We conduct extensive experiments across four DLMs, seven single- and multi-attribute steering tasks, and six baselines spanning residual-stream, prompt-based, and SAE-based approaches. Our key findings are: (i) SAE-based steering provides the most consistent control profile, avoiding the model-dependent failure modes of residual-stream and prompt-based approaches; and (ii) adaptive scheduling consistently improves the control-quality balance over uniform and step-function~\citet{kynkaanniemi2024applying} steering, achieving lower quality cost at matched control and shifting the Pareto frontier toward higher quality (\S\ref{sec:experiments}). In summary, we contribute: (i) a mechanistic characterization of how semantic attributes emerge during DLM denoising, identifying commitment timing, sharpness, and effect-size asymmetry as primary axes of variation; (ii) an adaptive steering framework that concentrates intervention where each attribute is actively forming, with per-attribute effect-size calibration, informed by a budgeted-allocation analysis; and (iii) a comprehensive evaluation showing that adaptive steering achieves the best quality-control balance across models.

%% file: 3_related_works.tex
\section{Related Work}

SAEs have become a standard tool for feature interpretability in autoregressive LMs, linking features to concepts via vocabulary projections~\citep{cunningham2023sparse, bricken2023monosemanticity, templeton2024scaling}, with applications to causal steering such as detoxification and safety control~\citep{goyal2025breaking, yeon2025gsae, cho2025corrsteer, rodriguez2026lineas,rodriguez2025controlling, arad2025saes}, though effectiveness can be sensitive to feature selection and layer choice~\citep{ronge2026coffee, basu2026interpretability}. %
Optimal-transport methods~\citep{rodriguez2025controlling, rodriguez2026lineas} fix operator strength within a single forward pass; our schedule modulates when to intervene across the $T$ denoising steps, making the two approaches complementary. For DLMs, concurrent work has begun exploring interpretability and control. DLM-Scope~\citep{wang2026dlmscope} trains SAEs on DREAM and LLaDA and studies temporal decoding dynamics, but does not perform cross-objective comparison or demonstrate compositional control. ILRR~\citep{avrahami2026ilrr} steers MDLM and LLaDA via reference-sequence alignment without feature-level interpretability. Other efforts study AR vs. DLM representations~\citep{goel2026skip} and temporal attention for hallucination detection~\citep{hemmat2026tdgnet}. Multi-attribute control in AR models requires specialized methods~\citep{oozeer2025ksteering, shafran2026directions}; we show that SAE feature disjointness in DLMs enables additive composition. Table~\ref{tab:related_work} summarizes how our work compares with concurrent efforts.

\begin{table*}[t!]
\centering
\caption{Comparison of our work with concurrent efforts. \cmark\ = fully, \pmark\ = partially, \xmark\ = not addressed.}
\label{tab:related_work}
\vspace{1mm}
\begin{adjustbox}{max width=\textwidth}
\begin{tabular}{@{} l l l ccccccc @{}}
\toprule
\textbf{Work} &
\textbf{Models} &
\textbf{Approach} &
\makecell{\textbf{Feature}\\\textbf{Interpretability}} &
\makecell{\textbf{Causal}\\\textbf{Steering}} &
\makecell{\textbf{Temporal}\\\textbf{Dynamics}} &
\makecell{\textbf{Cross-}\\\textbf{Objective}} &
\makecell{\textbf{Compositional}\\\textbf{Control}} &
\makecell{\textbf{Multi-}\\\textbf{Scale}} &
\textbf{Baselines} \\
\midrule
DLM-Scope \citep{wang2026dlmscope}
  & Dream, LLaDA
  & SAE (Top-K, 16K)
  & \pmark & \cmark & \cmark & \xmark & \xmark & \xmark & \pmark \\[3pt]
ILRR \citep{avrahami2026ilrr}
  & MDLM, LLaDA
  & Reference-based
  & \xmark & \cmark & \pmark & \xmark & \xmark & \pmark & \cmark \\[3pt]
Skip-to-Good-Part \citep{goel2026skip}
  & LLaDA, Dream
  & Probing / RSA
  & \xmark & \xmark & \pmark & \pmark & \xmark & \xmark & \xmark \\[3pt]
TDGNet \citep{hemmat2026tdgnet}
  & LLaDA, Dream
  & Temporal graphs
  & \xmark & \xmark & \cmark & \xmark & \xmark & \xmark & \xmark \\[3pt]
\midrule
\textbf{Ours}
  & \makecell[l]{MDLM, SEDD,\\Dream, LLaDA}
  & SAE (Top-K, 12--16K)
  & \cmark & \cmark & \cmark & \cmark & \cmark & \cmark & \cmark \\
\bottomrule
\end{tabular}
\end{adjustbox}
\end{table*}

%% file: 4_method.tex
\section{Interpreting Discrete Diffusion LMs via Sparse Autoencoders}
\label{sec:interpretation}

DLMs offer a structural advantage for controlled generation: all positions are refined across a temporal trajectory, so different attributes may commit at different denoising steps, creating natural intervention points. Characterizing when and how sharply each attribute commits, and how attributes interact through temporal overlap or shared features, is essential for designing effective steering schedules. We use sparse autoencoders (SAEs), which decompose dense activations into interpretable feature dictionaries and have been applied to steer AR models~\citep{goyal2025breaking, yeon2025gsae, cho2025corrsteer, arad2025saes}, but remain largely unstudied for DLMs. We train SAEs on four DLMs and study how features encoding sentiment, topic, and style evolve and interact.

\subsection{SAE Training and Attribute Feature Extraction}
\label{sec:method_overview}

First, we train TopK SAEs (k=32)~\citep{gao2025scaling, makhzani2013k, zhu2025abstopk} on residual-stream activations of four DLMs: MDLM (124M), SEDD (124M), LLaDA (8B), and DREAM (7B). They vary by architecture and diffusion objective: SEDD uses score-entropy loss while the other three use absorbing loss. We train at layers selected by our diffusion logit lens and layer study (\S\ref{sec:logit_lens}, \S\ref{sec:layer_study}). DLM activations shift with the masking rate, so we sample training activations across uniformly distributed masking rates to cover the full denoising trajectory (see \ref{sec:sae_training} for details).

Next, we study three attributes: \textit{sentiment}, \textit{topic}, and \textit{style}, a natural combination for multi-attribute study as they capture orthogonal aspects of text: subject matter, opinion polarity, and writing register. Following work on contrastive steering directions~\citep{zou2023representation, rimsky-etal-2024-steering, belinkov-2022-probing}, we construct balanced binary corpora: Sports vs.\ Business from AG News~\citep{zhang2015character} for topic, positive vs.\ negative IMDB reviews~\citep{maas-etal-2011-learning} for sentiment, and formal vs.\ informal text for style, using the Pavlick corpus~\citep{pavlick2016empirical} for MDLM/SEDD and classifier-labeled IMDB reviews~\citep{babakov2023don} for LLaDA/DREAM to match their generation domain (\S\ref{app:extraction_datasets}). %
For style, we track the direction with greater room for control: formality on MDLM/SEDD (baseline ~55–58\%) and informality on LLaDA/DREAM (baseline ~22–27\%). To identify attribute-encoding features, we noise each text at multiple mask ratios, pass the noised inputs through the model, and ReLU-encode the hidden states through the SAE at each layer (\S\ref{app:extraction_procedure}). For each feature $j$, we compute the Cohen's $d$ effect size~\citep{cohen2013statistical}: $d_j = (\bar{h}_j^{+} - \bar{h}_j^{-}) / \sqrt{(\text{Var}(h_j^{+}) + \text{Var}(h_j^{-}))/2}$, where $\bar{h}_j^{+}$ and $\bar{h}_j^{-}$ are the mean activations for the two contrastive classes. We retain the top 50 features per direction per layer that pass a Mann--Whitney $U$ test~\citep{mann1947test} ($p < 0.01$, Bonferroni-corrected~\citep{bonferroni1936teoria}) for the interpretability analysis below. Vocabulary grounding via $\mathbf{W}_\text{unembed}$ confirms that retained features are semantically coherent (\S\ref{app:feature_vocab_sparsity}).

\subsection{Interpreting Temporal Dynamics of Features}\label{sec:temporal_dynamics}

To analyze feature dynamics, we generate $N$ samples per model with SAE hooks active at every step, record per-step feature activations, filter to classifier-confident samples using off-the-shelf attribute classifiers (since not all generated texts express the target attribute; \S\ref{app:dynamics}), and partition the trajectory into $B$ temporal blocks. We quantify emergence timing via \emph{block fractions}, the share of total non-negative activation change in each block:
$f_b = {\max(0, \bar{h}_{e_b} {-} \bar{h}_{s_b})}\big/{\sum_{b'} \max(0, \bar{h}_{e_{b'}} {-} \bar{h}_{s_{b'}})}$,
where $s_b$ and $e_b$ are the start and end steps of block $b$. We summarize our key observations below.

\textbf{Commitment timing: attributes commit at different denoising steps.}
Attribute emergence timing varies by attribute, model, and layer. For topic (Figure~\ref{fig:compact_emergence}), MDLM and SEDD commit within the first few percent of denoising, with deeper layers faster; LLaDA shows gradual trajectories; DREAM defers emergence to the final quarter, likely due to its AR-based initialization which inherits late-trajectory token-finalization behavior. Figure~\ref{fig:commitment_times} reports $t_c(p)$ (trajectory percentage at which $p$\% of emergence has occurred) at the deepest layer per model. The patterns differ strikingly: on MDLM, topic reaches 50\% emergence at $t_c{=}22$ (2\% of trajectory) while sentiment reaches it at $t_c{=}232$ (23\%), a ${\sim}10\times$ gap. Emergence trajectories for sentiment and style are in \S\ref{app:emergence_all_attrs}.

\textbf{Commitment sharpness: models differ in how concentrated emergence is.} Figure~\ref{fig:compact_blockfrac} quantifies emergence concentration via block fractions. MDLM and SEDD show consistent early concentration across all layers; at the deepest layer, block~0 captures 71\% of topic emergence on MDLM and 75\% on SEDD. LLaDA shows a different pattern: sentiment concentrates in block~0 (${\sim}$72\%) while topic and style are more spread (27--35\%), and non-deepest layers show milder concentration overall (\S\ref{app:block_fractions_all_layers}). DREAM concentrates \emph{all} emergence in the final two blocks at every layer, with blocks~0--5 contributing zero. Commitment timing and sharpness together define each model's temporal profile; the diversity of these profiles across models, layers, and attributes means that \emph{no single fixed schedule can be optimal for controllability}, motivating adaptive steering (\S\ref{sec:adaptive}).

\begin{figure}[t!]
\centering
\subcaptionbox{Emergence trajectories\label{fig:compact_emergence}}{%
  \includegraphics[width=0.35\textwidth]{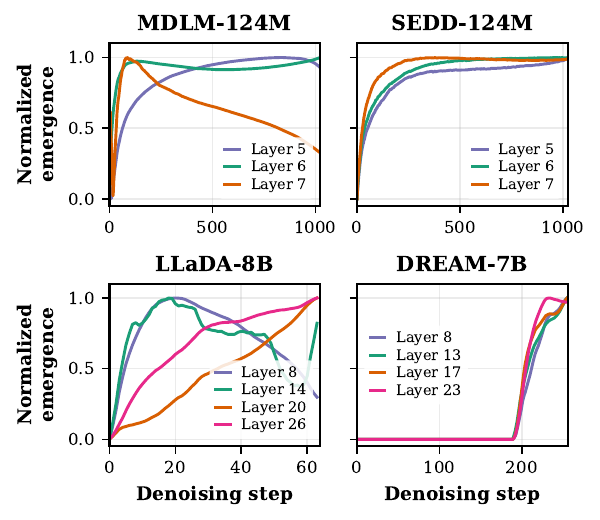}}%
\hfill
\subcaptionbox{Block fractions\label{fig:compact_blockfrac}}{%
  \includegraphics[width=0.35\textwidth]{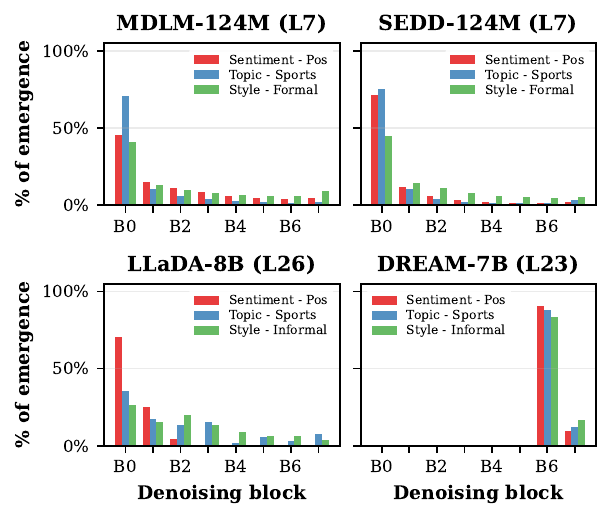}}%
\hfill
\subcaptionbox{Commitment times\label{fig:commitment_times}}{%
  \includegraphics[width=0.23\textwidth]{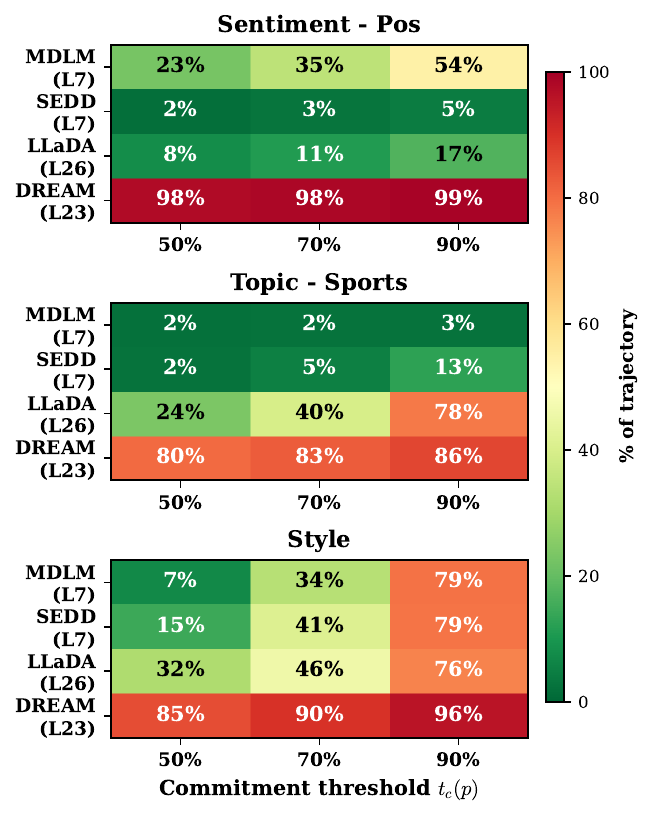}}
\caption{Temporal hierarchy of feature emergence: (a) Topic emergence (normalized $|a(t) - a(0)|$) across layers; (b) Block fractions ($B{=}8$) at deepest layers; (c) Commitment time $t_c(p)$: \% of trajectory at which $p$\% of emergence has occurred (deepest layer, all attributes). Green = early; Red = late.}

\label{fig:temporal_hierarchy}
\end{figure}

\begin{figure}[t!]
\centering
\subcaptionbox{Cumulative effect size\label{fig:effect_size_cumulative}}{%
  \includegraphics[height=2.5cm]{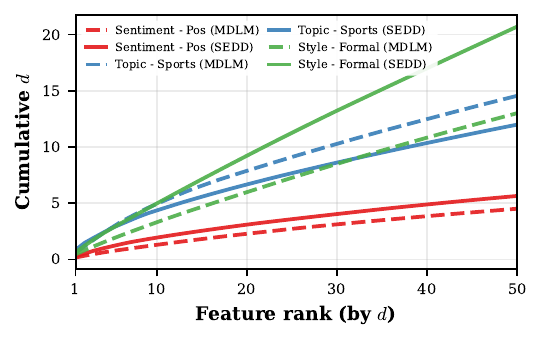}}%
\hfill
\subcaptionbox{Total effect size\label{fig:effect_size_total}}{%
  \includegraphics[height=2.7cm]{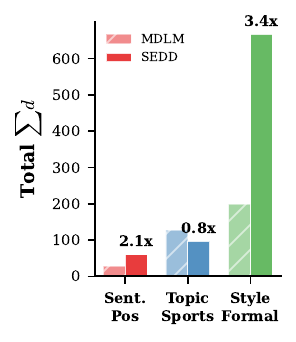}}%
\hfill
\subcaptionbox{Effect spread\label{fig:effect_size_spread}}{%
  \includegraphics[height=2.7cm]{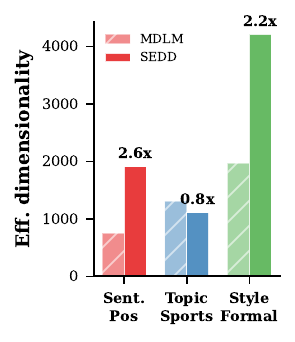}}%
\hfill
\subcaptionbox{Cross-model alignment\label{fig:compact_loss_geometry}}{%
  \includegraphics[height=2.5cm]{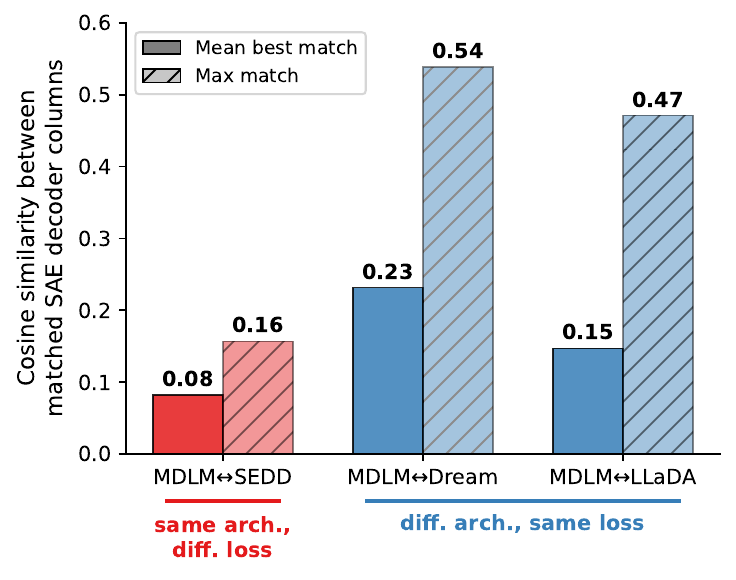}}
\caption{Training objective shapes feature geometry. \textbf{(a--c)} MDLM vs.\ SEDD (layer~7, Cohen's $d > 0$ features): \textbf{(a)} cumulative Cohen's $d$ vs.\ rank; \textbf{(b)} total $\sum d$; \textbf{(c)} effective dimensionality. \textbf{(d)} Cross-model alignment: sharing a loss family yields higher similarity than sharing architecture.}
\label{fig:effect_size}
\end{figure}

\textbf{Effect-size asymmetry: attributes differ in per-feature discriminative strength.}
Topic features carry $2{-}3\times$ larger per-feature Cohen's $|d|$ than sentiment across all models (Figure~\ref{fig:effect_size_cumulative}), meaning equal-strength interventions over-steer topic and under-steer sentiment, motivating effectiveness-ratio calibration (\S\ref{sec:adaptive}). Comparing MDLM and SEDD (same architecture and data, different loss) at layer~7 reveals further structure (Figure~\ref{fig:effect_size}): SEDD carries $2.1\times$ (sentiment) and $3.4\times$ (formality) more total signal ($\sum d$, Figure~\ref{fig:effect_size_total}), while topic is comparable ($0.8\times$). The effective dimensionality ($(\sum d)^2 / \sum d^2$; Figure~\ref{fig:effect_size_spread}) shows SEDD spreads sentiment across $2.6\times$ and formality across $2.2\times$ more features; topic is more concentrated. We examine the effect on steering in \S\ref{sec:experiments}.

\textbf{Training objective shapes feature geometry.}
Cross-model feature alignment (details in \S~\ref{app:cross_model_cosine}) confirms the training objective as the dominant factor (Figure~\ref{fig:compact_loss_geometry}): MDLM and SEDD yield near-orthogonal features (cosine sim.\ 0.08), while same-loss-family pairs show substantially higher alignment: MDLM and DREAM (different architecture, $56\times$ scale gap) achieve 0.23 mean and 0.54 max similarity, and MDLM and LLaDA ($64\times$ scale gap) achieve 0.15 mean and 0.47 max.

Together, these findings motivate a steering framework that adapts to each model's temporal structure.%

%% file: 5_interpretability_2.tex
\vspace{-10pt}
\section{Interpretability Informed Steering}
\label{sec:steering}

The analysis in \S\ref{sec:interpretation} reveals that attributes commit at different steps with different sharpness and carry different per-feature discriminative strengths. %
We exploit this by steering in SAE feature space, modulating intervention strength across denoising steps and attributes.

\textbf{Setup.} At denoising step $t$ and layer $\ell$, we encode $h_t = \mathrm{Enc}(x_t)$, reconstruct $\hat{x}_t = \mathrm{Dec}(h_t)$, and modify features on the attribute-specific set $F_a^\ell$:
\begin{equation}
\tilde{h}_{t,j} \;=\; h_{t,j} + \alpha_{\mathrm{eff}}(a, t, \ell)\,\delta_j^{(a)} \qquad (j \in F_a^\ell),
\label{eq:steering-features}
\end{equation}
where $\delta_j^{(a)} = \bar{h}_{j,+}^{(a)} - \bar{h}_{j,-}^{(a)}$ is the contrastive shift, signed so that positive $\alpha_{\mathrm{eff}}$ steers toward the target pole. 
The forwarded residual is 
$x'_t = x_t - \hat{x}_t + \tilde{x}_t$
with $\tilde{x}_t = \mathrm{Dec}(\tilde{h}_t)$; the off-basis term $x_t - \hat{x}_t$ preserves SAE-residual information, and the affine decoder gives $x'_t - x_t = W_{F_a^\ell}(\alpha_{\mathrm{eff}}\,\delta^{(a)})$, so the intervention acts only through the selected coordinates.

\subsection{Steering Configurations}
\label{sec:adaptive}

We decompose the effective steering strength into three components:\
\begin{equation}
\alpha_{\mathrm{eff}}(a, t, \ell) = \underbrace{\alpha}_{\text{overall scale}} \cdot \underbrace{r(a)}_{\text{attribute weight}} \cdot \underbrace{w_{\mathrm{dyn}}(t; a, \ell)}_{\text{temporal weight}}.
\label{eq:general_alpha}
\end{equation}
The scalar $\alpha$ controls the overall steering magnitude (swept during evaluation); $r(a)\in (0,1]$ rebalances strength across attributes; and $w_{\text{dyn}}(t;a,\ell)\in[0,1]$ modulates strength across denoising steps according to each attribute's empirical commitment profile.

For multiple attributes $\mathcal{A} = \{a_1, \dots, a_m\}$, we compose interventions additively in feature space:\phantomsection\label{sec:multi_attr}
\begin{equation}
\tilde{h}_{t,j} \;=\; \mathrm{ReLU}\!\left(h_{t,j} + \sum_{i:\, j \in F_{a_i}^\ell} \alpha_{\mathrm{eff}}(a_i, t, \ell)\, \delta_j^{(a_i)}\right),
\label{eq:additive_feature_composition}
\end{equation}
with steered residual $x'_t = x_t - \hat{x}_t + \tilde{x}_t$. The ReLU clamp enforces the non-negative activation domain of the TopK SAE (a single $\delta^{(a)}$ at moderate $\alpha_{\mathrm{eff}}$ does not drive coordinates negative, but a sum of updates can). This formulation is effective because the contrastive selection procedure (\S\ref{sec:feature_id}) yields near-disjoint feature sets across attributes. Different choices of $w_{\mathrm{dyn}}$ and $r$ yield four steering configurations of increasing sophistication, each incorporating a successive interpretability finding.

\textbf{Uniform ($r{=}1$, $w_{\mathrm{dyn}}{=}1$).} Constant $\alpha$ at every step. This ignores the temporal structure uncovered in \S\ref{sec:temporal_dynamics}: different attributes commit at different denoising steps %
(Figure~\ref{fig:commitment_times}), so a uniform schedule over-steers already-committed attributes and under-steers those still emerging.

\textbf{Adaptive ($r{=}1$).} Directly operationalizes commitment timing and sharpness (Figures~\ref{fig:compact_emergence}--\ref{fig:compact_blockfrac}): intervention is concentrated during active emergence and relaxed elsewhere. We partition the $T$ denoising steps into $K$ blocks and measure the \emph{commitment fraction}, the share of total feature activation growth in each block:
$f_b^{(a,\ell)} \;=\; \frac{\sum_{j \in F_a^\ell} \max\!\big(0,\, \bar h_j^{(s_b^{\mathrm{end}})} - \bar h_j^{(s_b^{\mathrm{start}})}\big)}{\sum_{b'} \sum_{j \in F_a^\ell} \max\!\big(0,\, \bar h_j^{(s_{b'}^{\mathrm{end}})} - \bar h_j^{(s_{b'}^{\mathrm{start}})}\big)},$
where $\bar h_j^{(t)}$ is the mean activation of feature $j$ at step $t$ (averaged over the contrastive corpus) and negative differences are clamped to zero, capturing net emergence only. The adaptive weight for all steps $t$ within block $b$ is
\begin{equation}
w_{\mathrm{dyn}}(t; a, \ell) \;=\; \frac{f_b^{(a,\ell)}}{\max_{b'} f_{b'}^{(a,\ell)}},
\label{eq:dynamics_weight}
\end{equation}
with the peak-emergence block receiving weight~1. When emergence is sharply concentrated (e.g., topic on MDLM, where block~0 captures 71\%; \S\ref{sec:temporal_dynamics}), $w_{\mathrm{dyn}} \approx 1$ early and ${\approx}\,0$ thereafter; when emergence is perfectly uniform ($\eta = 1$), $w_{\mathrm{dyn}} \approx 1$ everywhere and Adaptive reduces exactly to Uniform. In practice, no condition in our set reaches this limit ($\eta$ ranges from 1.44 to 5.01; \S\ref{app:eta_validation}), so Adaptive retains a quality advantage across all conditions.

\textbf{Uniform+E and Adaptive+E.} Addresses effect-size asymmetry (Figure~\ref{fig:effect_size_cumulative}): attributes carry different per-feature Cohen's $d$ magnitudes, so the same $\alpha$ produces disproportionate perturbations (e.g., topic features carry $2{-}3\times$ larger $|d|$ than sentiment; \S\ref{sec:temporal_dynamics}), causing stronger attributes to saturate before weaker ones reach target. We calibrate from single-attribute dose-response sweeps: for each attribute $a$, let $\alpha^*_a$ be the minimum strength reaching a reference control threshold. The effectiveness ratio is $r(a) = \alpha^*_a \big/ \max_{a'} \alpha^*_{a'}$,\label{eq:e_ratio} so the hardest-to-steer attribute receives $r{=}1$ and easier attributes are scaled down. We calibrate empirically rather than deriving $r(a)$ directly from Cohen's $d$ ratios because the relationship between feature-level effect size and classifier response is nonlinear and model-dependent. This normalizes the effect of $\alpha$ across attributes, allowing a single global budget without attribute-specific imbalances. Uniform+E sets $w_{\mathrm{dyn}}{=}1$; Adaptive+E uses both components.

\subsection{Analysis}
\label{sec:analysis}

To build intuition for why temporal concentration matters, we analyze an idealized allocation problem and show how its solution motivates our proxy schedule.

\textbf{Trajectory-level cost decomposition.} The steered and unsteered reverse processes are Markov chains with common initial distribution, so the chain rule gives $\mathrm{KL}(\tilde{p}_\theta \| p_\theta) = \sum_t \mathbb{E}_{x_t}[\mathrm{KL}(\tilde{p}_\theta(x_{t-1}|x_t) \| p_\theta(x_{t-1}|x_t))]$.
Writing the per-step gain as $\Delta g_t = \alpha_t s_t + o(\alpha_t)$ ($s_t \geq 0$) with $\Delta q_t(\alpha_t) = \tfrac{1}{2}c_t\alpha_t^2 + o(\alpha_t^2)$, the allocation problem is then
\begin{equation}
\min_{\alpha_t \geq 0}\; \tfrac{1}{2}\sum_{t=0}^{T-1} c_t \alpha_t^2 \quad \text{s.t.} \quad \sum_{t=0}^{T-1} \alpha_t s_t \;\geq\; E_{\mathrm{target}}.
\label{eq:budget}
\end{equation}

\textbf{Adaptive vs.\ uniform scheduling.}
The KKT conditions for \eqref{eq:budget} yield $\alpha_t^\star \propto s_t/c_t$ (\S\ref{app:prop1-proof}). We compare against the uniform allocation $\alpha_t \equiv \alpha$.

The following result characterizes the idealized setting where per-step gains and costs are known exactly. While these quantities are not directly observable in practice (we use an SAE-derived proxy, as described below), the result yields a testable ordinal prediction.

\begin{theorem}[Adaptive-vs-uniform efficiency ratio]
\label{thm:eff-ratio}
Assume $s_t \geq 0$, $c_t > 0$, and $E_{\mathrm{target}} > 0$. Among schedules with $\tfrac{1}{2}\sum_t c_t \alpha_t^2 = B$, let $E^\star$ and $E_{\mathrm{unif}}$ denote the maximum attribute shifts under the optimal and uniform schedules. Then
\begin{equation}
\rho^2 \;:=\; \left(\frac{E^\star}{E_{\mathrm{unif}}}\right)^2
\;=\; \frac{\big(\sum_t s_t^2/c_t\big)\big(\sum_t c_t\big)}
            {\big(\sum_t s_t\big)^2}
\;=\; 1 \;+\; \mathrm{CV}_c^2(s/c),
\label{eq:eff-ratio}
\end{equation}
where $\mathrm{CV}_c$ is the coefficient of variation of $s_t/c_t$ under the cost-weighted measure $c_t/\sum_{t'} c_{t'}$, with equality $\rho = 1$ iff $s_t/c_t$ is constant in $t$. (See \S\ref{app:prop1-proof} for assumptions and full proof.)
\end{theorem}

\textbf{Interpretation.} The adaptive gain is governed by cost-weighted dispersion of $s_t/c_t$: large when the attribute commits in a few steps, near one when emergence is spread evenly. Within each model, attributes with more concentrated emergence profiles should benefit more from adaptive scheduling. We test this prediction via the schedule concentration statistic $\eta = N\sum_i w_i^2$ and confirm it empirically in \S\ref{sec:results} (full details in \S\ref{app:eta_validation}).

\textbf{Proxy fidelity.} Neither $s_t$ nor $c_t$ is directly observable. The proxy $w_{\mathrm{dyn}}$ approximates $\alpha_t^\star \propto s_t/c_t$ under two conditions (\S\ref{sec:proxy-app}): (i)~active-set growth tracks the marginal sensitivity $s_t$, and (ii)~$c_t$ varies slowly relative to $s_t$. These conditions are plausible but not formally verified; the proxy is justified primarily by the empirical quality gains reported in \S\ref{sec:experiments}. %

\textbf{Cross-attribute coupling.} 
For additive composition \eqref{eq:additive_feature_composition}, we select top-20 features per attribute in \S\ref{sec:exp_setup}. The resulting sets are near-disjoint but not fully so, since some attributes are inherently correlated (e.g., positive sentiment and formality). A formal analysis of the idealized fully-disjoint case, including a decoder-Gram bound on cross-attribute interference, is in \S\ref{app:prop2}; realized cosine coupling values and feature-overlap effects are in \S\ref{app:gram-values}. Empirical interference is reported in \S\ref{sec:results}.

%% file: 6_experiments.tex
\section{Experiments}
\label{sec:experiments}

\subsection{Experimental Setup}
\label{sec:exp_setup}
\textbf{Models and attributes:} We evaluate the four SAE steering modes from \S\ref{sec:adaptive} (Uniform, Adaptive, Uniform+E, and Adaptive+E) on all four DLMs across 7 attribute combinations: 3 single-attribute, sentiment~(S), topic~(T), and style~(St), and 4 multi-attribute (S+T, S+St, St+T, S+T+St). Sentiment is steered toward positive, topic toward sports, and style toward formality (MDLM/SEDD) or informality (LLaDA/DREAM). Each attribute is steered using the top 20 features per direction based on a feature count ablation (\S\ref{app:feature_count_ablation}) that shows: lower-ranked features increasingly encode correlated attributes (e.g., sentiment features that also shift formality), increasing cross-attribute interference. MDLM and SEDD generate unconditionally over 1024 denoising steps; LLaDA and DREAM use 10 domain-neutral prompts (e.g., ``I think this is'') with temperature 0.8/0.7 over 64/256 steps (\S\ref{app:steering_details}).

\textbf{Baselines:} We compare against three residual-stream baselines without SAE decomposition: \textit{Contrastive Vectors} (mean class difference~\citep{zou2023representation, rimsky-etal-2024-steering}), \textit{Probe} (logistic regression weight~\citep{belinkov-2022-probing}), and \textit{PCA} (top principal component of contrastive activations).
We additionally implement a \textit{Step Function} baseline following~\citet{kynkaanniemi2024applying} that applies uniform steering within a commitment-derived interval (\S\ref{app:step_function_baseline}), isolating the contribution of adaptive scheduling from interval restriction alone.
For LLaDA and DREAM, we additionally evaluate \textit{Prompt}, prepending an attribute instruction to the prompt (\S\ref{app:baselines}). We note that other baselines such as ILRR~\citep{avrahami2026ilrr} require a reference sequence (a different problem setup), and classifier-based guidance methods~\citep{dhariwal2021diffusion} require differentiable classifiers integrated into the denoising loop and are not directly applicable to masked discrete diffusion.

\textbf{Metrics:}\label{sec:reporting-rule} Attribute control is measured by continuous classifier confidence (DistilBERT-SST2~\citep{sanh2019distilbert} for sentiment, BERT-AG News~\citep{devlin2019bert} for topic, RoBERTa formality ranker~\citep{babakov2023don} for style; \S\ref{app:evaluation}), which captures degree of attribute expression beyond binary thresholds. Quality is measured by Qwen2-7B perplexity (PPL) and LLM-judge fluency (LLMJ). We use PPL as the primary quality gate where reliable, following the established confidence-plus-PPL paradigm~\citep{dathathri2019plug,liu2021dexperts}. For MDLM, SEDD, and DREAM, PPL and LLMJ correlate strongly ($|r_s| = 0.77$ to $0.88$; \S\ref{app:evaluation}). For LLaDA this correlation breaks down ($r_s = +0.21$) because aggressive steering produces repetitive text with artificially low PPL; we use LLMJ as the quality gate for this model. Human evaluation of fluency is prohibitively expensive at our scale (four models, seven conditions, multiple baselines); recent work has shown LLM judges to be a reliable proxy for human quality assessment in generation tasks~\citep{goyal2025breaking,wu2025axbench}. We validate this for our setting via a human meta-evaluation across all four models covering fluency and all three attribute dimensions: the judge achieves 73\% agreement with human consensus on a 5-point fluency scale, with disagreements concentrated in mid-range scores rather than the extremes that drive quality-gate decisions (full protocol and inter-annotator statistics in \S\ref{app:llmj_meta}). %
For each method, we generate 200 samples per $\alpha \in \{1,\ldots,20\}$ and select the operating point maximizing target confidence (single-attribute) or mean confidence (multi-attribute), subject to PPL $< 100$ and LLMJ fluency $\geq 50\%$ of baseline. MDLM/SEDD residual-stream baselines are exempt from the PPL gate, as they
produce high PPL regardless of $\alpha$.

\subsection{Results}
\label{sec:results}
\input{figures_tables/table_results_ABC_2}
\textbf{SAE steering vs.\ baselines:} Table~\ref{tab:operator} reports mean steering performance averaged over conditions (per-condition results with selected $\alpha$ values and standard deviations in Tables~\ref{tab:steering_single_detail}--\ref{tab:steering_multi_detail}; \S\ref{app:steering_detail}). Steering via latent-space intervention inherently incurs some quality cost, as observed in AR steering studies~\citep{rimsky-etal-2024-steering,goyal2025breaking}. We observe the same in DLMs, but find that SAE-based methods generally degrade quality less than residual-stream baselines, with the exception of DREAM where residual-stream methods remain competitive.
Residual-stream baselines exhibit model-dependent trade-offs: on MDLM, they achieve high confidence but at severe quality cost; on LLaDA, they frequently fail to steer beyond the unsteered rate for topic and style; on SEDD and DREAM, Contrastive Vectors performs competitively with SAE methods, while Probe steers effectively but with greater quality degradation, and PCA largely fails. Prompt steering (LLaDA, DREAM) preserves quality but achieves limited attribute control, particularly for style and multi-attribute combinations. SAE-based methods provide a more consistent profile across all four models by decomposing the residual stream into monosemantic features.

\begin{wraptable}{r}{0.555\textwidth}
\vspace{-4pt}
\centering
\caption{Schedule comparison: \% of matched points where the first method is better (using model-specific quality metrics; see \S\ref{app:matched_alpha} and \S\ref{app:matched_conf} for details).}
\label{tab:schedule}
\scriptsize
\setlength{\tabcolsep}{3pt}
\begin{tabular}{@{}lccc@{}}
\toprule
 & Matched conf. & Matched $\alpha$ & Matched qual. \\
 & (quality $>$) & (quality $\geq$) & (conf.\ $>$) \\
\midrule
Adaptive vs.\ Uniform & 82.2\% & 97.5\% & 46.7\% \\
Adaptive vs.\ Step    & 71.5\% & 95.0\% & 47.5\% \\
Step vs.\ Uniform     & 66.5\% & 92.9\% & 44.6\% \\
\bottomrule
\end{tabular}

\vspace{8pt}

\caption{$\alpha$ sensitivity: confidence drop (pp) when using a single global $\alpha^*$ across all 7 conditions vs.\ per-condition optimal $\alpha$.}
\label{tab:alpha_sensitivity_main}
\scriptsize
\setlength{\tabcolsep}{3pt}
\begin{tabular}{@{}lcc cc cc@{}}
\toprule
 & \multicolumn{2}{c}{Uniform} & \multicolumn{2}{c}{Step Fn.} & \multicolumn{2}{c}{Adaptive} \\
\cmidrule(lr){2-3} \cmidrule(lr){4-5} \cmidrule(lr){6-7}
Model & mean & max & mean & max & mean & max \\
\midrule
MDLM  & 9.1 & 50.9 & 5.6 & 32.2 & 0.6 & 3.2 \\
SEDD  & 9.2 & 28.4 & 4.6 & 13.9 & 0.2 & 0.6 \\
LLaDA & 1.4 & 2.9  & 1.1 & 4.5  & 0.4 & 1.4 \\
DREAM & 3.1 & 9.1  & 3.2 & 10.8 & 1.2 & 2.1 \\
\bottomrule
\end{tabular}
\vspace{-2pt}
\end{wraptable}

\textbf{Adaptive balances quality and control:}
Among SAE-based methods, Adaptive scheduling achieves a better quality/control balance by concentrating perturbation on steps where attributes are actively forming.
At best operating points (Tables~\ref{tab:steering_single_detail}--\ref{tab:steering_multi_detail}), Adaptive achieves significantly better quality than the second-best SAE method in 22 of 28 conditions (Mann{-}Whitney $U$, $p < 0.05$), with 2 additional conditions showing non-significant improvement (\S\ref{app:significance}); the 4 exceptions all involve DREAM, whose AR-based initialization compresses the temporal window (\S\ref{app:tradeoff_curves}).
Table~\ref{tab:schedule} summarizes pairwise schedule comparisons across three matched-point analyses. Compared to Uniform: at matched confidence, Adaptive produces better quality in 82.2\% of comparisons; at matched $\alpha$, Adaptive achieves equal or better quality in 97.5\% of comparisons (\S\ref{app:matched_alpha}); and at matched quality, Adaptive achieves higher confidence in 46.7\% of comparisons, confirming that Uniform's additional quality degradation does not translate into stronger steering. 
The full three-way ordering (Uniform $\leq$ Step Function $\leq$ Adaptive) holds in 88.2\% of 560 (model, condition, $\alpha$) triplets (\S\ref{app:matched_alpha}).
Figure~\ref{fig:tradeoff_mdlm} illustrates this on MDLM: the confidence/quality Pareto curves show that all schedules reach the same confidence plateau, but Adaptive achieves it at substantially lower perplexity, shifting the trade-off frontier toward higher quality. The same pattern holds across all four models.

\begin{figure*}[t]
\centering
\includegraphics[width=\textwidth]{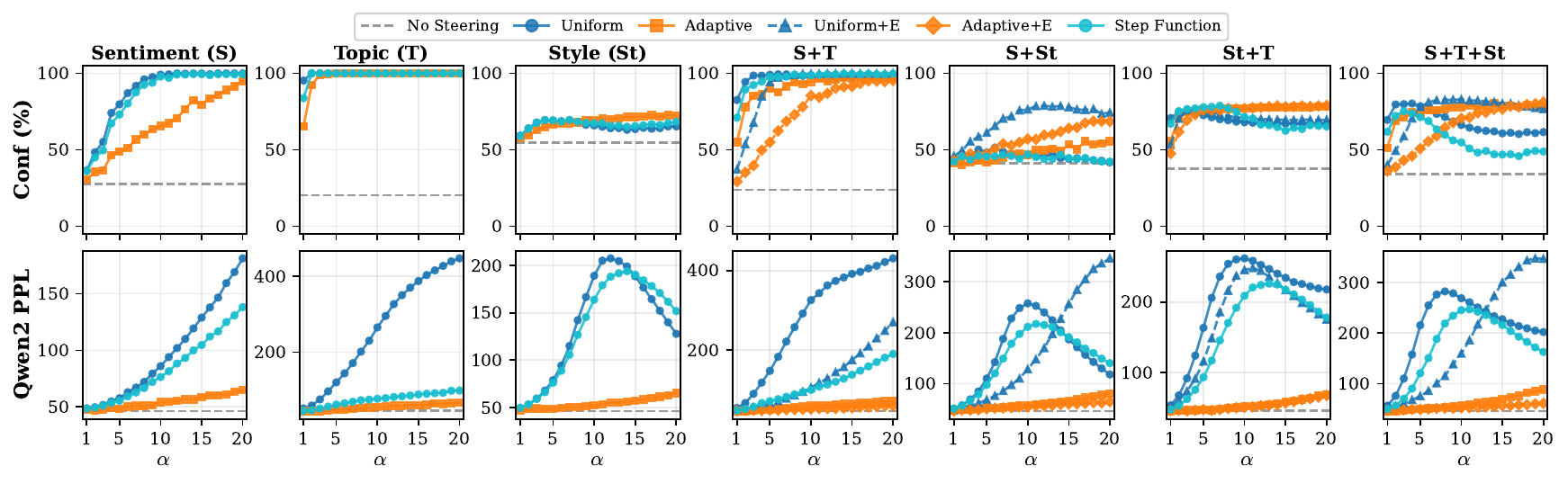}
\caption{MDLM trade-off curves (SAE methods). Top: target confidence vs.\ $\alpha$. Bottom: Qwen2-7B perplexity vs.\ $\alpha$. Full curves with all methods and models are in \S\ref{app:tradeoff_curves}.}
\label{fig:tradeoff_mdlm}
\end{figure*}

\textbf{$\alpha$ sensitivity:} Although tuning a steering coefficient is standard practice~\citep{rimsky-etal-2024-steering,goyal2025breaking,zou2023representation}, Adaptive is robust to $\alpha$ selection. Fixing a single global $\alpha^*$ per (model, method) across all 7 conditions, Adaptive incurs only 0.2–1.2~pp mean confidence drop from the per-condition optimum, while Uniform suffers 1.4–9.2~pp because higher $\alpha$ values violate the quality gate, forcing a low global setting that under-steers most conditions (Table~\ref{tab:alpha_sensitivity_main}; details in \S\ref{app:alpha_sensitivity}).

\begin{wraptable}{r}{0.55\textwidth}
\vspace{-4pt}
\centering
\caption{Effect of E-ratio scaling: mean $\Delta$ from adding E-ratio calibration, averaged over the 4 multi-attribute conditions. LLaDA PPL omitted as it is non-diagnostic. Flu = LLMJ fluency (1--5 scale).}
\label{tab:eratio}
\scriptsize
\setlength{\tabcolsep}{2.4pt}
\begin{tabular}{@{}lcccccc@{}}
\toprule
 & \multicolumn{3}{c}{Uniform vs.\ Uniform+E} & \multicolumn{3}{c}{Adaptive vs.\ Adaptive+E} \\
\cmidrule(lr){2-4}\cmidrule(l){5-7}
Model & $\Delta$Conf & $\Delta$PPL & $\Delta$Flu & $\Delta$Conf & $\Delta$PPL & $\Delta$Flu \\
\midrule
MDLM  & +6.4 & +2.5 & +0.01 & +3.0 & $-$8.2 & +0.05 \\
SEDD  & +1.1 & +7.5 & +0.01 & $-$12.2 & +10.6 & +0.04 \\
LLaDA & $-$6.7 & -- & +0.10 & $-$10.2 & -- & +0.12 \\
DREAM & +1.5 & $-$5.0 & +0.10 & +0.2 & $-$26.8 & +0.70 \\
\bottomrule
\end{tabular}
\vspace{-2pt}
\end{wraptable}
\textbf{E-ratio calibration:}
E-ratio re-balances per-attribute strength by effect size (\S\ref{sec:steering}). Table~\ref{tab:eratio} reports the mean change from adding E-ratio calibration across multi-attribute conditions. Adaptive+E achieves better or comparable quality to Adaptive in 12/16 multi-attribute conditions, with the exception of SEDD, where distributed feature geometry (\S\ref{sec:temporal_dynamics}) increases cross-attribute overlap, reducing calibration reliability.
Uniform+E shows more modest or inconsistent gains.
Adaptive+E achieves significantly better quality than Uniform+E in all comparable multi-attribute conditions (\S\ref{app:steering_modes}).%

\textbf{Cross-attribute interference:} We quantify interference as the mean absolute pp shift of non-target attributes from their unsteered baselines (Figure~\ref{fig:interference}, \S\ref{app:interference}). Some interference is unavoidable because attributes are naturally correlated (e.g., positive sentiment and formality). 
Adaptive consistently produces less interference than Uniform, with the largest gap on MDLM (up to 57pp for Uniform vs.\ 33pp for Adaptive),
where targeting high-emergence steps avoids disrupting other attributes during their critical commitment.

\textbf{Validating the efficiency prediction:}
Theorem~\ref{thm:eff-ratio} predicts that attributes with more concentrated emergence benefit more from adaptive scheduling. We measure this via the schedule concentration statistic $\eta = K\sum_i w_i^2$, where $w_i$ are the normalized block weights ($\sum_i w_i = 1$, so $\eta = 1$ for a uniform schedule; \S\ref{sec:analysis}). Across all 12 (model, attribute) conditions, using the appropriate quality metric for each model (\S\ref{sec:exp_setup}), the within-model rank ordering of quality gains perfectly matches the $\eta$ prediction for MDLM, SEDD, and LLaDA (Spearman $r_s = 1.0$ in each case; pooled exact permutation test across models, $p \approx 0.007$). DREAM shows $r_s = 0.5$ because its attributes share near-identical emergence profiles ($\eta$ range 4.09--5.01), leaving little room for rank separation (\S\ref{app:eta_validation}).

%% file: figures_tables/table_results_ABC_2.tex
\begin{table}[!t]
\centering
\caption{Steering performance comparison at best operating points (selection criteria in \S\ref{sec:exp_setup}), averaged over single-attribute (sentiment~(S), topic~(T), style~(St)) and multi-attribute (S+T, S+St, St+T, S+T+St) conditions. Conf = target attribute classifier confidence (\%), PPL = Qwen2-7B perplexity, Flu = LLM-judge fluency (1--5 scale). Per-condition results with std in Tables~\ref{tab:steering_single_detail}--\ref{tab:steering_multi_detail}.}
\label{tab:operator}
\resizebox{\textwidth}{!}{%
\scriptsize
\setlength{\tabcolsep}{2.5pt}
\begin{tabular}{@{}ll cccccccc | cccccccc@{}}
\toprule
 & & \multicolumn{8}{c|}{\textit{Single-attribute (mean of 3)}} & \multicolumn{8}{c}{\textit{Multi-attribute (mean of 4)}} \\
\cmidrule(lr){3-10}\cmidrule(l){11-18}
 & & & \multicolumn{3}{c}{SAE-based} & \multicolumn{4}{c|}{Res.\ stream / prompt} & & \multicolumn{3}{c}{SAE-based} & \multicolumn{4}{c}{Res.\ stream / prompt} \\
\cmidrule(lr){4-6}\cmidrule(lr){7-10}\cmidrule(lr){12-14}\cmidrule(l){15-18}
 & & Unsteered & Uniform & Step & Adaptive & Contr.V & Probe & PCA & Prompt & Unsteered & Uniform & Step & Adaptive & Contr.V & Probe & PCA & Prompt \\
\midrule
 \multirow{3}{*}{\textbf{MDLM}} & Conf$\uparrow$ & 34.2 & 89.5 & 89.6 & 89.1 & 69.9 & 80.2 & 26.4 & -- & 34.2 & 75.8 & 74.5 & 78.0 & 77.3 & 76.8 & 24.8 & -- \\
  & PPL$\downarrow$ & 46.0 & 86.3 & 88.5 & 59.7 & 309.3 & 462.6 & 300.0 & -- & 46.0 & 85.5 & 82.6 & 70.7 & 353.5 & 441.4 & 281.1 & -- \\
  & Flu$\uparrow$ & 2.32 & 2.06 & 2.07 & 2.15 & 1.34 & 1.21 & 1.39 & -- & 2.32 & 2.00 & 2.01 & 2.01 & 1.40 & 1.19 & 1.39 & -- \\
\midrule
 \multirow{3}{*}{\textbf{SEDD}} & Conf$\uparrow$ & 33.4 & 97.7 & 96.5 & 89.1 & 97.1 & 92.5 & 41.0 & -- & 33.3 & 89.7 & 88.8 & 88.6 & 92.2 & 84.6 & 45.3 & -- \\
  & PPL$\downarrow$ & 48.1 & 83.8 & 67.7 & 48.9 & 73.3 & 136.3 & 45.5 & -- & 48.1 & 91.3 & 88.2 & 42.5 & 77.5 & 130.1 & 51.1 & -- \\
  & Flu$\uparrow$ & 2.23 & 2.00 & 2.10 & 2.17 & 2.12 & 1.86 & 2.25 & -- & 2.23 & 1.99 & 1.99 & 2.11 & 2.02 & 1.83 & 2.25 & -- \\
\midrule
 \multirow{3}{*}{\textbf{LLaDA}} & Conf$\uparrow$ & 35.1 & 95.0 & 91.6 & 84.8 & 52.8 & 44.9 & 41.7 & 63.6 & 35.1 & 84.1 & 81.2 & 76.5 & 48.1 & 44.0 & 39.6 & 57.0 \\
  & PPL$\downarrow$ & 7.1 & 35.2 & 5.1 & 5.5 & 6.7 & 7.7 & 8.0 & 8.2 & 7.1 & 5.3 & 4.3 & 8.7 & 8.8 & 7.1 & 6.6 & 8.2 \\
  & Flu$\uparrow$ & 3.56 & 2.56 & 2.64 & 2.97 & 3.41 & 3.55 & 3.54 & 3.15 & 3.56 & 2.20 & 2.30 & 2.75 & 3.48 & 3.60 & 3.57 & 3.19 \\
\midrule
 \multirow{3}{*}{\textbf{DREAM}} & Conf$\uparrow$ & 30.8 & 99.5 & 99.2 & 98.8 & 95.5 & 90.0 & 45.0 & 54.1 & 30.8 & 95.3 & 95.3 & 95.8 & 84.7 & 78.3 & 41.9 & 53.5 \\
  & PPL$\downarrow$ & 8.6 & 35.8 & 26.7 & 20.8 & 12.2 & 13.2 & 26.9 & 7.7 & 8.6 & 63.6 & 58.6 & 64.2 & 11.7 & 15.7 & 21.7 & 8.8 \\
  & Flu$\uparrow$ & 3.75 & 2.05 & 2.08 & 2.21 & 2.96 & 2.79 & 2.84 & 3.57 & 3.75 & 1.69 & 1.43 & 1.34 & 2.23 & 2.72 & 2.81 & 3.41 \\
\bottomrule
\end{tabular}}%
\end{table}

%% file: 7_conclusion_limitations.tex
\vspace{-8pt}
\section{Conclusion}

We use SAEs to reveal how semantic features emerge during discrete diffusion denoising, discovering that attributes commit at different steps, with different sharpness, and at different magnitudes. 
These findings inform an adaptive schedule that concentrates intervention where each attribute is actively forming and improves the control--quality trade-off over uniform and interval-restricted schedules at matched control.
Our analysis separates \textit{when} concentration helps (headroom) from \textit{how well} a schedule exploits it (alignment), turning interpretability into a practical lever for controlled generation.

%% file: appendix.tex
\appendix

\section{Proofs and Extended Discussion for Steering Methods}
\label{app:steering_proofs}

We restate each result with a table explaining all notations used throughout the paper for convenience and provide full proofs.

\subsection{Notation and Terms Explanation}
\label{app:notations}
\begin{table}[h]
\centering
\small
\caption{Notation used throughout the paper. Indices: $t$ denotes a
denoising step (of $T$ total), $\ell$ a transformer layer, $a$ an attribute,
$b$ a temporal block, $j$ an SAE feature.}
\label{tab:notation}
\begin{tabular}{ll}
\toprule
\textbf{Symbol} & \textbf{Meaning} \\
\midrule
\multicolumn{2}{l}{\emph{Model and denoising}} \\
$T$            & Total number of denoising steps \\
$t$            & Denoising step index, $t \in \{1,\dots,T\}$ (or $0{:}T$)  \\
$\ell$         & Transformer layer index \\
$h_t^{\ell}$   & Residual-stream activation at step $t$, layer $\ell$ \\
\midrule
\multicolumn{2}{l}{\emph{Sparse autoencoder}} \\
$\mathrm{Enc},\,\mathrm{Dec}$ & SAE encoder / decoder \\
$W_{\mathrm{dec}}$ & SAE decoder matrix; column $w^{\mathrm{dec}}_j$ is feature $j$'s direction \\
$j$            & SAE feature index \\
$F_a^{\ell}$   & Attribute-specific feature set for attribute $a$ at layer $\ell$ \\
               & \quad ($|F_a^{\ell}|=20$ for steering; $50$ for the interpretability analysis)  \\
\midrule
\multicolumn{2}{l}{\emph{Attributes and emergence}} \\
$a$            & An attribute (sentiment, topic, or style) \\
$b$            & Temporal block index, $b \in \{1,\dots,K\}$ \\
$K$            & Number of temporal blocks  \\
$\bar h_j^{(t)}$ & Mean activation of feature $j$ at step $t$ over the contrastive corpus \\
$f_b^{(a,\ell)}$ & Commitment fraction: share of attribute $a$'s total feature-activation \\
               & \quad growth occurring in block $b$  \\
\midrule
\multicolumn{2}{l}{\emph{Steering}} \\
$\delta^{(a)}_j$ & Contrastive shift for feature $j$: $\bar h^{(a)}_{j,+}-\bar h^{(a)}_{j,-}$ \\
$\alpha$       & Overall steering scale (swept during evaluation) \\
$r(a)$         & Per-attribute effectiveness ratio, $r(a)\in(0,1]$ \\
$w_{\mathrm{dyn}}(t;a,\ell)$ & Adaptive temporal weight $\in[0,1]$ from the commitment profile \\
$\alpha_{\mathrm{eff}}(a,t,\ell)$ & Effective strength $=\alpha \cdot r(a)\cdot w_{\mathrm{dyn}}(t;a,\ell)$ \\
\midrule
\multicolumn{2}{l}{\emph{Theory}} \\
$s_t$          & Realized per-step attribute gain (marginal sensitivity at step $t$) \\
$c_t$          & Realized per-step KL cost of the intervention at step $t$ \\
$\rho^{2}$     & Adaptive-vs-uniform efficiency ratio, $=1+\mathrm{CV}_c^{2}(s/c)$  \\
$\mathrm{CV}_c$ & Coefficient of variation of $s_t/c_t$ under the cost-weighted measure \\
\bottomrule
\end{tabular}
\end{table}

\subsection{Proof of Theorem~\ref{thm:eff-ratio} (Adaptive-vs-uniform efficiency ratio)}
\label{app:prop1-proof}

\subsubsection{Trajectory-level KL: exact additive decomposition}
\label{app:additive-cost}

Theorem~\ref{thm:eff-ratio} optimizes against the budget $\tfrac{1}{2}\sum_t c_t \alpha_t^2 \leq B$, where each $c_t$ is the second-order cost of perturbation at step $t$. This budget arises as a tractable additive decomposition of the trajectory-level KL between steered and unsteered reverse processes; we derive the decomposition here.

Let $p_\theta, \tilde p_\theta$ denote the joint distributions of the unsteered and steered reverse trajectories $(x_T, x_{T-1}, \ldots, x_0)$. Both are Markov: the reverse transition at step $t$ is $p_\theta(x_{t-1} \mid x_t)$ for the unsteered model and $\tilde p_\theta(x_{t-1} \mid x_t)$ for the steered model (which differs only in the residual stream perturbation at step $t$). Both processes share the same initial distribution at $x_T$ (e.g., the fully-masked prior in absorbing-state diffusion), so $\mathrm{KL}(\tilde p_\theta(x_T) \,\|\, p_\theta(x_T)) = 0$. The chain rule for KL on Markov chains applied to the reverse joint, then it gives
\begin{equation}
\mathrm{KL}\!\big(\tilde p_\theta \,\big\|\, p_\theta\big)
\;=\; \sum_{t=1}^{T} \mathbb{E}_{x_t \sim \tilde p_\theta}\!\Big[\mathrm{KL}\big(\tilde p_\theta(x_{t-1} \mid x_t) \,\big\|\, p_\theta(x_{t-1} \mid x_t)\big)\Big].
\label{eq:markov-kl-app}
\end{equation}

\paragraph{First-order change-of-measure approximation.}
The outer expectation in \eqref{eq:markov-kl-app} is under the steered measure $\tilde p_\theta$, while $x_t$-marginal samples are easier to obtain under $p_\theta$. Replacing $\tilde p_\theta$ by $p_\theta$ in the outer expectation incurs an error that is sub-leading in $\alpha$: the Radon–Nikodym derivative $d\tilde p_\theta/dp_\theta$ equals $1$ at $\alpha = 0$ and varies smoothly, so the difference $\mathbb{E}_{\tilde p_\theta}[\Phi(x_t)] - \mathbb{E}_{p_\theta}[\Phi(x_t)]$ is $O(\alpha)$ for any bounded $\Phi$. Since the integrand $\Phi_t(x_t) := \mathrm{KL}(\tilde p_\theta(x_{t-1}\mid x_t)\,\|\,p_\theta(x_{t-1}\mid x_t))$ is itself $O(\alpha^2)$ in leading order (Lemma~\ref{lem:ct} below), the change-of-measure correction contributes $O(\alpha) \cdot O(\alpha^2) = O(\alpha^3)$, which is sub-leading relative to the $O(\alpha^2)$ leading cost. The replacement therefore preserves the leading-order behavior of the total cost. We use this approximation throughout.

\paragraph{Per-step quadratic expansion.}
For each summand, a second-order Taylor expansion of the KL around $\alpha_t = 0$ gives
\[
\Delta q_t(\alpha_t) \;=\; \mathbb{E}_{x_t \sim p_\theta}\!\Big[\mathrm{KL}\big(\tilde p_\theta(x_{t-1} \mid x_t) \,\big\|\, p_\theta(x_{t-1} \mid x_t)\big)\Big] \;=\; \tfrac{1}{2} c_t \alpha_t^2 + o(\alpha_t^2),
\]
where the linear term vanishes because $\alpha_t = 0$ globally minimizes each KL summand (KL is non-negative and zero at $\tilde p_\theta = p_\theta$, so the gradient at zero perturbation vanishes). The coefficient $c_t$ is given by the Jacobian-pullback expression of Lemma~\ref{lem:ct}. Summing across $t$ and combining with the first-order attribute gain $\Delta g_t(\alpha_t) = \alpha_t s_t + o(\alpha_t)$ gives the budgeted allocation problem stated in the main text:
\begin{equation}
\min_{\alpha_t \geq 0}\; \tfrac{1}{2}\sum_{t=0}^{T-1} c_t \alpha_t^2 \quad \text{s.t.} \quad \sum_{t=0}^{T-1} \alpha_t s_t \;\geq\; E_{\mathrm{target}}.
\label{eq:budget-app}
\end{equation}

\subsubsection{Proof of the efficiency ratio theorem}
\label{app:prop1-proof-statement}

\paragraph{Assumption 1 (Finite-horizon Markov reverse process).}
The steered and unsteered reverse processes, denoted by
$\tilde p_\theta$ and $p_\theta$, respectively, are finite-horizon
Markov chains with horizon $T < \infty$, sharing the same fully noised
initial distribution at $x_T$.

\paragraph{Assumption 2 (Smoothness of the attribute score).}
The attribute score function $g$ is twice continuously differentiable
with respect to the steering strength at $\alpha_t = 0$.
Specifically, $g$ is instantiated as the softmax probability of a
classifier evaluated on the realized output sequence, which is smooth
with respect to the residual-stream activation at the intervention layer.

\paragraph{Assumption 3 (Local validity of Taylor expansions).}
There exists a neighborhood around $\alpha_t = 0$ in which
\[
\Delta g_t(\alpha_t)
= \alpha_t s_t + o(\alpha_t),
\qquad
\Delta q_t(\alpha_t)
= \frac{1}{2} c_t \alpha_t^2 + o(\alpha_t^2),
\]
and the operating regime satisfies $\alpha_t = O(\lambda)$ such that
higher-order terms are negligible.

\paragraph{Assumption 4 (Nondegeneracy and feasibility).}
For all $t$, the coefficients satisfy
$s_t \geq 0$ and $c_t > 0$, with
$\sum_t s_t > 0$ and $E_{\mathrm{target}} > 0$.
These conditions ensure feasibility of the constrained optimization
problem in Eq.~\eqref{eq:budget}.

Assumption~A1 justifies the additive cost decomposition
(App.~\S\ref{app:additive-cost}).
Assumptions~A2 and~A3 ensure that $c_t$ is well defined as the second
derivative of the output KL .
Finally, Assumption~A4 guarantees that the associated KKT system admits
a unique nontrivial solution.

\paragraph{Theorem~\ref{thm:eff-ratio} (restated).}
\textit{Adaptive-vs-uniform efficiency ratio.} Consider \eqref{eq:budget-app} with $s_t \geq 0$, $c_t > 0$, and $E_{\mathrm{target}} > 0$. Fix a quality budget $B = \tfrac{1}{2}\sum_t c_t \alpha_t^2$ and let $E^\star$, $E_{\mathrm{unif}}$ denote the maximum attribute shifts achievable at this budget under the optimal and uniform schedules respectively. Then
\begin{equation}
\rho^2 \;:=\; \left(\frac{E^\star}{E_{\mathrm{unif}}}\right)^2
\;=\; \frac{\big(\sum_t s_t^2/c_t\big)\big(\sum_t c_t\big)}{\big(\sum_t s_t\big)^2}
\;=\; 1 \;+\; \mathrm{CV}_c^2(s/c),
\label{eq:eff-ratio-app}
\end{equation}
where $\mathrm{CV}_c$ denotes the coefficient of variation of $s_t/c_t$ under the cost-weighted measure $c_t/\sum_{t'} c_{t'}$. The optimum is achieved by $\alpha_t^\star = \lambda s_t/c_t$ with $\lambda = E_{\mathrm{target}}/\sum_t s_t^2/c_t$. Equality $\rho = 1$ holds iff $s_t/c_t$ is constant in $t$.

\begin{proof}
The proof has three parts: (i) derive the optimum location via KKT, (ii) compute $E_{\mathrm{unif}}$ for the constant-$\alpha$ schedule, and (iii) show the resulting ratio equals $1 + \mathrm{CV}_c^2(s/c)$.

\medskip
\noindent\textbf{Part 1: optimum location.}
The Lagrangian for \eqref{eq:budget-app} is
\[
\mathcal{L} \;=\; \tfrac{1}{2}\sum_t c_t \alpha_t^2 \;-\; \mu\!\left(\sum_t \alpha_t s_t - E_{\mathrm{target}}\right) \;-\; \sum_t \nu_t \alpha_t,
\]
with $\mu \geq 0$ and $\nu_t \geq 0$. Stationarity gives $c_t \alpha_t = \mu s_t + \nu_t$, i.e., $\alpha_t = (\mu s_t + \nu_t)/c_t$.

For $t$ with $s_t > 0$: if $\alpha_t = 0$ then $\nu_t = -\mu s_t \leq 0$, which combined with $\nu_t \geq 0$ forces $\mu s_t = 0$. Since $s_t > 0$ this requires $\mu = 0$, but then all $\alpha_t = 0$, violating $\sum_t \alpha_t s_t \geq E_{\mathrm{target}} > 0$. Therefore $\alpha_t > 0$ on $\{t : s_t > 0\}$, complementary slackness gives $\nu_t = 0$, and $\alpha_t^\star = \mu s_t/c_t$. For $t$ with $s_t = 0$, $\alpha_t^\star = 0$ is consistent with all KKT conditions and contributes nothing to either objective or constraint.

The control constraint binds at the optimum (else scaling all $\alpha_t$ down by $1-\epsilon$ would strictly decrease the objective while remaining feasible), so
\[
\sum_t \mu \frac{s_t^2}{c_t} \;=\; E_{\mathrm{target}} \quad\Longrightarrow\quad \mu^\star \;=\; \frac{E_{\mathrm{target}}}{\sum_t s_t^2/c_t}.
\]
Setting $\lambda := \mu^\star$ recovers the stated $\alpha_t^\star = \lambda s_t/c_t$. The optimal cost is
\[
B^\star \;=\; \tfrac{1}{2}\sum_t c_t (\alpha_t^\star)^2 \;=\; \tfrac{1}{2}\lambda^2 \sum_t \frac{s_t^2}{c_t} \;=\; \frac{E_{\mathrm{target}}^2}{2\sum_t s_t^2/c_t},
\]
and the maximum shift at any budget $B$ is $E^\star = \sqrt{2B \cdot \sum_t s_t^2/c_t}$. Strict convexity of the objective (since $c_t > 0$) gives uniqueness.

\medskip
\noindent\textbf{Part 2: uniform schedule shift.}
Fix the uniform schedule $\alpha_t = \alpha$ for all $t$. The budget constraint becomes
\[
\tfrac{1}{2}\sum_t c_t \alpha^2 \;=\; B \quad\Longrightarrow\quad \alpha \;=\; \sqrt{\frac{2B}{\sum_t c_t}}.
\]
The corresponding attribute shift is
\[
E_{\mathrm{unif}} \;=\; \sum_t \alpha s_t \;=\; \alpha \sum_t s_t \;=\; \sqrt{\frac{2B}{\sum_t c_t}} \cdot \sum_t s_t.
\]

\medskip
\noindent\textbf{Part 3: Adaptive-vs-uniform efficiency ratio and CV identity.}
Taking the ratio of $E^\star$ and $E_{\mathrm{unif}}$ at matched budget $B$,
\[
\rho^2 \;=\; \frac{(E^\star)^2}{(E_{\mathrm{unif}})^2}
\;=\; \frac{2B \cdot \sum_t s_t^2/c_t}{(2B/\sum_t c_t) \cdot (\sum_t s_t)^2}
\;=\; \frac{(\sum_t s_t^2/c_t)(\sum_t c_t)}{(\sum_t s_t)^2}.
\]
This proves the first equality in \eqref{eq:eff-ratio-app}.

For the CV form, define the cost-weighted probability measure $\pi_t = c_t/C$ where $C = \sum_{t'} c_{t'}$, and write $X_t = s_t/c_t$. Under $\pi$,
\[
\mathbb{E}_\pi[X] \;=\; \sum_t \pi_t X_t \;=\; \frac{1}{C}\sum_t s_t, \qquad \mathbb{E}_\pi[X^2] \;=\; \sum_t \pi_t X_t^2 \;=\; \frac{1}{C}\sum_t \frac{s_t^2}{c_t}.
\]
Substituting these into the ratio derived above,
\[
\rho^2 \;=\; \frac{(\sum_t s_t^2/c_t)(\sum_t c_t)}{(\sum_t s_t)^2}
\;=\; \frac{C \cdot \mathbb{E}_\pi[X^2] \cdot C}{(C \cdot \mathbb{E}_\pi[X])^2}
\;=\; \frac{\mathbb{E}_\pi[X^2]}{(\mathbb{E}_\pi[X])^2}.
\]
By the elementary identity $\mathbb{E}[X^2]/\mathbb{E}[X]^2 = 1 + \mathrm{Var}(X)/\mathbb{E}[X]^2 = 1 + \mathrm{CV}^2(X)$,
\[
\rho^2 \;=\; 1 + \mathrm{CV}_c^2(s/c).
\]

\medskip
\noindent\textbf{Equality condition.}
$\rho = 1$ iff $\mathrm{CV}_c(s/c) = 0$ iff $s_t/c_t$ is constant on $\{t : c_t > 0\}$. Equivalently, by Cauchy–Schwarz applied to $\sqrt{c_t}$ and $s_t/\sqrt{c_t}$,
\[
\Big(\sum_t s_t\Big)^2 \;=\; \Big(\sum_t \sqrt{c_t} \cdot \frac{s_t}{\sqrt{c_t}}\Big)^2 \;\leq\; \Big(\sum_t c_t\Big)\Big(\sum_t \frac{s_t^2}{c_t}\Big),
\]
with equality iff $\sqrt{c_t} \propto s_t/\sqrt{c_t}$, i.e., $s_t \propto c_t$. This recovers the same condition.
\end{proof}

\begin{corollary}[Specialization under slowly-varying cost]
\label{cor:constant-c}
When $c_t \approx c$ for all $t$, the efficiency ratio simplifies to
\[
\rho^2 \;\approx\; \frac{T \sum_t s_t^2}{(\sum_t s_t)^2} \;=\; 1 + \mathrm{CV}^2(s),
\]
the squared coefficient of variation of $s_t$ under the uniform measure. The optimum is $\alpha_t^\star \propto s_t$ and the achievable shift at fixed budget is governed by the dispersion of $s_t$ alone. This is the regime in which the proxy of Appendix~\ref{sec:proxy-app} approximates the optimum most directly: when $c_t$ is approximately constant, the proxy's tightness reduces to the alignment between activation growth and $s_t$.
\end{corollary}

\subsubsection{Suboptimality of any feasible schedule}
\label{app:efficiency}

Theorem~\ref{thm:eff-ratio} compares optimal to uniform; the same Cauchy–Schwarz machinery characterizes the suboptimality of any feasible schedule, including the proxy used in our experiments.

\begin{corollary}[Schedule efficiency]
\label{cor:efficiency}
For any schedule $\{\alpha_t \geq 0\}$ that exhausts the budget ($\tfrac{1}{2}\sum_t c_t \alpha_t^2 = B$), the achieved attribute shift satisfies
\[
\frac{E}{E^\star} \;=\; \cos\theta, \qquad \theta \;=\; \angle\!\Big(\big(\alpha_t \sqrt{c_t}\big)_t,\; \big(s_t/\sqrt{c_t}\big)_t\Big),
\]
with equality $E = E^\star$ iff $\alpha_t \propto s_t/c_t$. Equivalently, $B/B^\star = 1/\cos^2\theta$ at matched control.
\end{corollary}

\begin{proof}
Set $u_t = \alpha_t\sqrt{c_t}$ and $v_t = s_t/\sqrt{c_t}$. The budget constraint is $\|u\|^2 = 2B$ and the achieved shift is $E = \langle u, v\rangle$. From Part 1 of the proof of Theorem~\ref{thm:eff-ratio}, $\|v\|^2 = \sum_t s_t^2/c_t$, and the maximum shift at budget $B$ is $E^\star = \sqrt{2B \cdot \sum_t s_t^2/c_t} = \|u\| \cdot \|v\|$. By Cauchy–Schwarz,
\[
E \;=\; \langle u, v\rangle \;\leq\; \|u\|\,\|v\| \;=\; E^\star, \qquad \frac{E}{E^\star} \;=\; \frac{\langle u, v\rangle}{\|u\|\,\|v\|} \;=\; \cos\theta,
\]
with equality iff $u \propto v$, i.e., $\alpha_t \propto s_t/c_t$.

For the cost-ratio form: at matched control level $\bar E$, suppose the heuristic schedule achieves $E = \bar E$ at cost $B$, while the optimal schedule achieves $\bar E$ at cost $B^\star$. From the $E^\star$ formula, $E^\star(B) = \sqrt{2B}\,\|v\|$ and $E^\star(B^\star) = \sqrt{2B^\star}\,\|v\| = \bar E$, so $B^\star = \bar E^2/(2\|v\|^2)$. From the cosine identity, $\bar E = E^\star(B)\cos\theta = \sqrt{2B}\,\|v\|\cos\theta$, so $B = \bar E^2/(2\|v\|^2 \cos^2\theta)$. Dividing,
\[
\frac{B}{B^\star} \;=\; \frac{1}{\cos^2\theta}.
\]
\end{proof}

\begin{corollary}[Cost of using a heuristic schedule shape]
\label{cor:heuristic}
For any schedule $\alpha_t = \kappa w_t$ with $w_t \geq 0$ and $\kappa$ chosen to exhaust the budget,
\[
\frac{E}{E^\star} \;=\; \frac{\sum_t w_t s_t}{\sqrt{\sum_t w_t^2 c_t}\cdot\sqrt{\sum_t s_t^2/c_t}}.
\]
The scalar $\kappa$ cancels: efficiency depends only on the \emph{shape} of $w_t$, not its magnitude.
\end{corollary}

\begin{proof}
The budget constraint $\tfrac{1}{2}\sum_t c_t (\kappa w_t)^2 = B$ gives $\kappa = \sqrt{2B/\sum_t w_t^2 c_t}$. Substituting into Corollary~\ref{cor:efficiency}, the angle factor in the numerator becomes $\langle \kappa w \sqrt{c}, s/\sqrt{c}\rangle = \kappa\sum_t w_t s_t$, and $\|\kappa w \sqrt{c}\| = \kappa\sqrt{\sum_t w_t^2 c_t} = \sqrt{2B}$. The ratio $E/E^\star = \cos\theta$ from Corollary~\ref{cor:efficiency} expands as stated, and $\kappa$ cancels between numerator and denominator.
\end{proof}

This corollary licenses the practice in the experimental section of fixing the shape via $w_{\mathrm{dyn}}$ and sweeping $\alpha$ separately: only the shape determines the proxy's distance from optimal in the Fisher-weighted geometry.

\subsubsection{Operational form of $c_t$}
\label{app:ct-form}

The main text defines $c_t$ as the second derivative of the per-step output-distribution KL at $\alpha_t = 0$. We give an explicit chain-rule expression and verify positivity.

Let $u_t \in \mathbb{R}^d$ denote the steering direction in residual-stream space at step $t$: under composition rule \eqref{eq:steering-features}, $u_t = W_{F_a^\ell}\delta^{(a)}$. Let $f_t : \mathbb{R}^d \to \mathbb{R}^V$ be the map from the residual stream at the intervention layer to the output logits at step $t$, with Jacobian $J_t = \partial f_t/\partial x_t$. Let $F_t^{\mathrm{out}} \in \mathbb{R}^{V\times V}$ denote the Fisher information of the output categorical distribution with respect to its logits, evaluated at the unsteered logits $f_t(x_t)$.

\begin{lemma}
\label{lem:ct}
$c_t = u_t^\top J_t^\top F_t^{\mathrm{out}} J_t u_t$, with $c_t \geq 0$. Strict positivity holds whenever $J_t u_t$ is not in the null space of $F_t^{\mathrm{out}}$.
\end{lemma}

\begin{proof}
Writing the perturbed logits as $\tilde \theta(\alpha_t) = f_t(x_t + \alpha_t u_t)$ and Taylor-expanding the per-step KL,
\[
\mathrm{KL}\!\big(p_{f_t(x_t)} \,\big\|\, p_{\tilde\theta(\alpha_t)}\big) \;=\; \tfrac{1}{2}\alpha_t^2 \cdot (J_t u_t)^\top F_t^{\mathrm{out}} (J_t u_t) + O(\alpha_t^3).
\]
The linear term vanishes because $\alpha_t = 0$ globally minimizes the KL (non-negative everywhere, zero at the origin), forcing the gradient to vanish there. Matching against $\Delta q_t(\alpha_t) = \tfrac{1}{2}c_t \alpha_t^2 + o(\alpha_t^2)$ gives the stated form. Positivity follows from $F_t^{\mathrm{out}} \succeq 0$, with strict inequality when $J_t u_t$ has a non-uniform component (the null space of $F_t^{\mathrm{out}}$ for a categorical distribution is the all-ones direction in logit space, corresponding to logit shifts that leave the softmax invariant).
\end{proof}

\paragraph{Remark.}
$c_t$ is the pullback of the output Fisher through the layers above the intervention site. It is not the Fisher information of a parametric family in the standard sense; calling it ``Fisher information'' without specifying the family would be a misuse of the term. The operational characterization above is what enters the optimization; the chain-rule expression is included for completeness and is not used directly in our schedule. The global hypothesis $c_t > 0$ in Theorem~\ref{thm:eff-ratio} holds operationally because $J_t u_t$ is generically non-uniform (the steering direction $u_t = W_{F_a^\ell}\delta^{(a)}$ has support across many output dimensions, and the layers above the intervention site mix coordinates), placing $J_t u_t$ outside the all-ones null space of $F_t^{\mathrm{out}}$.

\subsubsection{Position-wise decomposition of the local gain}
\label{app:gain-decomp}

The local gain $s_t = \partial g/\partial \alpha_t$ admits a position-wise decomposition into direct and indirect pathways. The direct pathway provably localizes to masked positions; the indirect pathway is non-zero in general but is sub-leading in the regime supported by our measurements.

\paragraph{Setup.}
Fix a denoising trajectory and a step $t$. The steering operation applies a residual-stream perturbation $\alpha_t u_t$ at step $t$, where $u_t = W_{F_a^\ell}\delta^{(a)} \in \mathbb{R}^d$ is the same direction at every sequence position (the SAE injection is position-wise but uses the same $\delta^{(a)}$ across positions). Let $\mathcal{M}_t \subseteq \{1,\ldots,L\}$ denote the set of masked positions at step $t$; we condition on the masking schedule (which positions are unmasked when), so that the only stochasticity is in token content. For unmasked positions $p \notin \mathcal{M}_t$, the realized token $\mathrm{token}_p$ was determined at the step at which $p$ was unmasked, prior to step $t$, and is fixed conditional on the trajectory up to that step.

\begin{lemma}[Direct-effect localization]
\label{lem:gain-decomp}
Under the setup above, the local gain decomposes as
\[
s_t \;=\; \underbrace{\sum_{q}\frac{\partial g}{\partial\,\mathrm{token}_q}\cdot\frac{\partial\,\mathrm{token}_q}{\partial x_{t,q}}\cdot \frac{\partial x_{t,q}}{\partial \alpha_t}}_{\text{direct effect } s_t^{\mathrm{dir}}}
\;+\;
\underbrace{\sum_{q}\frac{\partial g}{\partial\,\mathrm{token}_q}\,\sum_{p \neq q}\frac{\partial\,\mathrm{token}_q}{\partial x_{t,p}}\cdot \frac{\partial x_{t,p}}{\partial \alpha_t}}_{\text{indirect effect } s_t^{\mathrm{ind}}},
\]
where the direct sum collects the $p=q$ diagonal of the position-by-position chain rule and the indirect sum collects the off-diagonal terms. The direct effect reduces to a sum over masked positions:
\[
s_t^{\mathrm{dir}} \;=\; \sum_{q\in\mathcal{M}_t}\frac{\partial g}{\partial\,\mathrm{token}_q}\cdot\frac{\partial\,\mathrm{token}_q}{\partial x_{t,q}}\cdot \frac{\partial x_{t,q}}{\partial \alpha_t}.
\]
\end{lemma}

\begin{proof}
For $q \notin \mathcal{M}_t$, the realized $\mathrm{token}_q$ was determined at the step at which $q$ was unmasked (prior to step $t$) and does not depend on $x_{t,q}$. By the chain rule, $\partial \mathrm{token}_q/\partial x_{t,q} = 0$, so the diagonal $p=q$ term vanishes. The direct-effect sum therefore restricts to $q \in \mathcal{M}_t$.
\end{proof}

\paragraph{Indirect pathway.}
The lemma controls only the direct effect. The indirect effect at unmasked position $p$ is non-zero in general: although $\mathrm{token}_p$ is fixed, the residual-stream activation $x_{t,p}$ enters the key/value projections used by self-attention at masked positions $q \in \mathcal{M}_{t'}$ for $t' \geq t$, and so can influence $\mathrm{token}_q$ at those positions. Whether the indirect path is sub-leading relative to the direct path is a regime-dependent question that we cannot settle in closed form.

We treat direct-path dominance as a modeling assumption supported by the masked-fraction analysis of Appendix~\ref{app:anticipatory_analysis}: across our four models, attribute-relevant features fire preferentially on masked positions throughout the early-to-middle trajectory, with masked-fraction values ranging from 0.55 to 0.80 for anticipatory attributes. This is consistent with — though does not prove — direct-path dominance, since features that fire on masked positions are positioned to exert direct influence on the tokens decoded there. 
The within-model agreement between schedule concentration $\eta$ and observed quality gains (~\ref{sec:experiments}) is consistent with this assumption being operationally adequate, though it does not test the assumption directly.

\subsubsection{Connecting the optimal schedule to the active-set proxy}
\label{sec:proxy-app}

Theorem~\ref{thm:eff-ratio} prescribes $\alpha_t^\star \propto s_t/c_t$, but neither $s_t$ nor $c_t$ is directly observable: $s_t$ requires gradient information about a target attribute classifier on the unrolled denoising trajectory, and $c_t$ requires the Jacobian-pullback of Lemma~\ref{lem:ct}. Our schedule $w_{\mathrm{dyn}}$ uses positive activation growth as a proxy. We justify the proxy through representational commitment under the TopK gate and characterize when it is tight.

\paragraph{Active-set transitions as a measure of representational commitment.}
The TopK gate produces, at each step $t$ and each input, a sparse active set $\mathcal{A}_t \subseteq \{1, \ldots, d_{\mathrm{SAE}}\}$ of features whose pre-gate scores exceed the top-$k$ threshold. The steering operation itself does \emph{not} depend on $\mathcal{A}_t$: the residual-stream perturbation $\alpha\, W_{F_a^\ell}\delta^{(a)}$ is added directly to the post-encode hidden state regardless of which features are currently active. The proxy is therefore not a statement about gradient localization through the gate, but about the model's \emph{use} of attribute-relevant directions. When attribute features $j \in F_a^\ell$ enter $\mathcal{A}_t$ on average across the contrastive corpus, the model is, on average, beginning to compute on those directions; subsequent layers read out from the residual stream, and a perturbation along directions the model is starting to use propagates through more of the downstream computation than a perturbation along directions still effectively unused.

This rationale is intuitive rather than rigorous: we cannot prove that residual-space perturbations along high-growth directions produce larger $\partial g/\partial \alpha_t$ than perturbations along low-growth directions, only that this is plausible given how the model's forward pass distributes information. The framework's predictive accuracy is consistent with results from Section~\ref{sec:experiments}.

\paragraph{Why activation growth, not activation level.}
A natural alternative would be to weight the schedule by raw activation level $A_t^{(a,\ell)} = \sum_{j \in F_a^\ell} \bar h_j^{(t)}$. This over-weights steps where attribute features are already stably active across the corpus: the model has already committed to those features, and perturbing further along directions that are already saturated may have diminishing marginal effect. Activation growth $[\Delta A_t^{(a,\ell)}]_+$ targets steps where features are crossing into the active set on the corpus, where the model's representational state for that attribute is most malleable. The asymmetric clamping at zero is necessary: steps with negative growth correspond to features dropping out of the active set, which we do not want to amplify steering on.

\paragraph{Proxy fidelity and matched-shape efficiency.}
By Corollary~\ref{cor:heuristic}, the efficiency of the proxy schedule relative to optimal is
\[
\frac{E_{\mathrm{proxy}}}{E^\star} \;=\; \frac{\sum_t w_{\mathrm{dyn}}(t) s_t}{\sqrt{\sum_t w_{\mathrm{dyn}}(t)^2 c_t}\cdot\sqrt{\sum_t s_t^2/c_t}} \;=\; \cos\theta_{\mathrm{proxy}},
\]
where $\theta_{\mathrm{proxy}}$ is the angle between $w_{\mathrm{dyn}}(t)\sqrt{c_t}$ and $s_t/\sqrt{c_t}$. The proxy is exactly optimal when (i) active-set growth is proportional to $s_t/c_t$ and (ii) $c_t$ varies slowly relative to $s_t$, so the geometry collapses to the constant-cost case of Corollary~\ref{cor:constant-c}.

\paragraph{When the proxy is loose, and what we observe.}
Two regimes follow from Theorem~\ref{thm:eff-ratio} and Corollary~\ref{cor:heuristic}:
\begin{itemize}
\item \emph{Sharp regime ($\mathrm{CV}_c(s/c)$ large).} The optimal schedule concentrates budget on a few steps, and any schedule whose shape matches that concentration recovers most of $E^\star$. On MDLM, where block 0 captures essentially all of topic emergence, the activation-growth profile is itself sharp, and Adaptive substantially outperforms Uniform at matched control — consistent with the proxy aligning with the sensitivity profile.
\item \emph{Flatter regime ($\mathrm{CV}_c(s/c)$ small).} By Theorem~\ref{thm:eff-ratio}, $\rho \to 1$ as emergence flattens, and the potential gain from any non-uniform schedule shrinks. Conditions with lower $\eta$ (e.g., MDLM sentiment, $\eta = 1.44$) are consistent with the prediction (\S\ref{app:eta_validation}). No condition in our experiments reaches the theoretical $\eta \approx 1$ limit where Adaptive and Uniform would be indistinguishable.
\end{itemize}
The proxy can therefore be loose without being harmful: in the flat regime, any schedule does about equally well; in the sharp regime, the proxy tracks the optimal shape closely enough to recover most of the available efficiency ratio. The regime where the proxy would actively underperform Uniform is one where activation growth and sensitivity are anti-correlated. We do not have direct evidence that this regime is avoided in our four models: our argument that Adaptive does not underperform Uniform is empirical (Table \ref{tab:steering_single_detail} and Table\ref{tab:steering_multi_detail}) rather than structural.

\subsection{Decoder-space interference: bound and analysis}
\label{app:prop2}

We analyze cross-attribute interference under the additive composition \eqref{eq:additive_feature_composition}. The bound below assumes disjoint attribute feature sets; \S\ref{app:gram-values} documents that this assumption holds in only a minority of deployed configurations (20/42 pairs at $|F^\ell_a| = 20$). Where it does not hold, we rely on empirical coupling measurements (also in \S\ref{app:gram-values}), which are roughly an order of magnitude tighter than the bound. We therefore present the bound as a lemma, with the empirical observation as the operative claim, and analyze the gap between them.

\paragraph{Assumption 1 (Disjoint feature sets).}
For all $i \neq j$, the feature sets satisfy
\[
F_{a_i}^\ell \cap F_{a_j}^\ell = \emptyset.
\]
The small-overlap correction for the empirically near-disjoint regime
(deployed top-20 sets share 0--5 features, mean 1.10; see \S\ref{app:gram-values}) is discussed in
App.~\S\ref{app:overlap}.

\paragraph{Assumption 2 (Positive-definite intra-attribute Gram matrices).}
For each attribute $a_i$, the intra-attribute Gram matrix
\[
G^{(ii)}
=
W_{F_{a_i}^\ell}^\top W_{F_{a_i}^\ell}
\]
is positive definite on its support. Equivalently, the decoder columns indexed by $F_{a_i}^\ell$ are linearly independent. This condition generically holds when $|F_{a_i}^\ell| \ll d$ and decoder columns are unit normalized.

\paragraph{Assumption 3 (Nonzero perturbations).}
The perturbations satisfy
\[
\Delta x_t^{(i)} \neq 0,
\qquad
\Delta x_t^{(j)} \neq 0,
\]
ensuring that the cosine term in the bound (Lemma~\ref{lem:gram}) is well defined. This condition holds whenever $\alpha_{\mathrm{eff}}(a_i, t, \ell) > 0$, $\delta^{(a_i)}$ has a nonzero component, and the ReLU clamp does not zero the entire perturbation over $F_{a_i}^\ell$.

Assumption~1 provides the structural precondition of the proposition. Assumption~2 guarantees that the denominator of the bound is strictly positive.
Finally, Assumption~3 ensures that the cosine similarity is well defined.

\subsubsection{Subspace containment and the role of ReLU}
\label{app:subspace}

The bound is stated for post-ReLU residual perturbations. We first establish that disjointness places each attribute's perturbation in its own decoder subspace and that ReLU preserves this structure.

\begin{lemma}[Subspace containment under disjointness]
\label{lem:subspace}
Let $\Delta x_t^{(i)} = \tilde x_t - \tilde x_t^{[\setminus i]}$ denote the residual-space contribution of attribute $i$ under composition rule \eqref{eq:additive_feature_composition}, where $\tilde x_t^{[\setminus i]}$ is the steered reconstruction with attribute $i$ omitted. If $F^\ell_{a_i}\cap F^\ell_{a_j} = \emptyset$ for all $j \neq i$, then $\Delta x_t^{(i)} \in \mathrm{span}(W_{F^\ell_{a_i}})$, and the ReLU clamp in \eqref{eq:additive_feature_composition} alters the magnitudes within $\Delta h_t^{(i)} := \tilde h_t - \tilde h_t^{[\setminus i]}$ but preserves its support on $F^\ell_{a_i}$.
\end{lemma}

\begin{proof}
Write the pre-ReLU composed feature update as $\hat h_{t,j} = h_{t,j} + \sum_{i:\,j\in F^\ell_{a_i}}\alpha_{\mathrm{eff}}(a_i, t, \ell)\delta_j^{(a_i)}$. Under disjointness, the inner sum has at most one term: for any $j$, either $j \in F^\ell_{a_i}$ for exactly one $i$ (call it $i(j)$), or $j$ lies in no attribute set. Consequently $\hat h_{t,j} - h_{t,j}$ depends on attribute $i(j)$ alone.

The ReLU clamp $\tilde h_{t,j} = \mathrm{ReLU}(\hat h_{t,j})$ acts coordinate-wise; its clipping decision at $j$ depends only on $\hat h_{t,j}$, which depends on only one attribute. Therefore $\tilde h_{t,j} - \tilde h_{t,j}^{[\setminus i]}$ is non-zero only for $j \in F^\ell_{a_i}$, regardless of whether ReLU clipped the value. Applying $W_{\mathrm{dec}}$,
\[
\Delta x_t^{(i)} \;=\; W_{\mathrm{dec}}\Delta h_t^{(i)} \;=\; W_{F^\ell_{a_i}}\big(\Delta h_t^{(i)}\big)\big|_{F^\ell_{a_i}} \;\in\; \mathrm{span}\big(W_{F^\ell_{a_i}}\big).
\]
ReLU may shrink magnitudes on $F^\ell_{a_i}$ but cannot move mass off the support.
\end{proof}

\paragraph{What disjointness buys.}
A weaker version of subspace containment — that disjoint feature \emph{indices} produce non-overlapping coordinate updates — follows trivially from the structure of \eqref{eq:additive_feature_composition}. Lemma~\ref{lem:subspace} is stronger: it controls the post-ReLU residual perturbation, not just the feature update, by exploiting that ReLU is coordinate-wise and that disjointness localizes its decision per attribute. This is what makes the Gram bound apply to the actual perturbations rather than only to a first-order linearization.

\subsubsection{Decoder-Gram bound}
\label{app:gram-proof}

\begin{lemma}[Decoder-Gram interference bound]
\label{lem:gram}
Let $G = W_{\mathrm{dec}}^\top W_{\mathrm{dec}}$ and let $G^{(ij)}$ denote its submatrix indexed by $F^\ell_{a_i}\times F^\ell_{a_j}$. If $F^\ell_{a_i}\cap F^\ell_{a_j} = \emptyset$, then for any post-ReLU residual perturbations $\Delta x_t^{(i)}, \Delta x_t^{(j)}$ produced by \eqref{eq:additive_feature_composition},
\begin{equation}
\big|\cos\!\big(\Delta x_t^{(i)},\,\Delta x_t^{(j)}\big)\big| \;\leq\; \frac{\sigma_{\max}(G^{(ij)})}{\sqrt{\sigma_{\min}(G^{(ii)})\,\sigma_{\min}(G^{(jj)})}}.
\label{eq:gram-bound}
\end{equation}
\end{lemma}

\begin{proof}
By Lemma~\ref{lem:subspace}, $\Delta x_t^{(i)} = W_{F^\ell_{a_i}}\Delta h_t^{(i)}$ with $\Delta h_t^{(i)}$ supported on $F^\ell_{a_i}$, similarly for attribute $j$. Write $\delta^{(i)} = (\Delta h_t^{(i)})|_{F^\ell_{a_i}}$. Then
\[
\big\langle \Delta x_t^{(i)},\,\Delta x_t^{(j)}\big\rangle \;=\; \big\langle W_{F^\ell_{a_i}}\delta^{(i)},\,W_{F^\ell_{a_j}}\delta^{(j)}\big\rangle \;=\; \delta^{(i)\top} G^{(ij)}\delta^{(j)}.
\]
By the operator-norm definition of $\sigma_{\max}$,
\[
\big|\delta^{(i)\top} G^{(ij)}\delta^{(j)}\big| \;\leq\; \|\delta^{(i)}\|\cdot\|G^{(ij)}\delta^{(j)}\| \;\leq\; \sigma_{\max}(G^{(ij)})\,\|\delta^{(i)}\|\,\|\delta^{(j)}\|.
\]
For each attribute,
\[
\|\Delta x_t^{(i)}\|^2 \;=\; \delta^{(i)\top} G^{(ii)}\delta^{(i)} \;\geq\; \sigma_{\min}(G^{(ii)})\,\|\delta^{(i)}\|^2,
\]
where positive definiteness of $G^{(ii)}$ on the support of $\delta^{(i)}$ ensures $\sigma_{\min}(G^{(ii)}) > 0$. Combining,
\[
\big|\cos\!\big(\Delta x_t^{(i)},\,\Delta x_t^{(j)}\big)\big| \;=\; \frac{|\delta^{(i)\top} G^{(ij)}\delta^{(j)}|}{\|\Delta x_t^{(i)}\|\,\|\Delta x_t^{(j)}\|} \;\leq\; \frac{\sigma_{\max}(G^{(ij)})}{\sqrt{\sigma_{\min}(G^{(ii)})\,\sigma_{\min}(G^{(jj)})}}.
\]
The bound is independent of $\delta^{(i)}, \delta^{(j)}$, so it holds uniformly over post-ReLU perturbations and is computable a priori from the SAE decoder.
\end{proof}

\subsubsection{Why the bound is small in practice}
\label{app:gram-small}

Three structural properties of SAE-based steering favor a small bound.

\paragraph{Contrastive selection produces near-orthogonal feature sets.}
Features are selected by attribute-specific Cohen's $d$ ($|d_j| \geq 0.2$, $p < 0.01$ via Mann–Whitney). A feature with high $|d|$ for sentiment typically has low $|d|$ for topic and vice versa, because sentiment and topic separate text along weakly correlated semantic axes. The realized overlaps in the deployed top-20 sets range from 0 to 5 shared features (mean 1.10; 20/42 pairs have zero overlap and 9/42 exceed 3 shared features), with the largest overlaps occurring at early layers of the smaller models. See \S\ref{app:gram-values} for the full distribution.

\paragraph{Overcompleteness and unit-norm decoder columns suppress off-diagonal correlation.}
Our SAEs use expansion factors $4\times$ to $16\times$ ($K \in \{12{,}288, 14{,}336, 16{,}384\}$ vs.\ $d \in \{768, 3{,}584, 4{,}096\}$). Decoder columns are normalized to unit norm after each training step. Trained overcomplete dictionaries with $K \gg d$ are observed to exhibit approximate incoherence, with $|\langle w_j, w_k\rangle|$ small for typical $j\neq k$, in line with what \citet{donoho2003optimally} establish for optimal incoherent dictionaries (their result is for random/optimal dictionaries; for trained SAEs the same scaling is empirical, not guaranteed). The Frobenius bound
\[
\sigma_{\max}(G^{(ij)}) \;\leq\; \|G^{(ij)}\|_F \;\leq\; \sqrt{|F^\ell_{a_i}|\cdot|F^\ell_{a_j}|}\cdot O(1/\sqrt{d})
\]
is loose in the rank dimension but captures the qualitative scaling: the bound shrinks as $d$ grows. The diagonal blocks $G^{(ii)}$ have unit diagonal and small off-diagonal entries, so by Gershgorin's theorem $\sigma_{\min}(G^{(ii)}) \geq 1 - (|F^\ell_{a_i}|-1)\cdot O(1/\sqrt{d})$, bounded away from zero whenever $|F^\ell_{a_i}| \ll \sqrt{d}$.

\paragraph{Realized Gram values.}
Appendix~\ref{app:gram-values} reports both the realized Lemma~\ref{lem:gram} bound and the empirical cosine coupling at the deployed steering configuration ($|F^\ell_a| = 20$). The realized coupling is at most \textbf{0.15} across disjoint feature sets (mean \textbf{0.06}), roughly an order of magnitude tighter than the corresponding Frobenius/spectral bounds. A supplementary analysis at $|F^\ell_a| = 50$ (matching the interpretability convention used elsewhere in this appendix) gives similar magnitudes (max \textbf{0.25}, mean \textbf{0.08}).

\paragraph{Contrast with dense steering directions.}
Dense methods (PCA, linear probes, contrastive vectors) steer by adding a single direction $v_a \in \mathbb{R}^d$. For two attributes the cross-cosine is just $\cos(v_{a_1}, v_{a_2})$ — a single number, with no structural mechanism to keep it small. Empirically, these dense-method cross-cosines are substantially larger than the SAE coupling reported in \S\ref{app:gram-values}; the corresponding cross-attribute interference for dense baselines is visible in Figure~\ref{fig:interference}. SAE steering replaces this single coupling with the sparser $\sigma_{\max}(G^{(ij)})$ regime above.

\subsubsection{Small-overlap correction}
\label{app:overlap}

Strict disjointness $F^\ell_{a_i}\cap F^\ell_{a_j} = \emptyset$ is required for Lemma~\ref{lem:subspace}. Contrastive selection produces near-disjointness in most deployed pairs (20 of 42 have zero overlap, mean 1.10 shared features out of $|F^\ell_a| = 20$; \S\ref{app:gram-values}). We characterize the resulting first-order correction.

Let $S_{ij} = F^\ell_{a_i}\cap F^\ell_{a_j}$. On $S_{ij}$, the pre-ReLU update at coordinate $j \in S_{ij}$ depends on both attributes: $\hat h_{t,j} = h_{t,j} + \alpha_{\mathrm{eff}}(a_i, t, \ell)\delta_j^{(a_i)} + \alpha_{\mathrm{eff}}(a_j, t, \ell)\delta_j^{(a_j)}$. ReLU's clipping decision at $j$ depends on the sum, introducing genuine cross-attribute coupling. Decompose
\[
\Delta x_t^{(i)} \;=\; \Delta x_t^{(i,\mathrm{disjoint})} + \Delta x_t^{(i,\mathrm{overlap})},
\]
where the first term is the contribution from $F^\ell_{a_i}\setminus S_{ij}$ (covered by Lemma~\ref{lem:subspace}) and the second is from $S_{ij}$. The Gram bound applies to the disjoint part. The overlap part contributes
\[
\big\|\Delta x_t^{(i,\mathrm{overlap})}\big\| \;\leq\; \sigma_{\max}(W_{S_{ij}})\cdot\big\|\Delta h_t^{(i)}\big|_{S_{ij}}\big\|,
\]
which is bounded by the same $O(1/\sqrt{d})$ scaling as the off-diagonal Gram entries times the small factor $|S_{ij}|/|F^\ell_{a_i}|$. The overlap fraction $|S_{ij}|/|F^\ell_{a_i}|$ is at most $5/20 = 0.25$ and typically much smaller (average overlap fraction $1.10/20 = 0.055$), well within the correction regime for the disjoint-case bound. Realized coupling values are reported in Figure~\ref{fig:interference}.

\subsection{Decoder Gram structure and attribute coupling}
\label{app:gram-values}

Lemma~\ref{lem:gram} is a \emph{conditional} statement: it assumes $F_{a_i}^\ell \cap F_{a_j}^\ell = \emptyset$ (disjoint feature sets), and under that assumption bounds cross-attribute coupling. It is not a claim that the deployed attribute feature sets are disjoint. Where the precondition holds, the bound is informative; where it does not, interference control is empirical (measured coupling, this section). This appendix reports both the bound and the empirical coupling, at both the deployed $|F^\ell_a| = 20$ (matching \S\ref{sec:steering}'s steering recipe) and the interpretability-standard $|F^\ell_a| = 50$ used elsewhere in this appendix.

\paragraph{Setup.}
For each model–layer pair we extract the decoder submatrices $W_{F_i}$ and $W_{F_j}$ for the sentiment, topic, and style attributes used at that model, compute the inter-attribute Gram submatrix $G^{(ij)} = W_{F_i}^\top W_{F_j}$ and intra-attribute matrices $G^{(ii)}, G^{(jj)}$, and report:
\begin{enumerate}[label=(\roman*)]
\item the \emph{theoretical bound} $\sigma_{\max}(G^{(ij)})/\sqrt{\sigma_{\min}(G^{(ii)})\,\sigma_{\min}(G^{(jj)})}$ from Lemma~\ref{lem:gram},
\item the \emph{realized cosine coupling} $|\cos(\Delta x^{(i)},\Delta x^{(j)})|$ where $\Delta x^{(i)} = \sum_k \delta_k^{(a_i)} \cdot W_{\mathrm{dec}}[F_{a_i}[k]]$ uses the per-feature contrastive shifts $\delta^{(a)} = \bar{h}_{+}^{(a)} - \bar{h}_{-}^{(a)}$ actually applied at inference, and
\item the \emph{disjoint variant} $|\cos|_{\mathrm{disj}}$, which restricts each attribute to $F_{a_i}\setminus F_{a_j}$ so that the Lemma~\ref{lem:gram} precondition holds exactly.
\end{enumerate}

\paragraph{Overlap and Lemma~\ref{lem:gram} status (deployed, $|F^\ell_a| = 20$).}
Across 42 model$\times$layer$\times$pair combinations, the top-20 sets share 0--5 features (mean 1.10; 20/42 pairs have zero overlap). The Lemma~\ref{lem:gram} bound is vacuous (RHS $\geq 1$) in 36 of 42 deployed configurations; the 6 non-vacuous cases (RHS $< 1$) all have zero overlap. This is exactly what the conditional statement of Lemma~\ref{lem:gram} predicts: informative bounds appear precisely where the precondition (disjointness) is met (Spearman $\rho = 0.75$, $p \approx 10^{-8}$ between shared-feature count and RHS). In the majority of deployed configurations the precondition does not hold, and the interference claim rests on the empirical coupling measurements below, not on the analytic bound.

\paragraph{Empirical coupling is small (deployed).}
The realized cosine coupling is small across all models and layers:
\begin{itemize}
\item \textbf{With overlap (raw top-20):} $\max|\cos| = 0.734$ (MDLM L5 sentiment--formality, 3 shared features), mean $= 0.121$.
\item \textbf{Disjoint feature sets:} $\max|\cos|_{\mathrm{disj}} = 0.145$ (SEDD L6 sentiment--formality), mean $= 0.057$.
\end{itemize}
Removing shared features cuts the maximum coupling by $5\times$. Low geometric coupling does not by itself rule out behavioral interference: non-target attributes can still shift through downstream computation and through natural correlations between attributes (e.g., positive sentiment and formality). Behavioral interference is reported separately in Table~\ref{tab:steering_single_detail} and Table ~\ref{tab:steering_multi_detail}.

\paragraph{Sensitivity to $|F^\ell_a| = 50$.}
For consistency with the interpretability characterizations in \S\ref{sec:appendix_sigma}, we also report the same quantities at top-50. Overlap grows to 0--14 (mean 3.4), reflecting that the top-50 sets are more permissive. All 42 Lemma~\ref{lem:gram} bounds are vacuous at this scale. Empirical coupling is qualitatively similar: $\max|\cos|_{\mathrm{disj}} = 0.252$ (MDLM L7 sentiment--topic), mean $= 0.075$. The deployed $|F^\ell_a| = 20$ view above is the operationally relevant measurement; the top-50 view characterizes the underlying SAE geometry at the scale used for other structural analyses.

\paragraph{Interpretation of the gap between bound and reality.}
Lemma~\ref{lem:gram} provides a valid but conservative bound: it guarantees that coupling cannot exceed a computable threshold under any post-ReLU perturbations consistent with disjoint feature sets. The empirical coupling is roughly an order of magnitude tighter. This gap is consistent with contrastive shifts $\delta^{(a)}$ avoiding the worst-case directions in the off-diagonal Gram structure i.e., contrastive selection finds attribute-aligned directions that are also geometrically near-orthogonal in the residual stream, beyond what disjointness alone guarantees. Lemma~\ref{lem:gram} characterizes \emph{why} disjoint feature sets contain interference (subspace containment); it is not a precondition for the additive composition, which operates regardless of overlap. In deployment, where sets are mostly not fully disjoint, interference is controlled empirically rather than by the analytic bound.

\subsection{Empirical Validation of Theorem~\ref{thm:eff-ratio}: $\eta$ vs Quality Gain}
\label{app:eta_validation}

Theorem~\ref{thm:eff-ratio} predicts that the adaptive-vs-uniform quality gain grows with the schedule concentration $\eta = N \sum_i w_i^2$, where $w_i$ are normalized block fractions and $N$ is the number of blocks. $\eta = 1$ corresponds to a perfectly uniform emergence profile (Adaptive $\approx$ Uniform), while $\eta = N$ indicates full concentration in a single block (maximum adaptive advantage). We test this prediction empirically by computing $\eta$ from each model's emergence schedule and comparing it against the measured quality gain of Adaptive over Uniform steering.

\textbf{Setup.} For each of the 12 (model, attribute, primary-layer) conditions, we compute $\eta$ from the $B{=}8$ block fractions used for steering, and measure the Adaptive-vs-Uniform quality gain as $(\text{Adaptive} - \text{Uniform}) / \text{Uniform}$, averaged over $\alpha \in [1, 15]$. We report two quality metrics: LLMJ fluency (a 1--5 LLM-as-judge score) and Qwen2-7B perplexity (PPL).

\textbf{Results.}
Table~\ref{tab:eta_conditions} reports all 12 conditions. %

\begin{table}[h]
\centering
\caption{Schedule concentration $\eta$ and Adaptive-vs-Uniform quality gains for all 12 conditions (4 models $\times$ 3 attributes), averaged over $\alpha \in [1, 15]$.}
\label{tab:eta_conditions}
\small
\begin{tabular}{llcrr}
\toprule
\textbf{Model} & \textbf{Attribute} & $\eta$ & \textbf{LLMJ flu.} & \textbf{Qwen2 PPL} \\
\midrule
MDLM & sentiment & 1.44 & $+7.8\%$ & $+29.0\%$ \\
MDLM & style     & 1.57 & $+28.7\%$ & $+51.2\%$ \\
MDLM & topic     & 2.58 & $+10.9\%$ & $+64.6\%$ \\
\midrule
SEDD & sentiment & 3.18 & $+8.0\%$ & $+28.4\%$ \\
SEDD & style     & 1.72 & $+9.0\%$ & $+26.4\%$ \\
SEDD & topic     & 3.51 & $+12.5\%$ & $+53.7\%$ \\
\midrule
LLaDA & sentiment & 1.68 & $+0.5\%$ & $-8.3\%$ \\
LLaDA & style     & 3.67 & $+15.7\%$ & $-19.9\%$ \\
LLaDA & topic     & 3.31 & $+13.8\%$ & $-12.7\%$ \\
\midrule
DREAM & sentiment & 4.72 & $+1.6\%$ & $+6.0\%$ \\
DREAM & style     & 4.09 & $+7.5\%$ & $+27.5\%$ \\
DREAM & topic     & 5.01 & $+16.3\%$ & $+29.8\%$ \\
\bottomrule
\end{tabular}
\end{table}

\textbf{Within-model magnitude ordering.}
Theorem~\ref{thm:eff-ratio} predicts that, within a given model, attributes with higher $\eta$ should show larger adaptive gains. We test this via Spearman rank correlation between $\eta$ and each quality gain across the three attributes per model ($n{=}3$). On Qwen2 PPL, MDLM and SEDD both show perfect within-model rank agreement ($r_s = +1.0$): the attribute with the most concentrated schedule yields the largest PPL improvement from Adaptive steering. However, LLaDA's Qwen2 PPL gains are negative across all three attributes ($-8\%$ to $-20\%$), because aggressive Adaptive steering on LLaDA induces repetitive text that achieves spuriously low Uniform PPL, inverting the metric (\S\ref{sec:reporting-rule}). For LLaDA, LLMJ fluency is the appropriate quality indicator, and under this metric $\eta$ perfectly ranks all three attributes ($r_s = +1.000$). DREAM shows $r_s = +0.5$ on both metrics. With three attributes per model, no single within-model correlation can reach significance (the minimum one-sided p for a perfect ranking is 1/6). Pooling across models with an exact permutation test on the sum of $r_s$, the probability under random rankings of observing a sum $\geq$ 3.5 (three perfect orderings and one $r_s$=0.5) is 9/1296 $\approx$ 0.007.

\section{Vocabulary Grounding of Top Contrastive Features}
\label{app:feature_vocab_sparsity}

To qualitatively verify that contrastive SAE features encode semantically coherent concepts, we project each feature's decoder column through the model's unembedding matrix and inspect the highest-scoring vocabulary tokens. Concretely, for a feature~$j$ with learned decoder vector $\mathbf{w}_j^\text{dec} \in \mathbb{R}^{d}$, we compute the logit-space projection $\mathbf{w}_j^\text{dec} \cdot \mathbf{W}_\text{unembed}^\top \in \mathbb{R}^{|\mathcal{V}|}$ and report the tokens with the largest scores. Intuitively, these are the vocabulary items that the feature most strongly promotes in the model's output distribution when active.

For each attribute (topic, formality/informality, sentiment), we select the top~3 features in each direction by Cohen's $d$: the 5 positive-direction features with the largest $d > 0$ (more active for the target class) and the 5 negative-direction features with the most negative $d < 0$ (more active for the contrast class). In Table \ref{tab:feature_vocab_all}, we report results at the deepest SAE layer for each model: layer~7 for MDLM-124M and SEDD-124M, layer~26 for LLaDA-8B, and layer~23 for DREAM-7B. MDLM and SEDD use the GPT-2 tokenizer (50K English tokens); LLaDA and DREAM use multilingual vocabularies (126K and 152K tokens, respectively), so some features project onto non-English tokens (Chinese glosses are provided where applicable).

\input{figures_tables/table_feature_vocab_sparsity_all_merged}

\section{Detailed Experimental Setup}
\label{sec:experimental_setup}

\subsection{Training SAEs on Discrete Diffusion Models}
\label{sec:sae_training}

We train TopK sparse autoencoders~\citep{cunningham2023sparse, zhu2025abstopk} on residual-stream activations from MDLM~\citep{sahoo2024mdlm}, SEDD~\citep{lou2024sedd}, LLaDA~\citep{nie2025llada}, and DREAM~\citep{dream2025}. We proposed a diffusion logits lens (\S\ref{sec:logit_lens}) and conducted a layer probing study (\S\ref{sec:layer_study}), which together reveal when and where token identities crystallize during generation and identify the layers most suitable for SAE training in DLMs. Given a hidden state $\mathbf{x} \in \mathbb{R}^{d_\text{model}}$, the SAE encodes it into an overcomplete latent space ($d_\text{SAE} \gg d_\text{model}$), applies a $\mathrm{TopK}$ activation to retain only the $k$ largest entries as a sparse code $\mathbf{h}$, and reconstructs via a learned decoder:
$\mathbf{h} = \mathrm{TopK}\!\bigl((\mathbf{x} - \mathbf{b}_\text{dec})\, \mathbf{W}_\text{enc} + \mathbf{b}_\text{enc},\; k\bigr),\;
\hat{\mathbf{x}} = \mathbf{h}\, \mathbf{W}_\text{dec} + \mathbf{b}_\text{dec}.$
Each nonzero entry $h_j$ corresponds to a \emph{feature} whose decoder direction $\mathbf{w}_j^\text{dec}$ (the $j$-th column of $\mathbf{W}_\text{dec}$) is the direction it adds to the residual stream. 
Training minimizes MSE reconstruction loss with an auxiliary dead-neuron loss (See Appendix\ref{sec:experimental_setup} for details). Unlike AR models, DLM activations vary with masking rate, so we train on activations sampled across uniformly distributed masking rates to ensure generalization across the full denoising trajectory. Table~\ref{tab:model_summary} summarizes all models and SAE configurations.

\begin{table}[t!]
\centering
\caption{Model and SAE summary. All SAEs use TopK activation with $k{=}32$.}
\label{tab:model_summary}
\resizebox{0.85\columnwidth}{!}{%
\small
\begin{tabular}{lcccccccc}
\toprule
\textbf{Model} & \textbf{Params} & \textbf{Arch.} & \textbf{Diffusion Loss} & \textbf{SAE Layers} & $d_\text{model}$ & $d_\text{SAE}$ & $d_\text{SAE}/d_\text{model}$ \\
\midrule
MDLM~\citep{sahoo2024mdlm}  & 124M & GPT-2~\citep{radford2019language} & Absorbing      & 5, 6, 7       & 768   & 12{,}288 & 16$\times$ \\
SEDD~\citep{lou2024sedd}  & 124M & GPT-2 & Score-entropy  & 5, 6, 7       & 768   & 12{,}288 & 16$\times$ \\
LLaDA~\citep{nie2025llada} & 8B   & LLaMA~\citep{touvron2023llama} & Absorbing      & 8, 14, 20, 26 & 4{,}096 & 16{,}384 & 4$\times$ \\
DREAM~\citep{dream2025} & 7B   & Qwen2~\citep{yang2024qwen2} & Absorbing      & 8, 13, 17, 23 & 3{,}584 & 14{,}336 & 4$\times$ \\
\bottomrule
\end{tabular}%
}
\end{table}

\subsection{Extracting Attribute Features from SAEs}
\label{sec:feature_id}
We study three attributes: \textbf{sentiment}, \textbf{topic}, and \textbf{style}, common controllability targets and a natural combination for studying multi-attribute interplay, as they capture orthogonal aspects of text: subject matter, opinion polarity, and writing register. Following work on contrastive steering directions~\citep{zou2023representation, rimsky-etal-2024-steering, belinkov-2022-probing}, we construct balanced binary corpora: Sports vs.\ Business from AG News~\citep{zhang2015character} for topic, positive vs.\ negative IMDB reviews~\citep{maas-etal-2011-learning} for sentiment, and formal vs.\ informal text for style, using the Pavlick corpus~\citep{pavlick2016empirical} for MDLM/SEDD and classifier-labeled IMDB reviews~\citep{babakov2023don} for LLaDA/DREAM to match their generation domain (\S\ref{app:extraction_datasets}). For style, we track formality features on MDLM/SEDD, whose baseline formality is near chance (${\sim}$55--58\%), and informality features on LLaDA/DREAM, whose generations are already highly formal (${\sim}$73--78\%; \S\ref{sec:experiments}), so that the features we study correspond to the direction with greater room for control in later steering studies (\S\ref{sec:steering}). To identify attribute-encoding features, we noise each text at multiple mask ratios, pass the noised inputs through the model, and ReLU-encode the hidden states through the SAE at each layer (using ReLU rather than TopK to avoid biasing effect-size estimates; \S\ref{app:extraction_procedure}). For each feature $j$, we compute the Cohen's $d$ effect size~\citep{cohen2013statistical}: $d_j = (\bar{h}_j^{+} - \bar{h}_j^{-}) / \sqrt{(\text{Var}(h_j^{+}) + \text{Var}(h_j^{-}))/2}$, where $\bar{h}_j^{+}$ and $\bar{h}_j^{-}$ are the mean activations for the two contrastive classes. We retain the top 50 features per direction per layer that pass a Mann--Whitney $U$ test~\citep{mann1947test} ($p < 0.01$, Bonferroni-corrected~\citep{bonferroni1936teoria}) for the interpretability analysis below. Vocabulary grounding (projecting decoder columns through $\mathbf{W}_\text{unembed}$) confirms that retained features are semantically coherent (see Appendix~\ref{app:feature_vocab_sparsity} for details).

\subsection{Models and SAE Training}
\label{sec:sae_training_model}

We study four DLMs that span two training objectives, two scales, and two provenance strategies:
MDLM~\citep{sahoo2024mdlm} and SEDD~\citep{lou2024sedd} are 124M-parameter GPT-2 architectures trained on OpenWebText with absorbing-state and score-entropy objectives, respectively---sharing architecture and data but differing only in loss function, forming a natural experiment for isolating the effect of training objective on learned representations.
Dream~\citep{dream2025} (7B) is fine-tuned from Qwen2.5 with masked diffusion, and LLaDA~\citep{nie2025llada} (8B) is trained from scratch with a masked absorbing objective, providing a contrast in model provenance at the 7--8B scale.

For each model, we train TopK sparse autoencoders~\citep{makhzani2013k} on residual-stream activations extracted from middle-to-late transformer layers via forward hooks.
The SAE encodes an input activation $\mathbf{x} \in \mathbb{R}^{d}$ into a sparse latent representation $\mathbf{h} \in \mathbb{R}^{D}$ ($D \gg d$) and reconstructs:
\begin{align}
    \mathbf{h} &= \mathrm{TopK}\bigl((\mathbf{x} - \mathbf{b}_\text{dec}) \mathbf{W}_\text{enc} + \mathbf{b}_\text{enc},\; k\bigr), \label{eq:encode} \\
    \hat{\mathbf{x}} &= \mathbf{h}\, \mathbf{W}_\text{dec} + \mathbf{b}_\text{dec}, \label{eq:decode}
\end{align}
where $\mathrm{TopK}$ retains only the $k$ largest activations and zeros the rest, providing an exact sparsity guarantee.
We use $k{=}32$ for all models.
Decoder columns are normalized to unit norm after each training step.
We train with MSE reconstruction loss plus an auxiliary loss ($\lambda{=}1/32$) that encourages dead neurons (inactive for ${>}10{,}000$ steps) to reconstruct the current residual error.
All SAEs are trained for 40-80k steps with AdamW (lr $10^{-4}$, cosine schedule to $10^{-5}$) on batches of 4,096 activations. All hyperparameters used for training SAE can be found in Table~\ref{tab:sae_hparams}.

\paragraph{DLM-specific design choice: noise-level diversity.}
Unlike autoregressive models where each token has a single context, DLM activations depend on the noise level $\sigma$ (or equivalently, the fraction of masked tokens). We extract activations across uniformly sampled noise levels during training, ensuring the SAE learns features that are robust across the full denoising trajectory rather than specialized to a single noise regime. This is validated by our sigma-stratified analysis (\S\ref{sec:appendix_sigma}), which shows feature activations vary by less than 15\% across noise levels.

\subsection{Prediction Formation via Diffusion Logit Lens}
\label{sec:logit_lens}

We introduce the first \emph{Diffusion Logit Lens}, adapting the logit lens technique~\citep{belrose2023eliciting} to discrete diffusion language models. By projecting intermediate-layer hidden states through the unembedding matrix at each denoising step, we construct a two-dimensional \emph{prediction formation surface} (layer $\times$ timestep) that reveals when and where token identities crystallize during generation.

\begin{figure}[h!]
\centering
\includegraphics[width=\textwidth]{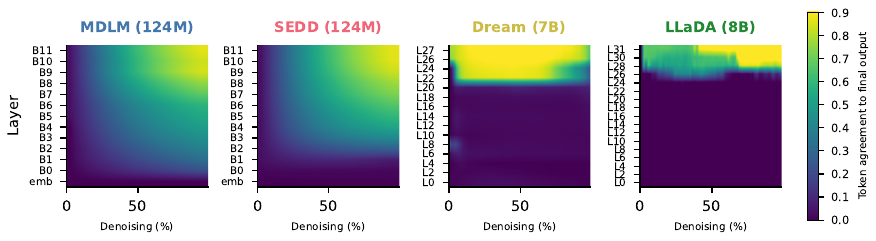}
\caption{Prediction formation via logit lens across four diffusion language models. Each panel shows token agreement between intermediate-layer representations and the final output, as a function of layer depth and denoising progress. Small models (MDLM, SEDD) form predictions gradually across layers and time, while large models (DREAM, LLaDA) exhibit sharp layer boundaries where predictions crystallize abruptly.}
\label{fig:prediction_formation}
\end{figure}

\subsubsection{Setup}
\label{sec:logit_lens_setup}

For each model, we generate samples using the standard denoising schedule, hooking into every transformer block at each step. At layer $\ell$ and step $t$, we extract the hidden state $\mathbf{h}_t^\ell \in \mathbb{R}^{d_\text{model}}$ and project it into vocabulary space:
\begin{equation}
\hat{y}_t^\ell = \arg\max\; \mathbf{W}_\text{unembed}\, \mathbf{h}_t^\ell,
\label{eq:logit_lens}
\end{equation}
comparing $\hat{y}_t^\ell$ against the final generated token at each position. We report \emph{agreement}: the fraction of positions where the intermediate prediction matches the final output. For models with explicit noise conditioning (MDLM, SEDD), we additionally evaluate a \emph{conditioned} lens that applies the full output head including adaptive layer normalization (adaLN) with noise-level conditioning. Table~\ref{tab:logit_lens_setup} summarizes the experimental configuration.

\begin{table}[h]
\centering
\caption{Logit lens experimental configuration per model.}
\label{tab:logit_lens_setup}
\small
\begin{tabular}{lccccl}
\toprule
\textbf{Model} & \textbf{Samples} & \textbf{Steps} & \textbf{Layers} & \textbf{Lens Type} & \textbf{Notes} \\
\midrule
MDLM-124M  & 200 & 128  & 13 & Raw + Conditioned & adaLN sigma conditioning \\
SEDD-124M  & 200 & 256  & 13 & Raw + Conditioned & Score-entropy output head \\
DREAM-7B   & 100 & 64   & 15 & Raw               & AR-style logit shift \\
LLaDA-8B   &  64 & 64   & 17 & Raw               & No noise conditioning \\
\bottomrule
\end{tabular}
\end{table}

\subsubsection{Finding 1: Small Models Show Gradual Diagonal Crystallization}
\label{sec:logit_lens_gradual}

MDLM-124M's conditioned logit lens reveals a smooth diagonal gradient in the (layer, timestep) space. Agreement with the final output increases gradually across both dimensions simultaneously, reaching 86.9\% at the final layer and step.

\begin{table}[h]
\centering
\caption{MDLM-124M conditioned lens agreement (\%) at selected layers and steps.}
\label{tab:mdlm_agreement}
\small
\begin{tabular}{lccc}
\toprule
\textbf{Layer} & \textbf{Step 0} & \textbf{Step 64} & \textbf{Step 127} \\
\midrule
Block 0  & 0.02 & 11.3 & 20.9 \\
Block 5  & 3.8  & 37.9 & 54.9 \\
Block 8  & 3.8  & 53.3 & 77.4 \\
Block 11 & 3.6  & 60.1 & \textbf{86.9} \\
\bottomrule
\end{tabular}
\end{table}

Every layer contributes incrementally to prediction formation---there is no single ``decision layer.'' The prediction surface is truly two-dimensional: both more denoising steps \emph{and} deeper layers are jointly required for accurate predictions. This distributed processing is consistent with MDLM's relatively shallow architecture (12 blocks), which forces all layers to contribute.

\paragraph{Sigma conditioning is essential.}
The raw logit lens (projecting hidden states directly through $\mathbf{W}_\text{unembed}$ without noise-level conditioning) achieves a maximum of only ${\sim}$8\% agreement on MDLM. The conditioned lens (applying the full output layer with adaLN sigma modulation) reaches 87\%. This 10$\times$ gap proves that sigma conditioning is not merely a scaling factor---it fundamentally transforms the representation space. Without knowing the noise level, intermediate hidden states are essentially unreadable. The adaLN mechanism acts as a ``decoder ring'' that makes the residual stream interpretable at each noise level.

\subsubsection{Finding 2: Masked vs.\ Unmasked Position Dynamics}
\label{sec:logit_lens_positions}

Separating agreement by position type reveals distinct dynamics:

\paragraph{MDLM.} Unmasked (already-revealed) positions are correctly predicted by block~5 onward ($>$80\% agreement)---these are trivially predictable since the token is visible. Masked positions carry the interesting signal: their agreement rises from 3.6\% (step~0) to 53.1\% (final step) at block~11, indicating that the model progressively ``plans'' future tokens at still-masked positions. This anticipatory encoding is the mechanistic basis for why SAE features fire on masked positions before content commitment (\S\ref{sec:temporal_dynamics}).

\paragraph{DREAM.} Unmasked agreement peaks at 97.6\% (step~32) but \emph{degrades} to 66.4\% (step~63), mirroring the V-shaped entropy. Even the model's ability to predict already-visible tokens deteriorates at low mask ratios, confirming the out-of-distribution hypothesis.

\subsubsection{Cross-Model Summary}
\label{sec:logit_lens_summary}

\begin{table}[h]
\centering
\caption{Cross-model comparison of prediction formation properties.}
\label{tab:logit_lens_summary}
\small
\begin{tabular}{lcccc}
\toprule
\textbf{Property} & \textbf{MDLM} & \textbf{SEDD} & \textbf{DREAM} & \textbf{LLaDA} \\
\midrule
Crystallization & Gradual diagonal & Gradual diagonal & Sharp (L22) & Sharp (L26) \\
Layer contribution & All layers & All layers & L22+ only & L26+ only \\
Raw lens max & 8\% & 72\% & 93\% & 100\% \\
Entropy trajectory & Monotonic $\downarrow$ & Monotonic $\downarrow$ & V-shaped & Monotonic $\downarrow$ \\
Sigma conditioning needed & Yes (essential) & Yes (but raw works) & N/A & N/A \\
Max agreement & 87\% & 86\% & 93\% & 100\% \\
\bottomrule
\end{tabular}
\end{table}

The logit lens reveals a fundamental dichotomy in how DLMs form predictions: \textbf{small models crystallize gradually} across a 2D surface (requiring both depth and denoising progress), while \textbf{large models exhibit sharp phase transitions} at specific layers. The training objective further modulates readability---score-entropy (SEDD) produces representations that are $9\times$ more readable without conditioning than absorbing-state (MDLM). These macroscopic patterns have direct implications for SAE-based steering: features in pre-crystallization layers (e.g., DREAM L8--L20) encode abstract semantic information without token commitment, making them natural intervention points for high-level attribute control without disrupting low-level token selection.

\subsection{SAE Layer Selection}
\label{app:layer_selection}

Selecting which transformer layers to train SAEs on is non-trivial: in autoregressive models, middle layers are known to balance representational richness with causal malleability~\citep{zou2023representation}, but this principle has never been validated for discrete diffusion transformers.
We describe the systematic methodology used and summarize the results for each model.

\subsubsection{Methodology}

We apply a four-stage protocol uniformly across all models:

\begin{enumerate}
    \item \textbf{Linear probing.} Logistic regression with 5-fold stratified cross-validation on mean-pooled activations from every transformer layer. We probe two attributes: sentiment (IMDB, 2-class) and topic (AG News, 4-class). To account for the stochastic nature of diffusion representations, each text is passed through the model with 3 independent noise samples (mask ratio 0.5), yielding $3\times$ the effective dataset size (6,000 data points per layer from 2,000 source texts).

    \item \textbf{SAE training.} TopK SAEs are trained on candidate layers with identical hyperparameters per model (Table~\ref{tab:sae_hparams}), ensuring fair comparison. We monitor reconstruction loss and dead neuron fraction as quality indicators.

    \item \textbf{Contrastive feature identification.} Per-feature Cohen's $d$ between attribute groups, retaining features with $|d| > 0.05$ and $p < 0.01$. This reveals which layers produce the strongest and most numerous attribute-discriminative features.

    \item \textbf{Steering ablation.} The decisive test: we compare steering effectiveness (positive-rate swing across an $\alpha$ sweep) for different layer windows, isolating the causal contribution of each layer group.
\end{enumerate}

\subsubsection{Layer Probing Results}
\label{sec:layer_study}
Table~\ref{tab:layer_probing} reports sentiment probing accuracy at selected layers for all four models.

\begin{table}[h]
\centering
\small
\caption{Linear probing accuracy (sentiment, 5-fold CV) at selected layers. Layers in \textbf{bold} are those selected for SAE training. MDLM and SEDD share the GPT-2 architecture (12 layers); Dream has 28 layers (Qwen2); LLaDA has 32 layers (LLaMA).}
\label{tab:layer_probing}
\begin{tabular}{rcccc}
\toprule
\multirow{2}{*}{Layer} & MDLM & SEDD & Dream & LLaDA \\
 & (12L, 124M) & (12L, 124M) & (28L, 7B) & (32L, 8B) \\
\midrule
0  & 80.9 & 81.5 & 88.5 & 86.4 \\
4  & 82.2 & 77.8 & 90.7 & 88.1 \\
\textbf{5}  & \textbf{82.7} & \textbf{78.6} & --- & --- \\
\textbf{6}  & \textbf{83.1} & \textbf{79.9} & --- & --- \\
\textbf{7}  & \textbf{84.0} & \textbf{81.3} & --- & --- \\
\textbf{8}  & 84.7$^\dagger$ & {82.1} & \textbf{95.0} & \textbf{91.4} \\
\textbf{9}  & 84.4 & {81.8} & --- & --- \\
\textbf{10} & 84.0 & {82.4} & --- & --- \\
11 & 84.4 & 83.2$^\dagger$ & --- & --- \\
\textbf{13} & --- & --- & \textbf{97.6} & --- \\
\textbf{14} & --- & --- & --- & \textbf{95.5} \\
\textbf{17} & --- & --- & \textbf{98.2}$^\dagger$ & --- \\
\textbf{20} & --- & --- & 98.0 & \textbf{97.3} \\
\textbf{23} & --- & --- & \textbf{98.0} & 98.1$^\dagger$ \\
\textbf{26} & --- & --- & --- & \textbf{98.0} \\
27 & --- & --- & 98.0 & --- \\
31 & --- & --- & --- & 97.9 \\
\bottomrule
\multicolumn{5}{l}{\footnotesize $^\dagger$Peak probing accuracy for that model.}
\end{tabular}
\end{table}

Three patterns emerge:
\begin{itemize}
    \item \textbf{Monotonic rise then plateau.} All four models show increasing probing accuracy with depth, saturating in the upper layers. MDLM plateaus at L8--11 ($\sim$84\%); Dream and LLaDA plateau earlier in relative terms ($\sim$L13/L17 and $\sim$L20/L23 respectively), with much higher absolute accuracy ($>$97\%) reflecting the larger model capacity.
    \item \textbf{SEDD dips before rising.} Unlike MDLM, SEDD shows a characteristic dip at L1--5 (77--79\%) below the L0 baseline (81.5\%), before rising through L6--11. This U-shaped profile suggests that score-entropy diffusion initially disrupts the pretrained embeddings' linear separability before rebuilding it in later layers.
    \item \textbf{Scale raises the floor.} The 7--8B models (Dream, LLaDA) start at 86--89\% even at L0, compared to 81\% for the 124M models. Richer input embeddings (Qwen2.5 / LLaMA) provide a higher baseline that subsequent layers refine.
\end{itemize}

Topic probing follows a similar profile but peaks later: MDLM at L11 (79.1\%), Dream at L23 (95.3\%), LLaDA at L28 (95.7\%). This offset is consistent with topic being a more global, multi-token property that requires deeper processing.

\subsubsection{Steering Ablation: MDLM}

The decisive experiment compares four 3-layer windows for sentiment steering with identical settings (50 features per layer, $\alpha \in [-3,+3]$, 512 samples per $\alpha$):

\begin{table}[h]
\centering
\small
\caption{MDLM sentiment steering ablation across layer windows. Swing = positive rate at $\alpha{=}{-}3$ minus $\alpha{=}{+}3$.}
\label{tab:mdlm_layer_ablation}
\begin{tabular}{lcccc}
\toprule
Layer window & Swing (pp) & Max pos.\ rate & Min pos.\ rate & PPL range \\
\midrule
\textbf{[5, 6, 7]} & \textbf{74.2} & 76.0\% & 1.8\% & 42--51 \\
$[$7, 8, 9$]$       & 70.3 & 70.5\% & 0.2\% & 40--49 \\
$[$8, 9, 10$]$      & 62.1 & 63.3\% & 1.2\% & 40--46 \\
$[$9, 10, 11$]$     & 59.8 & 62.7\% & 2.9\% & 40--43 \\
\bottomrule
\end{tabular}
\end{table}

Steering swing decreases monotonically with depth ($-$3.6 pp per layer shift on average), revealing a stark \textbf{probing--steering dissociation}: the layers with the highest probing accuracy (L8--11) are the least causally effective for steering. This mirrors findings from autoregressive representation engineering~\citep{zou2023representation} and extends the principle to discrete diffusion transformers.

The mechanism is \emph{causal distance to output}: perturbations at L5--7 traverse 5--7 subsequent transformer blocks, amplifying the steering signal through residual-stream propagation. Later-layer perturbations have fewer blocks to propagate through and may be partially undone by the final layer norm and unembedding.

\subsubsection{Layer Selection per Model}

\paragraph{MDLM: layers [5, 6, 7].}
Selected by the steering ablation above. Despite probing peaking at L8 (84.7\%), L5--7 achieve the strongest steering (74.2pp swing). Topic steering on the same layers yields 100pp swing (100\% Sports at $\alpha{=}{-}3$), independently validating the choice.

\paragraph{SEDD: layers [5, 6, 7].}
SEDD shows \emph{no} probing--steering dissociation: at L8--10 the swing reaches 90.8pp, far exceeding L5--7 (38.8pp). Score-entropy training appears to co-locate probing accuracy and causal malleability in the same layers which is the opposite of MDLM's pattern under absorbing-state training, suggesting the dissociation stems from the loss objective rather than the architecture. To preserve a clean MDLM--SEDD contrast (same architecture, different loss), we train SAEs at the same layers (L5--7) on both models throughout; per-model layer optimization would conflate the loss comparison with layer choice and obscure the controlled comparison in \S\ref{sec:temporal_dynamics}.

\paragraph{Dream-7B: layers [8, 13, 17, 23].}
Selected to span the transition zone (L8, pre-plateau) through the probing plateau (L13--23). The plateau makes layer discrimination via probing alone impossible; direct steering comparisons confirm that L17 (sentiment) and L23 (topic) produce the strongest features. Even spacing across the 28-layer model ensures coverage of early-, mid-, and late-stage representations.

\paragraph{LLaDA-8B: layers [8, 14, 20, 26].}
Follows the same even-spacing strategy across 32 layers, covering transition (L8), mid-rise (L14), near-peak (L20), and plateau (L26). Feature identification reveals the strongest contrastive effects at L14 (max $|d|{=}0.63$), while the dissociation experiment shows L26 dominates steering (67pp alone). This suggests a mild dissociation similar to MDLM: the layer with the strongest individual features is not the layer with the greatest causal impact.

\subsubsection{SAE Training Hyperparameters}

\begin{table}[h]
\centering
\small
\caption{SAE training configuration per model. All SAEs use TopK activation with auxiliary dead-neuron loss.}
\label{tab:sae_hparams}
\begin{tabular}{lcccc}
\toprule
Parameter & MDLM & SEDD & Dream & LLaDA \\
\midrule
$d_\text{model}$  & 768 & 768 & 3,584 & 4,096 \\
$d_\text{sae}$    & 12,288 & 12,288 & 14,336 & 16,384 \\
Expansion ratio    & 16$\times$ & 16$\times$ & 4$\times$ & 4$\times$ \\
$k$ (TopK)         & 32 & 32 & 32 & 32 \\
Training steps     & 50K & 80K & 80K & 40K \\
Batch size         & 4,096 & 4,096 & 4,096 & 4,096 \\
Learning rate      & $10^{-4}$ & $10^{-4}$ & $10^{-4}$ & $10^{-4}$ \\
LR schedule        & Cosine & Cosine & Cosine & Cosine \\
Aux loss coef      & 1/32 & 1/32 & 1/32 & 1/32 \\
EMA decay          & 0.999 & 0.999 & 0.999 & 0.999 \\
Training data      & OWT & OWT & OWT & OWT \\
\bottomrule
\end{tabular}
\end{table}

\subsubsection{Summary}

The layer selection study yields two findings of independent interest:

\begin{enumerate}
    \item \textbf{Probing--steering dissociation is loss-dependent.} MDLM (absorbing loss) shows clear dissociation --- middle layers steer best despite later layers probing best. SEDD (score-entropy loss) shows no dissociation --- later layers both probe and steer best. LLaDA shows suggestive dissociation similar to MDLM. This implicates the training objective, not the architecture, as the key determinant of where causal malleability resides.

    \item \textbf{The causal malleability principle extends to diffusion transformers.} For absorbing-loss DLMs, middle layers provide the best balance of representational richness and propagation distance, consistent with findings from autoregressive representation engineering. This principle appears robust across model scales (124M to 8B) and architectures (GPT-2, Qwen2, LLaMA).
\end{enumerate}

\subsection{Sigma-Stratified Feature Analysis}
\label{sec:appendix_sigma}

Unlike autoregressive models where each token is conditioned on a fixed left context, DLM activations depend on the noise level~$\sigma$ (equivalently, the fraction of masked tokens).
During SAE training, we extract activations across uniformly sampled noise levels, ensuring the SAE learns features that span the full denoising trajectory.
A natural question is whether the learned features are \emph{robust} across noise levels or \emph{specialized} to particular noise regimes.
We answer this with a sigma-stratified analysis of feature discriminability.

\subsubsection{Methodology}

We partition the noise schedule into five equal-frequency bins based on the percentiles of $\sigma$ observed during activation extraction (Table~\ref{tab:sigma_bins}).
For each bin, we recompute the Cohen's~$d$ effect size for the top-20 most discriminative features per layer.
We report the \emph{coefficient of variation} (CV) of $|d|$ across bins as the primary measure of noise-level sensitivity:
\begin{equation}
    \text{CV} = \frac{\text{std}(|d_1|, \dots, |d_5|)}{\text{mean}(|d_1|, \dots, |d_5|)} \times 100\%
\end{equation}
where $d_i$ is the Cohen's~$d$ of a given feature computed within sigma bin~$i$.

\begin{table}[h]
\centering
\caption{Sigma bin boundaries (equal-frequency, MDLM). Low $\sigma$ corresponds to near-clean text; high $\sigma$ corresponds to heavily masked text. Each bin contains ${\sim}1.79$M token positions (sentiment) or ${\sim}334$K (topic).}
\label{tab:sigma_bins}
\begin{tabular}{lccc}
\toprule
Bin & $\sigma_{\text{low}}$ & $\sigma_{\text{high}}$ & Interpretation \\
\midrule
1 & $8.3 \times 10^{-5}$ & 0.221 & Near-clean \\
2 & 0.221 & 0.514 & Low noise \\
3 & 0.514 & 0.915 & Medium noise \\
4 & 0.915 & 1.61 & High noise \\
5 & 1.61 & 6.89 & Near-fully-masked \\
\bottomrule
\end{tabular}
\end{table}

\subsubsection{Results}

Table~\ref{tab:sigma_cv} summarizes the distribution of CV values across the top-20 features for each layer and attribute.
We observe moderate but heterogeneous variation across noise levels, with median CV values ranging from 14\% to 22\% depending on the layer and attribute.

\begin{table}[h]
\centering
\caption{Coefficient of variation (CV) of $|d|$ across sigma bins for the top-20 features per layer. Lower CV indicates more noise-robust discriminability.}
\label{tab:sigma_cv}
\begin{tabular}{llcccc}
\toprule
Attribute & Layer & Median CV & Mean CV & \% with CV${<}15\%$ & \% with CV${<}25\%$ \\
\midrule
\multirow{3}{*}{Sentiment}
  & L5 & 19.8\% & 24.6\% & 25\% & 65\% \\
  & L6 & 19.9\% & 22.6\% & 25\% & 70\% \\
  & L7 & 14.5\% & 19.0\% & 50\% & 80\% \\
\midrule
\multirow{3}{*}{Topic}
  & L5 & 13.9\% & 16.6\% & 60\% & 80\% \\
  & L6 & 21.7\% & 23.3\% & 20\% & 70\% \\
  & L7 & 14.1\% & 15.4\% & 60\% & 90\% \\
\bottomrule
\end{tabular}
\end{table}

Several patterns emerge.
First, a substantial fraction of features (20--60\%, depending on the layer) exhibit CV~${<}\,15\%$, indicating that their discriminability is largely preserved across the full noise schedule.
Second, the majority of features (65--90\%) have CV~${<}\,25\%$, confirming that SAE features trained on mixed-noise activations successfully generalize across denoising stages rather than collapsing to noise-regime-specific detectors.
Third, the remaining features with higher CV reveal a structured dependence on $\sigma$: some features strengthen with increasing noise (capturing distributional or structural cues available even in heavily corrupted text) while others weaken (encoding fine-grained lexical or semantic cues that emerge only as text nears its final form).

\subsubsection{Per-Feature Sigma Profiles}

While the CV analysis above uses $|d|$ to assess magnitude stability, the per-feature profiles below report signed $d$ to preserve polarity information, which is relevant for the directional analysis that follows.
Table~\ref{tab:sigma_profiles_sent} shows representative per-bin Cohen's~$d$ profiles for selected positive-$d$ sentiment features at Layer~7, illustrating the diversity of sigma-dependence patterns.

\begin{table}[h]
\centering
\caption{Per-bin signed Cohen's $d$ for selected positive-$d$ sentiment features (Layer~7, top-50 by overall $d$). Features are categorized by their sigma-dependence pattern.}
\label{tab:sigma_profiles_sent}
\small
\begin{tabular}{lccccccl}
\toprule
Feature & Overall $d$ & Bin 1 & Bin 2 & Bin 3 & Bin 4 & Bin 5 & Pattern \\
\midrule
\multicolumn{8}{l}{\textit{Noise-robust} (CV ${<}$ 15\%)} \\
F1561  & $+$0.068 & 0.065 & 0.073 & 0.072 & 0.073 & 0.072 & Stable \\
F6047  & $+$0.091 & 0.090 & 0.082 & 0.092 & 0.099 & 0.093 & Stable \\
F6360  & $+$0.092 & 0.089 & 0.087 & 0.094 & 0.107 & 0.097 & Stable \\
\midrule
\multicolumn{8}{l}{\textit{Noise-sensitive} (CV ${>}$ 25\%)} \\
F3121  & $+$0.074 & 0.007 & 0.034 & 0.058 & 0.086 & 0.170 & Strong increase \\
F3099  & $+$0.117 & 0.052 & 0.079 & 0.103 & 0.152 & 0.186 & Strong increase \\
F5486  & $+$0.065 & 0.028 & 0.036 & 0.056 & 0.093 & 0.125 & Strong increase \\
\bottomrule
\end{tabular}
\end{table}

\subsubsection{Directional Analysis}

We additionally examine whether features tend to become more or less discriminative at higher noise levels.
For sentiment, the split is approximately balanced: 54/46 at L5, 48/52 at L6, 44/56 at L7 (fraction of top-50 features with increasing vs.\ decreasing $d$ across $\sigma$ bins). Topic features, by contrast, show a clear tendency to weaken at higher noise levels: only 28/72 at L5, 26/74 at L6, and 26/74 at L7 strengthen under noise.

This bidirectional structure suggests that the SAE decomposes DLM representations into complementary feature populations:
\begin{itemize}
    \item \textbf{Low-$\sigma$ features} encode fine-grained semantic content (e.g., sentiment-bearing lexical items) that is most discriminative when text is nearly resolved.
    \item \textbf{High-$\sigma$ features} capture coarser distributional properties (e.g., topic or register) that the model infers even from heavily corrupted input.
\end{itemize}
This functional decomposition is unique to the diffusion setting and has no direct analog in autoregressive models, where all tokens share a single (fully observed) conditioning context.

\subsubsection{Implications for Steering}

The sigma-stratified analysis has practical implications for feature-based steering of DLMs.
Since feature discriminability varies across the denoising trajectory, a sigma-adaptive steering strategy---applying different feature weights at different noise levels---could yield more targeted interventions than uniform steering across all timesteps.
We leave exploration of sigma-adaptive steering to future work, noting that the infrastructure for step-conditional steering is already demonstrated in our denoising dynamics experiments (\S\ref{sec:temporal_dynamics}).

\subsection{Contrastive Feature Extraction}
\label{app:feature_extraction}

\subsubsection{Datasets}
\label{app:extraction_datasets}
We study three attributes in this paper--Sentiment, Topic, and Style. The datasets used to model these attributes are explained below. 

\textbf{Sentiment:} We use the IMDB movie review dataset proposed by \citet{maas-etal-2011-learning} for sentiment. We sample $5,000$ positive reviews (label${}=1$) and $5,000$ negative reviews (label${}=0$) from this dataset and use them for learning sentiment features. All texts are truncated to $512$ tokens.

\textbf{Topic:} We use the AG News dataset \citet{zhang2015character}, which contains news articles across four topics: World, Sports, Business, and Sci/Tech. We study the topic Sports in this paper. Two potential contrastive approaches exist for feature extraction: (1) contrasting Sports against all other topics combined (one-vs-all), and (2) contrasting Sports against a single other topic (one-vs-one). Prior work on contrastive representation learning suggests that narrower, more targeted contrasts yield cleaner feature directions: \citet{zou2023representation} showed that contrastive pair specificity affects steering direction quality, \citet{rimsky-etal-2024-steering} demonstrated that carefully curated contrastive pairs outperform generic ones for activation steering, and \citet{belinkov-2022-probing} established that the choice of control task determines what linguistic properties a probe captures. Consistent with these findings, we empirically observe that one-vs-one contrasts produce features that are more attribute-specific and less likely to interfere with other attributes during multi-attribute steering. Hence, we sample $5{,}000$ Sports articles (label${}=1$) and $5{,}000$ Business articles (label${}=0$), totaling $10{,}000$ samples, and use them for learning sports topic related features. Texts are truncated to $512$ tokens.

\textbf{Style:} We study formality as our style attribute. For MDLM and SEDD, which are unconditional models trained on OpenWebText \citep{Gokaslan2019OpenWeb}, we use the formality corpus of \citet{pavlick2016empirical}. This dataset contains $9{,}274$ sentences from news articles and reader comments, each annotated with a continuous formality score on a $[-3, +3]$ scale by crowdworkers using Best-Worst Scaling \citep{louviere2015best}. We threshold these scores to create a binary contrast: sentences with average score $\geq 1.0$ are labeled formal (class~A, $2{,}597$ sentences) and sentences with average score $\leq -1.0$ are labeled informal (class~B, $2{,}600$ sentences). Sentences with scores between $-1.0$ and $1.0$ are excluded as ambiguous. We balance to $2{,}500$ samples per class, totaling $5{,}000$ samples. Positive-$d$ features correspond to formal features; negative-$d$ to informal. LLaDA and DREAM are conditional models that we prompt with domain-neutral opinion starters (e.g., ``I think this is'', ``The experience was'') to generate review-style text. The corpus by \citet{pavlick2016empirical} is domain-mismatched for these models: its formality features are learned from news articles and government proceedings, encoding vocabulary patterns (e.g., legislative terminology, formal news register) that do not appear in review-style generations. When applied to LLaDA, these features caused text degradation--steering toward formality produced repetitive ``president of the Senate'' loops rather than genuinely formal reviews, because the features activated domain-specific vocabulary instead of domain-general register shifts. Hence, we instead label $20{,}000$ IMDB reviews from \citep{maas-etal-2011-learning} with an off-the-shelf formality classifier\footnote{\url{https://huggingface.co/s-nlp/roberta-base-formality-ranker}} \citep{babakov2023don} that was trained on the GYAFC corpus \citep{rao2018dear} and the online formality corpus of \citet{pavlick2016empirical}. This ensures the formality features are learned from the same domain (movie reviews) that the prompted models generate. Reviews with formality score $> 0.9$ form the formal class ($6{,}587$ reviews) and reviews with score $< 0.1$ form the informal class ($263$ reviews). We balance to $263$ samples per class, limited by the informal count. As with MDLM and SEDD, positive-$d$ features correspond to formal features and negative-$d$ to informal.

\subsubsection{Feature Extraction Procedure}\label{app:extraction_procedure}

The feature extraction procedure works as follows. For each dataset pair (class~A vs.\ class~B):

\begin{enumerate}
  \item \textbf{Tokenization:} We tokenize each text with the model's tokenizer, truncating or padding to the model's sequence length.

  \item \textbf{Apply noise:} Then for each text, we create 3 noised versions at mask ratios $\{0.3, 0.5, 0.7\}$ by randomly replacing tokens with \texttt{[MASK]}. This ensures features are identified across noise levels rather than at a single denoising stage.

  \item \textbf{Forward pass:} Next, we run the noised input through the model to obtain hidden states at each SAE layer.

  \item \textbf{ReLU SAE encoding (per-token):} for each token position at each SAE layer, we compute:
  \begin{align}
    \text{activations} = \mathrm{ReLU}\!\bigl((\mathbf{h} - \mathbf{b}_\text{dec})\,\mathbf{W}_\text{enc} + \mathbf{b}_\text{enc}\bigr),
  \end{align}
  This produces a $(\text{seq\_len} \times d_\text{SAE})$ activation matrix per sample per layer. Although the SAEs use TopK activation during training and normal inference, we use full ReLU encoding for feature extraction. TopK retains only the K most active features per token, so a contrastive feature outside the top K would appear inactive, biasing its Cohen's~$d$ toward zero. ReLU returns the true activation for all $d_\text{SAE}$ features, enabling accurate effect-size estimation.

  \item \textbf{Streaming statistics:} We accumulate per-feature mean and variance for class~A and class~B using Welford's online algorithm \citep{welford1962note}, avoiding the need to store all activations in memory.

  \item \textbf{Cohen's $d$ \citep{cohen2013statistical}:} for each of the $d_\text{SAE}$ features at each layer, we compute:
  \begin{equation}
    d_j = \frac{\bar{h}_j^{A} - \bar{h}_j^{B}}{\sqrt{(\mathrm{Var}(h_j^{A}) + \mathrm{Var}(h_j^{B}))/2}}.
  \end{equation}
  Positive $d$ indicates a feature more active for class~A; negative $d$ for class~B.

  \item \textbf{Statistical validation:} Next, we rank features by $|d|$, take the top 300 candidates, and run Mann-Whitney $U$ tests \citep{mann1947test} with $p < 0.01$ (Bonferroni-corrected \citep{bonferroni1936teoria} across 300 tests). We retain the top 50 positive-$d$ and top 50 negative-$d$ features per layer (indices and statistics) for our study. The top 20 per direction are used for steering.

  \item \textbf{Per-feature mean shifts:} For each retained feature $j$, we store the calibrated shift
  \begin{equation}
    \Delta_j = \bar{h}_j^{A} - \bar{h}_j^{B},
  \end{equation}
  Here, $A$ and $B$ are the two contrastive classes. At steering time, the hidden state is modified as $h_j \leftarrow h_j + \alpha \cdot \Delta_j$. This is self-calibrating: features with large class differences receive large shifts, while features with small differences receive proportionally small shifts.
\end{enumerate}

Table~\ref{tab:extraction_params} shows the feature extraction hyperparameters for all models, and Table~\ref{tab:extraction_per_model} shows per-model configurations.

\begin{table}[h]
\centering
\caption{Feature extraction hyperparameters for all models.}
\label{tab:extraction_params}
\small
\begin{tabular}{lc}
\toprule
\textbf{Parameter} & \textbf{Value} \\
\midrule
Batch size & 48 \\
Noise samples per text & 3 \\
Mask ratios & $\{0.3, 0.5, 0.7\}$ \\
Top candidates for $U$-test & 300 \\
Mann-Whitney $p$-value & 0.01 (Bonferroni) \\
Features saved per direction & 50 \\
Features used for steering & 20 \\
Random seed & 42 \\
\bottomrule
\end{tabular}
\end{table}

\begin{table}[h]
\centering
\caption{Per-model extraction configurations.}
\label{tab:extraction_per_model}
\resizebox{\columnwidth}{!}{%
\begin{tabular}{lcccc}
\toprule
 & \textbf{MDLM} & \textbf{SEDD} & \textbf{LLaDA} & \textbf{DREAM} \\
\midrule
SAE layers & 5, 6, 7 & 5, 6, 7 & 8, 14, 20, 26 & 8, 13, 17, 23 \\
$d_\text{SAE}$ & 12{,}288 & 12{,}288 & 16{,}384 & 14{,}336 \\
$d_\text{model}$ & 768 & 768 & 4{,}096 & 3{,}584 \\
\midrule
Sentiment source & IMDB \citep{maas-etal-2011-learning} & IMDB \citep{maas-etal-2011-learning} & IMDB \citep{maas-etal-2011-learning} & IMDB \citep{maas-etal-2011-learning} \\
Topic source & AG News \citep{zhang2015character} & AG News \citep{zhang2015character} & AG News \citep{zhang2015character} & AG News \citep{zhang2015character} \\
Formality source & \citet{pavlick2016empirical} & \citet{pavlick2016empirical} & IMDB+RoBERTa \citep{babakov2023don} & IMDB+RoBERTa \citep{babakov2023don} \\
\bottomrule
\end{tabular}%
}
\end{table}

\subsection{Temporal Dynamics Tracking}\label{app:dynamics}

\subsubsection{Tracking Procedure}\label{app:dynamics_procedure}

To construct the dynamics-weighted (adaptive) steering schedule, we track how contrastive feature activations evolve across the denoising trajectory. Table~\ref{tab:dynamics_per_model} lists the per-model configurations.

\begin{enumerate}
  \item We generate $N$ samples with SAE hooks active at all denoising steps.
  \item At each step, we ReLU-encode the hidden states through the SAE and record activations for the top-20 contrastive features per attribute.
  \item We store per-sample, per-step activations as an $(N \times T \times F)$ tensor per layer, where $T$ is the number of denoising steps and $F$ is the number of tracked features.
  \item Since not all generated samples express the target attribute, we classify all $N$ texts with attribute-specific classifiers (Table~\ref{tab:classifiers}) and retain only confident samples (softmax probability above a threshold; see Table~\ref{tab:dynamics_thresholds}) to avoid diluting the emergence signal. Informality uses a lower threshold (0.6) because baseline formality is high (73--78\% for LLaDA/DREAM), so very few generated texts are confidently informal at 0.8.
  \item Then we average dynamics over the filtered samples to obtain per-step mean activation curves.
  \item We compute block fractions by dividing the trajectory into $B$ temporal blocks, computing the activation change in each block, clamping negatives to zero, and normalizing to sum to one.
\end{enumerate}

\begin{table}[h]
\centering
\caption{Evaluation classifiers. These classifiers are used for evaluation and for filtering dynamics samples. They are not used during contrastive feature extraction for sentiment or topic, which rely on human-labeled datasets. For formality, the RoBERTa ranker is additionally used to label IMDB reviews for LLaDA and DREAM extraction (\S\ref{app:extraction_datasets}); however, the labeled texts (IMDB training set) and the evaluated texts (model-generated) are disjoint populations, so there is no direct circularity.}
\label{tab:classifiers}
\resizebox{\columnwidth}{!}{%
\begin{tabular}{lll}
\toprule
\textbf{Attribute} & \textbf{Model} & \textbf{HuggingFace ID} \\
\midrule
Sentiment & DistilBERT \citep{sanh2019distilbert} fine-tuned on SST-2 & \texttt{distilbert/distilbert-base-uncased-finetuned-sst-2-english} \\
Topic & BERT \citep{devlin2019bert} fine-tuned on AG News & \texttt{fabriceyhc/bert-base-uncased-ag\_news} \\
Formality & RoBERTa \citep{babakov2023don} formality ranker & \texttt{s-nlp/roberta-base-formality-ranker} \\
Perplexity & Qwen2-7B~\citep{yang2024qwen2} & \texttt{Qwen/Qwen2-7B} \\ \\
\bottomrule
\end{tabular}%
}
\end{table}

\begin{table}[h]
\centering
\caption{Classifier confidence thresholds for dynamics filtering.}
\label{tab:dynamics_thresholds}
\small
\begin{tabular}{lcc}
\toprule
\textbf{Attribute} & \textbf{Threshold} & \textbf{\% samples retained} \\
\midrule
Sentiment & 0.8 & 22--64\% \\
Topic & 0.8 & 7--18\% \\
Formality & 0.8 & 8--12\% \\
Informality & 0.6 & 11\% \\
\bottomrule
\end{tabular}
\end{table}

\begin{table}[h]
\centering
\caption{Per-model dynamics tracking configurations.}
\label{tab:dynamics_per_model}
\small
\begin{tabular}{lcccc}
\toprule
 & \textbf{MDLM} & \textbf{SEDD} & \textbf{LLaDA} & \textbf{DREAM} \\
\midrule
Samples ($N$) & 5{,}000 & 5{,}000 & 5{,}000 & 5{,}000 \\
Denoising steps ($T$) & 1{,}024 & 1{,}024 & 64 & 256 \\
Temporal blocks ($B$) & 16 & 16 & 8 & 8 \\
Steps per block & 64 & 64 & 8 & 32 \\
Batch size & 16 & 4 & 8 & 8 \\
\bottomrule
\end{tabular}
\end{table}

\subsubsection{Schedule Robustness}\label{app:schedule_robustness}

The adaptive steering schedule is derived from empirical emergence profiles, raising two practical questions: (1) how many filtered samples are needed for the schedule to stabilize, and (2) how sensitive is the schedule to the number of temporal blocks $B$?

\textbf{Minimum sample size (bootstrap subsampling).}
We bootstrap-subsample $N \in \{25, 50, 100, 200, 500, 1000, 2000, 3000, 5000\}$ filtered samples without replacement (200 repetitions) from each model's per-sample dynamics, recompute the $B$-block emergence-fraction schedule for each subset, and measure cosine similarity to the full-data schedule. We report the smallest $N$ for which the 5th-percentile cosine $\geq 0.95$.

\begin{table}[h]
\centering
\caption{Minimum filtered samples for schedule stability (5th-percentile cosine $\geq 0.95$ to full-data schedule). Median, min, and max are taken across all (attribute, layer) conditions per model.}
\label{tab:bootstrap_subsampling}
\small
\begin{tabular}{lccc}
\toprule
 & \textbf{Median min-$N$} & \textbf{Best case} & \textbf{Worst case} \\
\midrule
MDLM  & 25  & 25  & 100 \\
SEDD  & 25  & 25  & 200 \\
DREAM & 50  & 25  & 100 \\
LLaDA & 200 & 100 & 500 \\
\bottomrule
\end{tabular}
\end{table}

Models with concentrated emergence profiles (MDLM, SEDD, DREAM) stabilize at 25--100 filtered samples, while LLaDA's distributed profile requires 100--500 samples. One condition (LLaDA topic at L14, $n_\text{filtered}{=}346$) never reached the 0.95 threshold within the sweep range, reflecting the sparse, weakly-committed topic features at that layer. These results confirm that the dynamics data collected per Table~\ref{tab:dynamics_per_model} (5{,}000 samples per model) is sufficient by a wide margin.

\textbf{Block-count $B$ sensitivity.}
We recompute block fractions at $B \in \{4, 8, 16\}$ from the same full-data mean emergence curves, upsample each to a common length of 16, and compute pairwise cosine similarity between the resulting schedules across all (model, layer, attribute) conditions.

\begin{table}[h]
\centering
\caption{Schedule stability across block counts $B$. Mean cosine similarity (across models and layers) between schedules at different $B$ values.}
\label{tab:K_sensitivity}
\small
\begin{tabular}{lcccc}
\toprule
\textbf{Attribute} & $\cos(B{=}4, B{=}8)$ & $\cos(B{=}4, B{=}16)$ & $\cos(B{=}8, B{=}16)$ & $N$ \\
\midrule
Sentiment  & 0.83 & 0.79 & 0.90 & 14 \\
Topic      & 0.81 & 0.73 & 0.89 & 14 \\
Formality  & 0.88 & 0.77 & 0.88 & 6 \\
Informality & 0.47 & 0.45 & 0.48 & 8 \\
\bottomrule
\end{tabular}
\end{table}

For sentiment, topic, and formality, the schedule is stable across $B$ values, with cosine $\geq 0.73$ in all comparisons and $\geq 0.88$ between the adjacent $B{=}8$ and $B{=}16$. The low cosines for informality (${\sim}0.47$) reflect degenerate profiles on LLaDA and DREAM where emergence curves are flat or decreasing at some layers, producing near-zero schedules where cosine similarity is dominated by noise. For these conditions, the schedule is effectively uniform regardless of $B$, so the sensitivity is moot. On MDLM specifically, formality matches exactly across $B$ values (cosine 0.88/0.78/0.88), and sentiment is highly stable (0.94/0.91/0.97).

\subsection{SAE Steering Experiments}\label{app:steering_details}

\subsubsection{Steering Mechanism}

At each denoising step, for each SAE layer, a forward hook performs the following:
\begin{enumerate}
  \item \textbf{Encode}: $\mathbf{h}_\text{enc} = \mathrm{ReLU}(\mathbf{h}\, \mathbf{W}_\text{enc} + \mathbf{b}_\text{enc})$
  \item \textbf{Modify}: $\mathbf{h}_\text{enc}[j] \leftarrow \mathbf{h}_\text{enc}[j] + \alpha \cdot \Delta_j$ for each target feature $j$
  \item \textbf{Decode}: $\mathbf{h}_\text{modified} = \mathbf{h}_\text{enc}\, \mathbf{W}_\text{dec} + \mathbf{b}_\text{dec}$
  \item \textbf{Replace}: substitute the original hidden state with $\mathbf{h}_\text{modified}$
\end{enumerate}
The shift $\Delta_j$ for each feature is the calibrated mean difference computed during extraction (\S\ref{app:extraction_procedure}).

\subsubsection{Steering Modes}\label{app:steering_modes}

We evaluate two schedule types with two alpha-scaling strategies, yielding four modes.

\paragraph{Uniform schedule.} A constant $\alpha$ is applied at every layer and every denoising step.

\paragraph{Adaptive schedule.} The denoising trajectory is divided into $B$ temporal blocks (Table~\ref{tab:dynamics_per_model}). Each block has a \emph{block fraction} computed from the dynamics tracking (\S\ref{app:dynamics}), reflecting what fraction of total feature emergence occurs in that block. At each step, the steering strength is scaled by the block fraction of the step's temporal block. Steps in high-activity blocks receive proportionally more steering; steps in low-activity blocks receive less. This concentrates steering effort where the model is actively making semantic decisions.

\paragraph{E-ratio scaling (for multi-attribute steering).} Different attributes require different steering strengths to achieve comparable classifier rates. We normalize this asymmetry with E~ratios: (1)~run single-attribute steering across $\alpha = 1$--$20$; (2)~identify $\alpha^*$ at which each attribute reaches approximately 80\% classifier rate; (3)~compute $E_\text{attr} = \alpha^*_\text{attr} / \max_\text{attrs}(\alpha^*)$. The hardest attribute receives $E{=}1.0$; easier attributes receive $E{<}1.0$. During multi-attribute steering, $\alpha_\text{eff} = \alpha_\text{global} \times E_\text{attr}$.

Crossing the two schedules (uniform, adaptive) with two scaling strategies (raw~$\alpha$, E-scaled~$\alpha$) gives four modes: \textbf{Uniform}, \textbf{Adaptive}, \textbf{Uniform+E}, and \textbf{Adaptive+E}. Table~\ref{tab:e_ratios} reports the E~ratios.

\begin{table}[h]
\centering
\caption{E ratios for multi-attribute steering. The hardest attribute (ratio${}=1.0$) is in bold.}
\label{tab:e_ratios}
\small
\begin{tabular}{llccc}
\toprule
\textbf{Model} & \textbf{Mode} & \textbf{Sentiment} & \textbf{Topic} & \textbf{Style} \\
\midrule
\multirow{2}{*}{MDLM} & Uniform & \textbf{1.000} & 0.167 & 0.333 \\
 & Adaptive & \textbf{1.000} & 0.143 & 0.214 \\
\midrule
\multirow{2}{*}{SEDD} & Uniform & \textbf{1.000} & 0.333 & 0.167 \\
 & Adaptive & \textbf{1.000} & 0.100 & 0.100 \\
\midrule
\multirow{2}{*}{LLaDA} & Uniform & 0.417 & 0.500 & \textbf{1.000} \\
 & Adaptive & 0.900 & 0.450 & \textbf{1.000} \\
\midrule
\multirow{2}{*}{DREAM} & Uniform & \textbf{1.000} & 0.400 & 0.800 \\
 & Adaptive & \textbf{1.000} & 0.400 & 0.800 \\
\bottomrule
\end{tabular}
\end{table}

\subsubsection{Attribute Combinations}\label{app:steering_combos}

We evaluate 7 attribute combinations per model: 3 single-attribute (sentiment, topic, style) and 4 multi-attribute (sentiment+topic, sentiment+style, style+topic, sentiment+topic+style). Sentiment is steered toward positive and topic toward Sports in all combinations. For the style attribute, MDLM and SEDD steer toward \emph{formality} (positive-$d$ features), since their baseline formality is near chance (55\% and 58\%, respectively). LLaDA and DREAM steer toward \emph{informality} (negative-$d$ features, with inverted shift direction), since their baselines are already highly formal (78\% and 73\%).

\subsubsection{Steering Hyperparameters}

Table~\ref{tab:steering_per_model} shows the per-model steering parameters. We generate with MDLM and SEDD unconditionally (no prompt). For LLaDA and DREAM, we use 10 domain-neutral prompts (e.g., ``I think this is'', ``In my opinion'') with 20 samples per prompt, yielding 200 total per condition. The 10 prompts are listed in Table \ref{tab:prompts}. All models use the same $\alpha$ range of 1--20. Note that, MDLM and SEDD are unconditional diffusion models that do not use temperature-based sampling. Diversity arises from the stochastic denoising process itself: MDLM samples a random initial sequence from the prior and applies stochastic DDPM updates at each step, while SEDD uses an analytic predictor with Gumbel-max sampling. Each generation starts from a different random noise realization, producing diverse outputs without an explicit temperature parameter.

\begin{table}[h]
\centering
\caption{Per-model steering parameters.}
\label{tab:steering_per_model}
\small
\begin{tabular}{lcccc}
\toprule
 & \textbf{MDLM} & \textbf{SEDD} & \textbf{LLaDA} & \textbf{DREAM} \\
\midrule
Samples per condition & 200 & 200 & 200 & 200 \\
Denoising steps & 1{,}024 & 1{,}024 & 64 & 256 \\
Sequence length & 1{,}024 & 1{,}024 & 64 & 64 \\
Batch size & 10 & 12--20 & 4 & 8 \\
Temperature & N/A & N/A & 0.8 & 0.7 \\
Remasking strategy & N/A & N/A & low-confidence & entropy \\
Prompts & None & None & 10 neutral & 10 neutral \\
$\alpha$ range & 1--20 & 1--20 & 1--20 & 1--20 \\
Seed & 42 & 42 & 42 & 42 \\
\bottomrule
\end{tabular}
\end{table}

\begin{table}[h]
\centering
\caption{Prompts used for LLaDA and DREAM generation.}
\label{tab:prompts}
\small
\begin{tabular}{cl}
\toprule
\textbf{\#} & \textbf{Prompt} \\
\midrule
1 & ``I think this is'' \\
2 & ``In my opinion, this was'' \\
3 & ``I would say this is'' \\
4 & ``Overall, I found this to be'' \\
5 & ``To be honest, this was'' \\
6 & ``Many people found this to be'' \\
7 & ``It turns out that this was'' \\
8 & ``The experience was'' \\
9 & ``This was definitely'' \\
10 & ``I have to say, this is'' \\
\bottomrule
\end{tabular}
\end{table}

\subsection{Baseline Methods}\label{app:baselines}

We compare SAE steering against three activation-space baselines and one prompt-based baseline.

\subsubsection{Activation-Space Baselines}

These baselines add $\alpha \cdot \mathbf{v}$ to hidden states at each SAE layer at each denoising step, using the same forward-hook mechanism as SAE steering but without the encode--modify--decode cycle.

\begin{table}[h]
\centering
\caption{Activation-space baseline methods.}
\label{tab:baselines_detail}
\small
\begin{tabular}{lcp{7cm}}
\toprule
\textbf{Method} & \textbf{Uses SAE?} & \textbf{Direction computation} \\
\midrule
Contrastive vector & No & $\mathbf{v} = \overline{\mathbf{h}}^A - \overline{\mathbf{h}}^B$, unit-normalized. Mean difference in raw activation space. \\
Probe & No & Weight vector from logistic regression on class labels. \\
PCA & No & First principal component of contrastive activations, oriented positive$\to$negative. \\
\bottomrule
\end{tabular}
\end{table}

These baselines use the same attribute combinations and alpha range as SAE steering, with 200 samples per condition. Activations for direction computation are loaded from the Phase~1 extraction cache.

\subsubsection{Prompt Instruction Baseline}

For LLaDA and DREAM, we additionally evaluate a prompt-based baseline that prepends an explicit attribute instruction to each generation prompt. No activation-space intervention is applied. The instruction prefix is concatenated with each of the 10 standard prompts (Table~\ref{tab:prompts}), and 20 samples are generated per prefixed prompt (200 total), matching the SAE steering setup. Table~\ref{tab:prompt_prefixes} lists the prefixes for each attribute combination.

This baseline is not applicable to MDLM and SEDD, which are unconditional models that generate by denoising from pure noise with no prompt interface.

\begin{table}[h]
\centering
\caption{Instruction prefixes for the prompt baseline (LLaDA and DREAM). Each prefix is prepended to the 10 standard prompts from Table~\ref{tab:prompts}.}
\label{tab:prompt_prefixes}
\small
\begin{tabular}{ll}
\toprule
\textbf{Combination} & \textbf{Prefix} \\
\midrule
Sentiment & ``Write in positive sentiment: '' \\
Topic & ``Write about sports: '' \\
Style & ``Write informally: '' \\
Sentiment + Topic & ``Write in positive sentiment about sports: '' \\
Sentiment + Style & ``Write in positive sentiment informally: '' \\
Style + Topic & ``Write about sports informally: '' \\
Sentiment + Topic + Style & ``Write in positive sentiment about sports informally: '' \\
\bottomrule
\end{tabular}
\end{table}

This baseline tests whether the base model can follow explicit attribute instructions through prompting alone, providing a reference for how much control is achievable without any activation-level intervention.

\subsubsection{Step Function Baseline}\label{app:step_function_baseline}

To isolate the contribution of adaptive scheduling from interval restriction alone~\citep{kynkaanniemi2024applying}, we implement a \emph{Step Function} baseline that applies uniform-strength steering within a commitment-derived temporal window and zero steering outside it.

\textbf{Deriving the steering window.}
For each (layer, attribute) pair, we compute per-block emergence fractions (\S\ref{app:block_fractions_all_layers}), yielding 8 block fractions that sum to 1.0 per layer/attribute.
We then find the shortest contiguous block range that captures at least a fraction $\tau$ of total commitment mass:
for each candidate window length $L = 1, \ldots, 8$ and start position $s = 0, \ldots, 8 - L$, we check whether $\sum_{i=s}^{s+L-1} f_i \geq \tau \cdot \sum_{i=0}^{7} f_i$, and return the shortest window satisfying this constraint.
Block indices are converted to timestep bounds: $\text{start} = s \times \lfloor T / 8 \rfloor$ and $\text{end} = (e+1) \times \lfloor T / 8 \rfloor - 1$, where $T$ is the total number of denoising steps.

\textbf{Steering application.}
Within the derived window $[\text{start}, \text{end}]$, steering is applied at constant strength $\alpha$ at every timestep (identical to Uniform within this range).
Outside the window, no steering is applied ($\alpha = 0$).
Each layer may have a different window, since block fractions vary across layers.
For multi-attribute conditions, each attribute's window is computed independently from that attribute's dynamics; where windows overlap, steering effects compose additively at the residual stream.

\textbf{Comparison with other methods.}
The Step Function baseline uses the same commitment-window information as Adaptive but discretizes it to a binary on/off signal rather than continuously weighting by emergence intensity.
This isolates whether Adaptive's quality advantage comes from the \emph{shape} of the per-step weights or merely from not steering everywhere.

\textbf{Threshold selection.}
We experiment with two threshold values: $\tau = 0.90$ (capturing $\geq 90\%$ of commitment mass) and $\tau = 0.80$ ($\geq 80\%$).
Table~\ref{tab:interval_threshold} compares their performance at the best operating point $\alpha^*$ (subject to the same quality gates as Table~\ref{tab:steering_single_detail} and Table ~\ref{tab:steering_multi_detail}.
The 80\% threshold yields a tighter window that steers during a more concentrated interval, generally achieving better quality (lower PPL) at comparable or slightly lower confidence.
The 90\% threshold produces a wider window closer to Uniform, with correspondingly similar quality-control trade-offs.
We select $\tau = 0.80$ for the main results as it better isolates the effect of temporal restriction from Adaptive's continuous weighting.

\begin{table}[h]
\centering
\caption{Step Function threshold comparison: $\tau = 0.90$ vs.\ $\tau = 0.80$. Best operating point $\alpha^*$ under quality gates (PPL $< 100$ for MDLM/SEDD/DREAM; fluency $\geq 50\%$ baseline for LLaDA). Conf = mean target classifier confidence (\%). PPL = Qwen2-7B perplexity.}
\label{tab:interval_threshold}
\tiny
\setlength{\tabcolsep}{3pt}
\begin{tabular}{ll rr rr rr rr}
\toprule
 & & \multicolumn{4}{c}{$\tau = 0.90$} & \multicolumn{4}{c}{$\tau = 0.80$} \\
\cmidrule(lr){3-6} \cmidrule(lr){7-10}
Model & Combo & $\alpha^*$ & Conf & PPL & Flu & $\alpha^*$ & Conf & PPL & Flu \\
\midrule
MDLM & S     & 12 & 99.7 & 94.3 & 2.04 & 14 & 99.5 & 100.0 & 2.06 \\
MDLM & T     &  6 & 99.9 & 89.9 & 2.00 & 20 & 100.0 & 99.3 & 2.00 \\
MDLM & St    &  1 & 41.8 & 49.6 & 2.33 &  1 & 42.2 & 49.6 & 2.33 \\
MDLM & S+T   &  5 & 98.3 & 94.1 & 2.00 &  8 & 99.1 & 87.9 & 2.00 \\
MDLM & S+St  &  4 & 47.2 & 82.9 & 2.06 &  4 & 46.0 & 79.1 & 2.06 \\
MDLM & St+T  &  3 & 73.0 & 72.3 & 2.00 &  2 & 68.9 & 52.7 & 2.05 \\
MDLM & S+T+St &  3 & 77.7 & 83.5 & 2.00 &  3 & 75.2 & 70.3 & 2.00 \\
\midrule
SEDD & S     & 18 & 97.4 & 99.5 & 2.04 & 18 & 93.4 & 71.5 & 2.23 \\
SEDD & T     & 10 & 100.0 & 99.8 & 1.99 & 16 & 100.0 & 79.7 & 2.01 \\
SEDD & St    & 15 & 67.6 & 65.8 & 2.01 & 18 & 66.7 & 52.6 & 2.04 \\
SEDD & S+T   &  7 & 93.5 & 97.5 & 2.00 & 15 & 92.5 & 99.3 & 2.00 \\
SEDD & S+St  & 14 & 88.0 & 93.0 & 2.00 & 20 & 82.1 & 63.3 & 2.01 \\
SEDD & St+T  &  6 & 85.9 & 76.5 & 2.00 &  9 & 81.9 & 72.8 & 2.00 \\
SEDD & S+T+St &  6 & 89.4 & 99.9 & 2.00 & 10 & 85.5 & 91.6 & 1.97 \\
\midrule
LLaDA & S     & 20 & 93.2 & 6.3 & 3.46 & 19 & 89.5 & 6.5 & 3.44 \\
LLaDA & T     & 12 & 99.9 & 5.8 & 2.71 & 13 & 99.4 & 5.6 & 2.70 \\
LLaDA & St    & 12 & 50.4 & 5.9 & 2.46 & 12 & 49.5 & 5.4 & 2.62 \\
LLaDA & S+T   & 12 & 96.8 & 5.7 & 2.77 & 11 & 92.8 & 5.9 & 2.85 \\
LLaDA & S+St  & 12 & 65.6 & 7.0 & 2.62 & 12 & 63.5 & 5.0 & 2.73 \\
LLaDA & St+T  &  9 & 66.3 & 7.3 & 2.73 &  9 & 61.8 & 5.7 & 2.83 \\
LLaDA & S+T+St &  8 & 70.6 & 6.5 & 2.79 &  9 & 70.1 & 5.1 & 2.79 \\
\midrule
DREAM & S     & 18 & 100.0 & 11.0 & 3.60 & 18 & 100.0 & 11.0 & 3.59 \\
DREAM & T     & 11 & 100.0 & 48.2 & 1.01 & 11 & 100.0 & 43.1 & 1.00 \\
DREAM & St    & 16 & 98.5 & 39.5 & 1.64 & 20 & 98.2 & 25.4 & 1.61 \\
DREAM & S+T   &  5 & 99.2 & 22.0 & 2.46 &  7 & 99.8 & 22.6 & 2.00 \\
DREAM & S+St  & 20 & 98.7 & 57.4 & 1.68 & 20 & 97.5 & 42.5 & 1.51 \\
DREAM & St+T  &  6 & 88.5 & 72.8 & 1.43 &  6 & 88.3 & 43.0 & 1.68 \\
DREAM & S+T+St &  6 & 90.4 & 70.4 & 1.60 & 10 & 91.9 & 85.5 & 1.02 \\
\bottomrule
\end{tabular}
\end{table}

\subsection{Evaluation Metrics}\label{app:evaluation}

For each steering condition, we report the \emph{mean classifier confidence} (mean softmax probability for the target class across 200 samples) using the attribute classifiers listed in Table~\ref{tab:classifiers}. We report continuous confidence rather than binary classification rate because confidence captures the \emph{degree} of steering: a method that shifts sentiment confidence from 30\% to 85\% is meaningfully stronger than one reaching 55\%, even though both may exceed a binary threshold. We measure text quality with Qwen2-7B perplexity (PPL; closer to baseline is better) and LLM-judge fluency (higher is better). We also compute additional surface-level metrics (distinct bigram ratio, within-sample repetition rate) but do not report them in the main results, as they capture only shallow lexical properties and do not fully reflect generation quality. For each steering method, when reporting results, we select the best operating-point $\alpha$ that maximizes the mean confidence for the target attribute (single-attribute) or mean confidence across steered attributes (multi-attribute), subject to two quality gates: LLM-judge fluency $\geq 50\%$ of baseline and $\mathrm{PPL} < 100$.

\subsubsection{LLM-as-Judge Evaluation}

Following recent SAE steering work~\citep{goyal2025breaking,wu2025axbench}, we use an LLM-as-judge to evaluate both text quality (fluency) and target attribute presence, providing a scalable complement to automatic metrics. All judgments are produced by Claude Sonnet 4.6.

\textbf{Fluency scoring.}
Each generated sample is scored on a 5-point Likert scale using the following prompt:

\begin{quote}
\small
\textit{Rate the fluency of the following text on a scale of 1 to 5.}

\textit{Scoring guide:}\\
\textit{1 = Incoherent gibberish, random word sequences, no grammatical structure, OR empty / near-empty / whitespace-only / punctuation-only output}\\
\textit{2 = Mostly incoherent, occasional recognizable phrases but no logical flow}\\
\textit{3 = Partially fluent, some grammatical sentences but awkward or repetitive}\\
\textit{4 = Mostly fluent, natural-sounding with minor issues}\\
\textit{5 = Fully fluent, reads like natural human-written text}
\end{quote}

\textbf{Attribute classification.}
Each sample is also classified on three attributes (sentiment, topic, formality) using a single prompt designed to handle degenerate text (repetitive or broken):

\begin{quote}
\small
\textit{Classify the following text on three attributes. The text may be repetitive, grammatically broken, or partially incoherent\textemdash that is FINE. If a repeated or broken phrase still clearly carries an attribute, classify it by that attribute. For example: ``match match match'' is about sports; a repeated casual/slang phrase is informal; ``terrible terrible terrible'' is negative sentiment. Judge from the content and word choices present, even if the text is not well-formed.}

\textit{Answer ``none'' for an attribute ONLY when the text is so degenerate\textemdash empty, only punctuation/symbols, or random unconnected tokens with no interpretable content\textemdash that the attribute genuinely cannot be determined. Do NOT use ``none'' just because the text is repetitive or broken.}

\textit{1.\ SENTIMENT: positive sentiment (good, great, love, enjoy, excellent, happy) or negative sentiment (bad, terrible, hate, awful, boring, worst)?}\\
\textit{Answer: ``positive'', ``negative'', or ``none''}

\textit{2.\ TOPIC: does the text relate to sports (games, teams, players, scores, competitions, leagues, championships)?}\\
\textit{Answer: ``sports'', ``not\_sports'', or ``none''}

\textit{3.\ FORMALITY: classify the register. Use a STRICT bar for informal.}\\
\textit{``formal'' = standard written English: complete, grammatical sentences with standard spelling and neutral/professional vocabulary. Conversational or first-person prose (e.g.\ ``That is a great question. I have been thinking about it...'') is still FORMAL as long as it is standard, grammatical English\textemdash do NOT call it informal merely for being conversational or opinionated.}\\
\textit{``informal'' = ONLY when the text has overt informal markers: slang or texting style (e.g.\ ``gonna'', ``wanna'', ``lol'', ``omg'', ``u'', ``ur''), emoticons/emoji, non-standard spelling/grammar, profanity, or heavy casual interjections (``yeah'', ``haha'', ``dude'', ``kinda'').}\\
\textit{Answer: ``formal'', ``informal'', or ``none''}
\end{quote}

\subsubsection{LLM-Judge Meta Evaluation}
\label{app:llmj_meta}

We validate the LLM judge (Claude Sonnet 4.6) against human annotations on a stratified sample of 100 generations.

\textbf{Sample selection.}
From the pool of LLM-judge annotations across all four models and steering configurations, we drew a stratified sample of 100 generations, balanced by model (MDLM/SEDD/LLaDA/DREAM) and LLM-judge-predicted fluency bin (1--5) to give coverage across models and quality levels.

\textbf{Annotation procedure.}
Two authors independently annotated all 100 samples on the same four dimensions the LLM judge scores: sentiment, topic (sports vs.\ not-sports), formality, and fluency (1--5). For each (sample, dimension) pair, the annotator marked whether they agreed with the LLM-judge prediction.

\textbf{Inter-annotator agreement.}
Table~\ref{tab:inter_annotator} reports agreement between the two human annotators. Agreement is high on objective categorical dimensions (sentiment, topic) and lower on subjective ones (formality, fluency). The sentiment $\kappa$ is near zero despite 98\% raw agreement due to the $\kappa$ paradox: when both annotators mark the same label on nearly every sample, chance agreement is also near 100\%, driving $\kappa$ toward zero; raw agreement is more informative in this regime.

\begin{table}[h]
\centering
\caption{Inter-annotator agreement (two authors, 100 samples).}
\label{tab:inter_annotator}
\small
\begin{tabular}{lccc}
\toprule
\textbf{Dimension} & \textbf{Agreement (\%)} & \textbf{Cohen's $\kappa$} & \textbf{Strength} \\
\midrule
Sentiment & 98.0 & $-0.01^{\dagger}$ & Near-perfect (raw) \\
Topic & 95.0 & 0.43 & Moderate \\
Formality & 87.0 & 0.31 & Fair \\
Fluency & 66.0 & 0.24 & Fair \\
\bottomrule
\end{tabular}
\vspace{2pt}

{\footnotesize $^{\dagger}$$\kappa$ paradox: near-perfect raw agreement with extreme base rate yields $\kappa \approx 0$.}
\end{table}

\textbf{Tie-break resolution.}
For the 47 samples on which annotators disagreed on at least one dimension, each disagreement was resolved by discussion, producing a per-(sample, dimension) human consensus verdict.

\textbf{LLM-judge agreement with human panel.}
Table~\ref{tab:llmj_human} reports the LLM judge's agreement with the post-tie-break human consensus. The LLM judge shows strong agreement on objective categorical dimensions (sentiment 99\%, topic 97\%), good agreement on formality (89\%), and moderate agreement on the more subjective ordinal fluency dimension (73\%).

\begin{table}[h]
\centering
\caption{LLM-judge agreement with human consensus (100 samples, post tie-break).}
\label{tab:llmj_human}
\small
\begin{tabular}{lcc}
\toprule
\textbf{Dimension} & \textbf{Agreement (\%)} & \textbf{Strength} \\
\midrule
Sentiment & 99.0 & Strong \\
Topic & 97.0 & Strong \\
Formality & 89.0 & Good \\
Fluency & 73.0 & Moderate \\
\bottomrule
\end{tabular}
\end{table}

\textbf{Error analysis.}
Where the LLM judge disagrees with humans, errors concentrate in specific label bins: formality ``informal'' (55\% agreement vs.\ 99\% for ``formal'') and fluency middle scores (score 3: 55\% agreement), while extreme fluency scores and formal formality are reliable ($\geq$78\%). These results justify using LLM-judge scores throughout our evaluation, with fluency treated as a coarser signal than the categorical dimensions.

\subsubsection{Metric Correlations}

Table~\ref{tab:quality_corr} reports Spearman rank correlations between quality metrics across all operating points (each datapoint is the mean metric value for one condition--method--$\alpha$ combination).
For MDLM, SEDD, and DREAM, Qwen2-7B PPL correlates strongly with LLM-judge fluency ($\rho = -0.77$ to $-0.88$, all $p < 0.001$), confirming that PPL rankings are reliable for these models.
For LLaDA, PPL shows near-zero correlation with fluency ($\rho = +0.21$), confirming that PPL is unreliable for this model; distinct bigram ratio (d2) instead correlates very strongly with fluency ($\rho = 0.93$), validating its use as a secondary quality indicator.

\begin{table}[h]
\centering
\caption{Quality metric correlations (Spearman $\rho$). Each datapoint is the mean metric value for one (condition, method, $\alpha$) operating point ($n = 1{,}140$ per model). $^{***}p<0.001$.}
\label{tab:quality_corr}
\small
\begin{tabular}{l cc}
\toprule
Model & Qwen2 PPL vs.\ LLMJ & d2 vs.\ LLMJ \\
\midrule
MDLM & $-0.775^{***}$ & $-0.058\text{ (ns)}$ \\
SEDD & $-0.774^{***}$ & $+0.367^{***}$ \\
LLaDA & $+0.210^{***}$ & $+0.927^{***}$ \\
DREAM & $-0.875^{***}$ & $+0.447^{***}$ \\
\bottomrule
\end{tabular}
\end{table}

Table~\ref{tab:attr_corr} reports Spearman correlations between classifier confidence and LLM-judge attribute classification rate, aggregated across all four models. All three attributes show strong agreement ($\rho \geq 0.83$), with topic showing the highest correlation ($\rho = 0.92$) and style the lowest ($\rho = 0.83$), reflecting the inherent subjectivity of formality judgments.

\begin{table}[h]
\centering
\caption{Target attribute correlations: classifier confidence vs.\ LLM-judge attribute rate (Spearman $\rho$, aggregated across all 4 models, $n = 4{,}560$). All $p < 0.001$.}
\label{tab:attr_corr}
\small
\begin{tabular}{l c}
\toprule
Attribute & $\rho$ \\
\midrule
Sentiment & $0.869$ \\
Topic & $0.918$ \\
Style & $0.830$ \\
\bottomrule
\end{tabular}
\end{table}

\section{Extended Interpretability Analysis}
\label{app:extended_interp}

This section provides the full set of interpretability analyses summarized in Section~\ref{sec:interpretation} of the main paper. For each analysis type, we describe the generation methodology, present the complete figures spanning all models and layers, and discuss patterns not covered in the main text. Throughout, the three attributes studied are positive sentiment (IMDB), sports topic (AG News), and style (formality for MDLM/SEDD; informality for LLaDA/DREAM), as detailed in \S\ref{app:extraction_datasets}.

\subsection{Emergence Trajectories for All Attributes}
\label{app:emergence_all_attrs}

\paragraph{Methodology.} The main paper (Figure~\ref{fig:compact_emergence}) shows emergence trajectories for \emph{topic} features only. Here we present the full $4 \times 3$ grid covering all three attributes across all four models. For each model, attribute, and layer, the emergence trajectory is computed as follows:
\begin{enumerate}
    \item We generate $N$ samples (Table~\ref{tab:dynamics_per_model}) with SAE hooks active at every denoising step, recording per-step activations for the top-20 contrastive features per attribute per layer.
    \item We filter to classifier-confident samples (Table~\ref{tab:dynamics_thresholds}) and average across retained samples to obtain a per-step mean activation curve $\bar{h}(t)$.
    \item We compute absolute deviation from baseline: $|\bar{h}(t) - \bar{h}(0)|$, which captures emergence regardless of whether the mean activation increases or decreases during denoising.
    \item We apply moving-average smoothing with model-specific window sizes (MDLM/SEDD: no smoothing for 1024-step trajectories since they are sufficiently smooth; LLaDA: window${}=5$ for 64-step trajectories; DREAM: window${}=9$ for 256-step trajectories) to reduce per-step noise while preserving emergence timing.
    \item We normalize each curve to $[0, 1]$ via min-max scaling, so that $0$ corresponds to baseline (fully masked) and $1$ to full emergence.
\end{enumerate}

Each panel in Figure~\ref{fig:emergence_all_attrs} shows a single model--attribute combination, with one line per SAE layer. This reveals both the \emph{timing} of emergence (how early or late features activate) and the \emph{depth gradient} (whether shallow or deep layers commit first).

\begin{figure}[h]
\centering
\includegraphics[width=\textwidth]{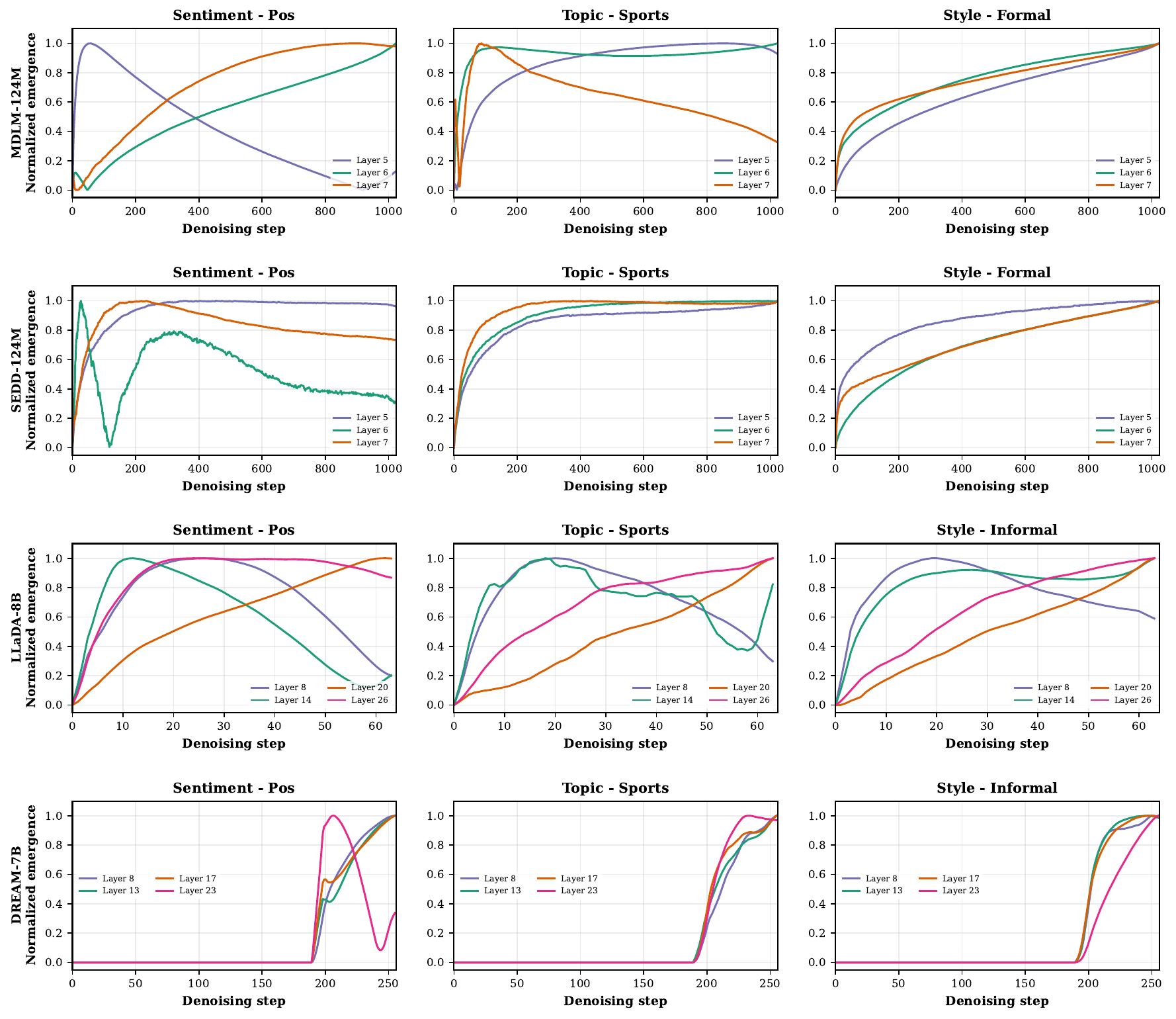}
\caption{Emergence trajectories for all four models and all three attributes (positive sentiment, sports topic, formal/informal style). Each panel shows normalized absolute deviation $|\bar{h}(t) - \bar{h}(0)|$ from baseline across the full denoising trajectory, with one line per SAE layer. MDLM and SEDD (top two rows) use 1024 denoising steps; LLaDA uses 64 steps; DREAM uses 256 steps. The main paper Figure~\ref{fig:compact_emergence} shows the topic column only.}
\label{fig:emergence_all_attrs}
\end{figure}

\subsection{Block Fractions Across All Layers}
\label{app:block_fractions_all_layers}

\paragraph{Methodology.} The main paper (Figure~\ref{fig:compact_blockfrac}) shows block fractions at the deepest layer per model only. Here we present block fractions at \emph{every} SAE layer, revealing how the temporal distribution of emergence shifts with depth. Block fractions are computed as follows:
\begin{enumerate}
    \item The denoising trajectory is divided into $K{=}8$ equal temporal blocks. For MDLM/SEDD (1024 steps), each block spans 128 steps; for LLaDA (64 steps), 8 steps; for DREAM (256 steps), 32 steps.
    \item For each block $b$, we compute the activation increase from the block's start to its end: $\delta_b = \bar{h}(e_b) - \bar{h}(s_b)$, where $s_b$ and $e_b$ are the first and last steps of block $b$.
    \item Negative deltas are clamped to zero: $\delta_b^+ = \max(0, \delta_b)$. This prevents blocks where activation temporarily decreases (e.g., due to non-monotonic trajectories) from receiving negative weight.
    \item Block fractions are computed by normalizing: $f_b = \delta_b^+ / \sum_{b'} \delta_{b'}^+$, so they sum to one.
\end{enumerate}

Figure~\ref{fig:block_fractions_all} presents the full results, organized as a 4-row grid (one row per model, one column per layer). The deepest-layer panels correspond to the main paper's Figure~\ref{fig:compact_blockfrac}.

\begin{figure}[h]
\centering
\includegraphics[width=\textwidth]{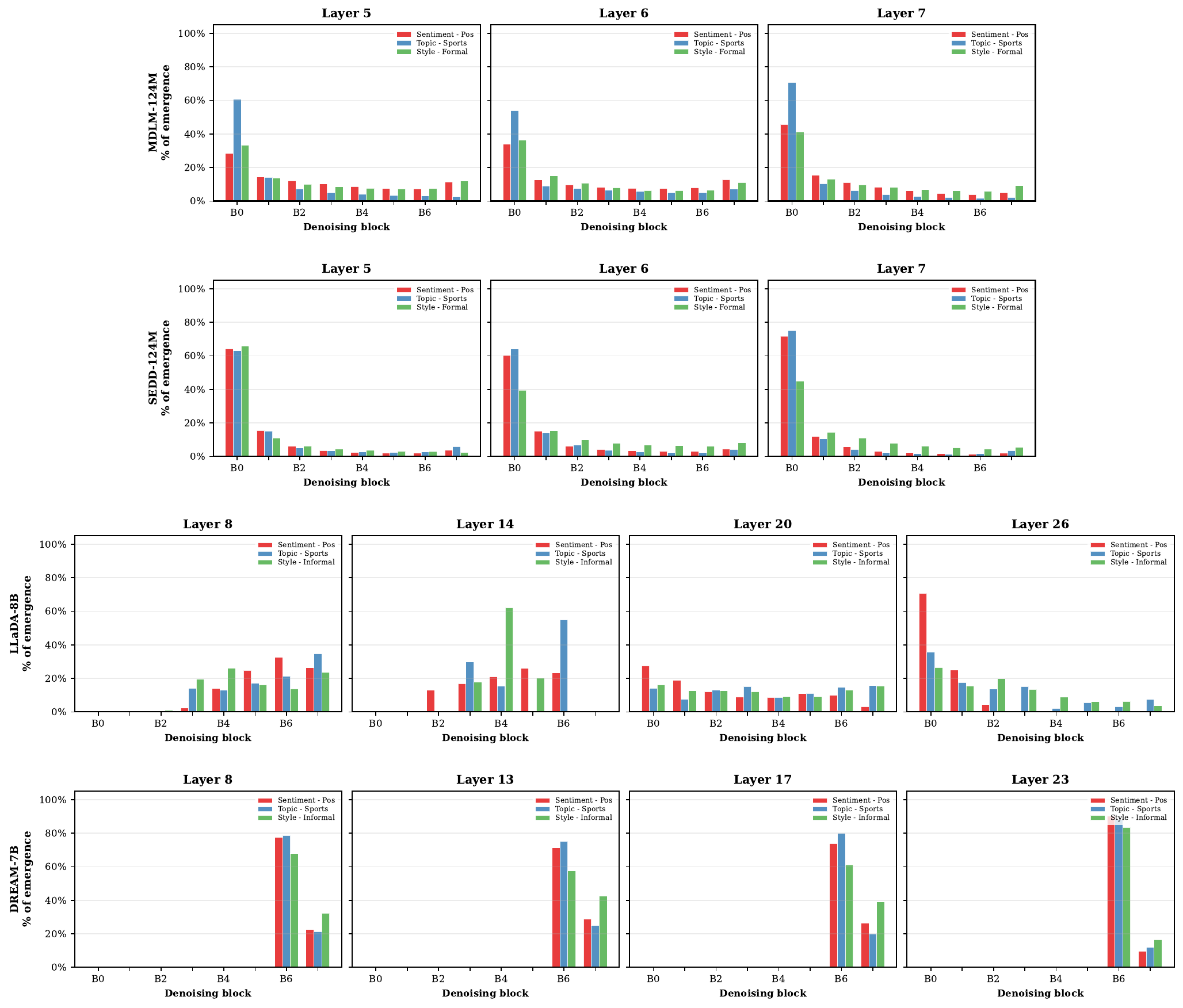}
\caption{Block fractions across all SAE layers for all four models. Each row corresponds to a model; each column to a layer. Comparing columns within a row reveals how the temporal distribution of emergence shifts with layer depth. MDLM and SEDD (rows~1--2, 3 layers each) show clear depth gradients; LLaDA (row~3, 4 layers) shows mild variation; DREAM (row~4, 4 layers) is uniformly late across all layers.}
\label{fig:block_fractions_all}
\end{figure}

\subsection{Cumulative Effect Size Across Models and Layers}
\label{app:effect_size_all}

\paragraph{Methodology.} The main paper (Figure~\ref{fig:effect_size}) compares cumulative effect size between MDLM and SEDD at layer~7 only. Here we present cumulative $|d|$ curves across all layers and all four models. For each model, attribute, and layer, the procedure is:
\begin{enumerate}
    \item We load the full set of $d_\text{SAE}$ Cohen's $d$ values computed during contrastive feature extraction (\S\ref{app:extraction_procedure}).
    \item We take the absolute value $|d_j|$ and sort features in descending order.
    \item We plot the cumulative sum $\sum_{i=1}^{r} |d_{(i)}|$ versus feature rank $r$ for the top~50 features.
\end{enumerate}

The cumulative $|d|$ curve captures how discriminative signal is distributed across features: a steep initial rise indicates that a few features carry most of the signal (concentrated representation), while a gradual rise indicates distributed encoding across many features. Figure~\ref{fig:sparsity_all} presents the full $4 \times (3\text{--}4)$ grid.

\begin{figure}[h]
\centering
\includegraphics[width=\textwidth]{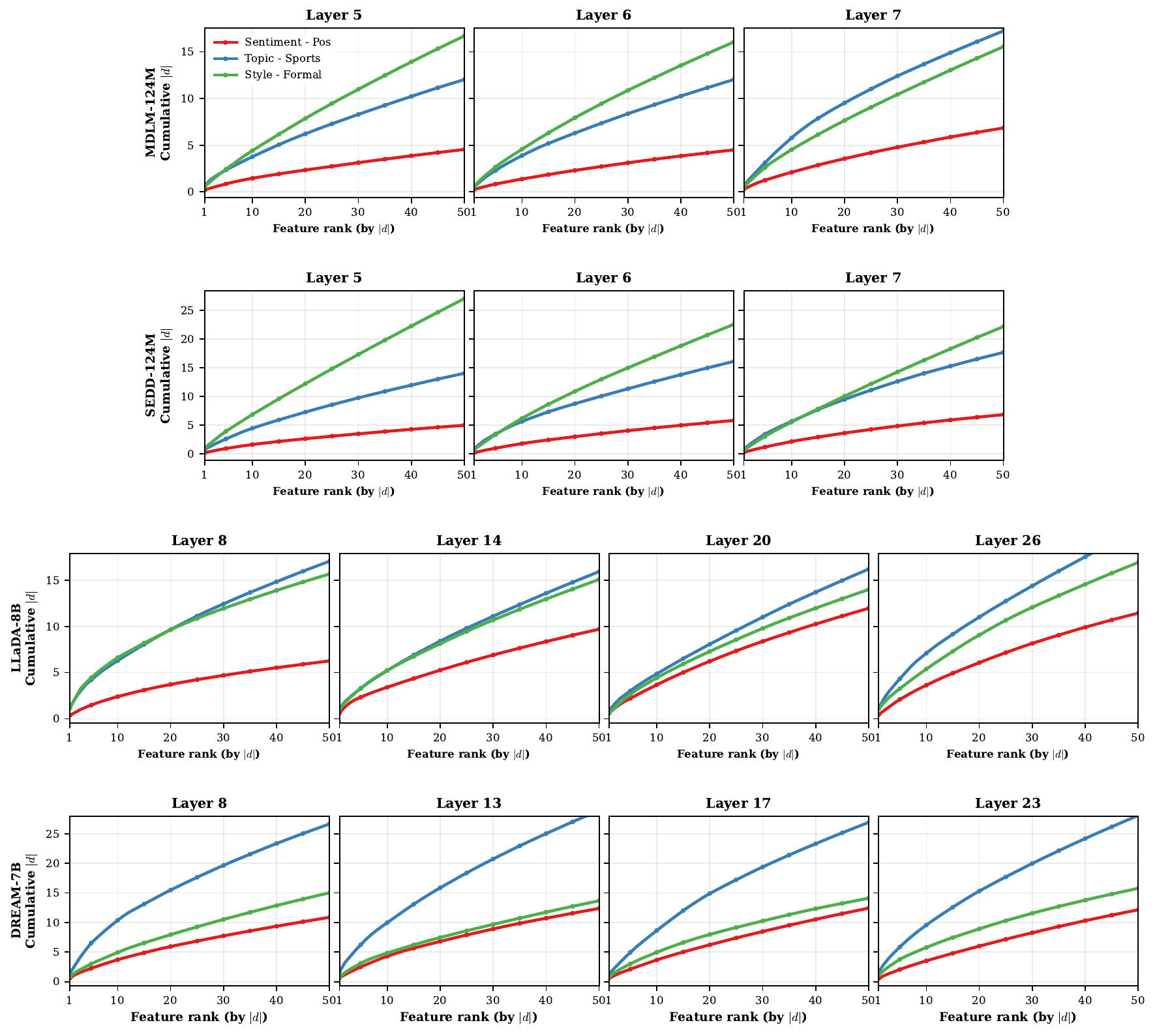}
\caption{Cumulative $|$Cohen's $d|$ vs.\ feature rank for all four models and all SAE layers. Each row is a model (MDLM-124M, SEDD-124M, LLaDA-8B, DREAM-7B); each column is a layer. Features are sorted by $|d|$ in descending order and the top~50 are shown. Topic features (blue) consistently dominate, accumulating the most signal per feature rank. Attributes: positive sentiment (red), sports topic (blue), and formal/informal style (green).}
\label{fig:sparsity_all}
\end{figure}

\subsection{Anticipatory vs.\ Reactive Encoding}
\label{app:anticipatory_analysis}

\paragraph{Methodology.} To understand whether contrastive SAE features encode attribute information \emph{before} a token is revealed (anticipatory) or \emph{after} (reactive), we partition feature activations at each denoising step by mask status. For each step $t$, each token position is either masked (not yet revealed) or unmasked (already decoded). We compute separate mean activations over the top-20 contrastive features for each pool:
\begin{itemize}
    \item $\bar{h}^{\text{mask}}(t)$: mean activation across masked positions.
    \item $\bar{h}^{\text{unmask}}(t)$: mean activation across unmasked positions.
\end{itemize}
The \emph{masked fraction} is then $\bar{h}^{\text{mask}}(t) / (\bar{h}^{\text{mask}}(t) + \bar{h}^{\text{unmask}}(t))$. Values above~0.5 indicate \emph{anticipatory} encoding: features fire preferentially on positions where no token has been revealed yet, suggesting the model is ``planning ahead'' for what will appear there. Values below~0.5 indicate \emph{reactive} encoding: features respond to tokens that have already been placed. Smoothing is applied with the same model-specific window sizes as for emergence curves (MDLM/SEDD: window${}=31$; LLaDA: window${}=5$; DREAM: window${}=9$).

Figure~\ref{fig:anticipatory_reactive} shows the masked fraction across denoising progress (0--100\%) for each attribute and layer on all four models.

\paragraph{Discussion.}

\begin{itemize}
    \item \textbf{MDLM.} Topic is consistently anticipatory across all layers (${\sim}$0.55--0.80 early on), with the strongest signal at layer~7. Sentiment is strongly reactive at layer~5 (starting at ${\sim}$0.10) and progressively less so at deeper layers; layer~7 shows anticipatory sentiment early (${\sim}$0.65). Formality exhibits a similar depth gradient: reactive at layer~5, mixed at deeper layers.
    \item \textbf{SEDD.} Qualitatively similar to MDLM---topic is anticipatory, sentiment is reactive at shallow layers---but with noisier trajectories and compressed dynamic range. This reflects SEDD's score-entropy training objective, which produces smaller absolute SAE activations: the masked fraction ratio remains meaningful (topic is still above 0.5, sentiment below), but the signal-to-noise ratio is lower because both masked and unmasked activations are closer to the ReLU threshold.
    \item \textbf{LLaDA.} Topic is anticipatory at layers~8--14, then transitions to reactive at deeper layers. Sentiment and informality are mildly reactive across most layers. Layer~26 shows anticipatory sentiment.
    \item \textbf{DREAM.} All attributes remain near 0.5 for the first ${\sim}$75\% of denoising, consistent with the absence of feature emergence during this phase (\S\ref{sec:temporal_dynamics}). Late-stage differentiation is minimal, with slight reactive shifts for topic at layers~8 and~23.
    \item \textbf{Convergence toward equilibrium.} All anticipatory signals decay toward 0.5 as denoising progresses and the mask ratio drops. This is expected: as most positions become unmasked, the distinction between masked and unmasked pools diminishes.
\end{itemize}

The anticipatory--reactive distinction has implications for steering: anticipatory features are active early and on positions that will be filled later, suggesting that intervening on these features during early denoising steps can influence the content of tokens not yet decoded. Reactive features, by contrast, are most active on already-visible tokens, making them suitable for reinforcing or modifying existing content in later denoising stages.

\begin{figure}[h]
\centering
\includegraphics[width=\textwidth]{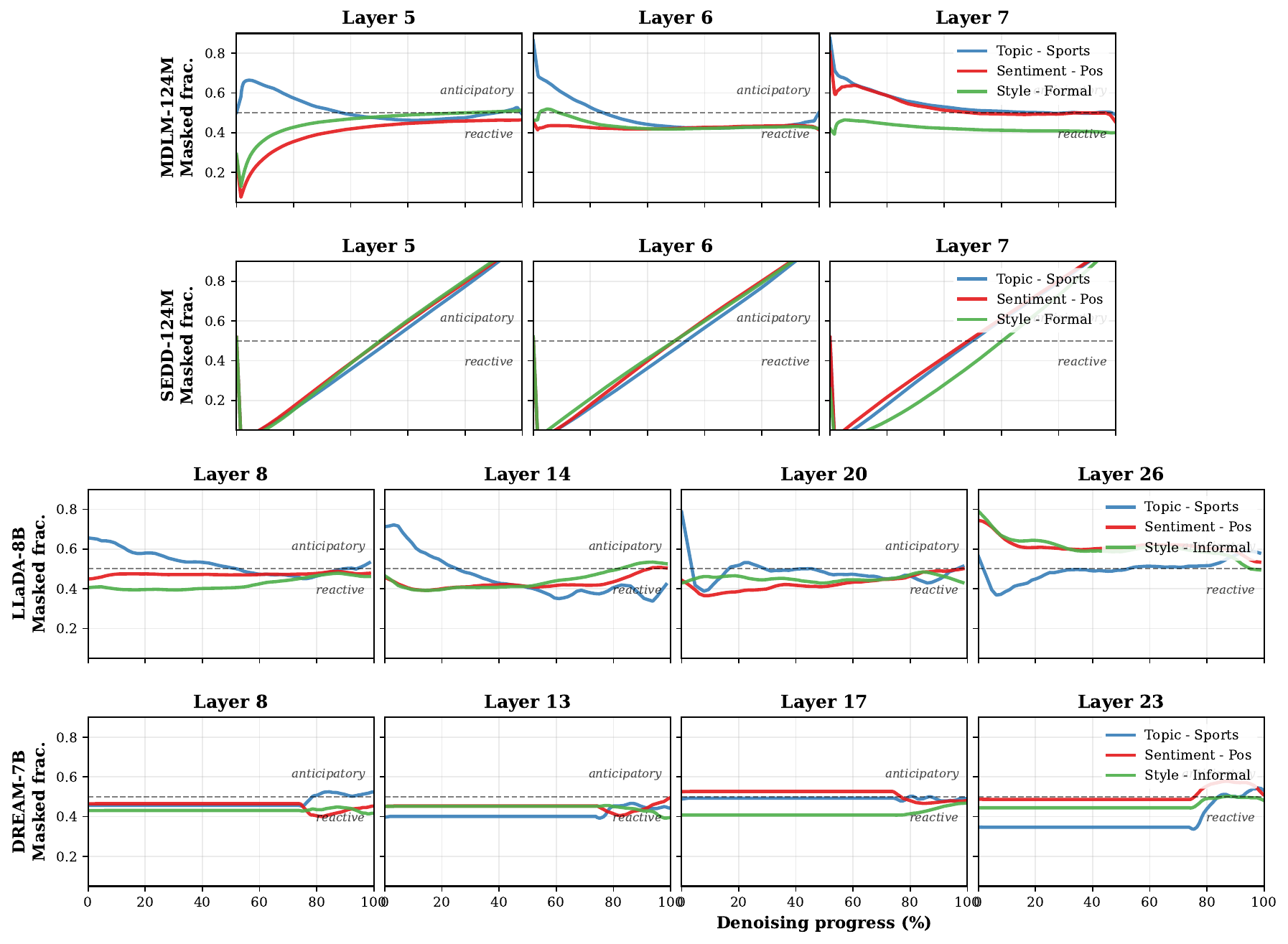}
\caption{Masked fraction of SAE feature activation vs.\ denoising progress for all four models and layers. Values above the dashed line (0.5) indicate \emph{anticipatory} encoding; values below indicate \emph{reactive} encoding. Each row is a model; each column is a layer. Attributes studied: positive sentiment (red), sports topic (blue), and formal/informal style (green). SEDD shows qualitatively similar patterns to MDLM but with compressed dynamic range due to smaller absolute SAE activations from score-entropy training (see text).}
\label{fig:anticipatory_reactive}
\end{figure}

\subsection{Cross-Model Feature Cosine Similarity}
\label{app:cross_model_cosine}

This section details the methodology and full results behind the cross-model feature alignment analysis.
The goal is to test whether models that share a training objective learn geometrically aligned SAE features, even when they differ in architecture and scale.

\paragraph{Comparison pairs.}
We compare three model pairs that isolate the effect of training objective from architecture and scale:

\begin{enumerate}[label=(\alph*)]
    \item \textbf{MDLM $\leftrightarrow$ SEDD} (124M $\leftrightarrow$ 124M): same GPT-2 architecture, same OpenWebText training data, \emph{different} training objectives (absorbing diffusion vs.\ score-entropy diffusion).
    \item \textbf{MDLM $\leftrightarrow$ Dream} (124M $\leftrightarrow$ 7B): \emph{same} absorbing-diffusion loss family, different architectures (GPT-2 vs.\ Qwen2.5), $56\times$ scale difference.
    \item \textbf{MDLM $\leftrightarrow$ LLaDA} (124M $\leftrightarrow$ 8B): \emph{same} absorbing-diffusion loss family, different architectures (GPT-2 vs.\ LLaMA-based), $64\times$ scale difference.
\end{enumerate}

\noindent For each pair, we select the top-50 features by $|\text{Cohen's } d|$ for sentiment and topic attributes from each model's SAE.

\paragraph{Method 1: Direct decoder-space comparison (same hidden dimension).}
When both models share the same hidden dimension $d$ (MDLM and SEDD, both $d{=}768$), we compare decoder directions directly:
\begin{equation}
    \mathrm{sim}(f_i^A, f_j^B) = \cos\!\big(\mathbf{w}^A_{\mathrm{dec},i},\; \mathbf{w}^B_{\mathrm{dec},j}\big),
\end{equation}
where $\mathbf{w}_{\mathrm{dec},i} \in \mathbb{R}^d$ is the $i$-th row of $\mathbf{W}_\mathrm{dec}$, $\ell_2$-normalized.
We compute the full $50 \times 50$ similarity matrix across each model's top features and extract, for each feature in model $A$, its \emph{best-match} similarity $\max_j \mathrm{sim}(f_i^A, f_j^B)$.
The ``mean best-match similarity'' averages these 50 best-match values.

This comparison is valid because both SAEs operate in the same residual-stream vector space: both models share the GPT-2 architecture and the SAEs are trained on activations at corresponding layers (L5, L6, L7).

\paragraph{Method 2: Vocabulary-projected comparison (different hidden dimensions).}
When hidden dimensions differ (MDLM $d{=}768$ vs.\ Dream $d{=}3584$ or LLaDA $d{=}4096$), direct decoder comparison is not possible.
Instead, we project each decoder direction through the model's unembedding matrix $\mathbf{W}_\mathrm{unembed}$ into a shared vocabulary space:
\begin{equation}
    \mathbf{v}_i^A = \frac{\hat{\mathbf{w}}^A_{\mathrm{dec},i} \cdot {\mathbf{W}^A_\mathrm{unembed}}^\top}{\big\|\hat{\mathbf{w}}^A_{\mathrm{dec},i} \cdot {\mathbf{W}^A_\mathrm{unembed}}^\top\big\|},
    \qquad
    \mathrm{sim}(f_i^A, f_j^B) = \mathbf{v}_i^A \cdot \mathbf{v}_j^B,
    \label{eq:vocab_proj_sim}
\end{equation}
where $\hat{\mathbf{w}}_{\mathrm{dec},i}$ denotes the $\ell_2$-normalized decoder row and $\mathbf{v}_i^A \in \mathbb{R}^{|V_\mathrm{shared}|}$ is the projected, renormalized vector restricted to shared vocabulary tokens.
The shared vocabulary consists of all tokens present in both models' tokenizers ($|V_\mathrm{shared}| = 42{,}284$ for MDLM--Dream; $42{,}450$ for MDLM--LLaDA).

This comparison is inherently \emph{asymmetric}: the best match for MDLM feature $f_i$ among Dream's features ($A \to B$ direction) need not coincide with the best match in the reverse direction ($B \to A$).
We report both directions separately and average them for the summary statistics.

\paragraph{Layer alignment.}
For each model and attribute, we select the layer with the highest mean $|\text{Cohen's } d|$ among the top-50 features:
\begin{itemize}
    \item Sentiment: MDLM L7 $\leftrightarrow$ Dream L17 $\leftrightarrow$ LLaDA L20 $\leftrightarrow$ SEDD L7.
    \item Topic: MDLM L7 $\leftrightarrow$ Dream L13 $\leftrightarrow$ LLaDA L26 $\leftrightarrow$ SEDD L7.
    \item Formality: MDLM L5 $\leftrightarrow$ Dream L23 $\leftrightarrow$ LLaDA L26 $\leftrightarrow$ SEDD L5.
\end{itemize}
For MDLM $\leftrightarrow$ SEDD, both models share the same architecture, so we compare at matched layers directly.

\paragraph{Full results.} Results are shown in Table~\ref{tab:cross_model_feature_alignment}.

\begin{table}[h]
\centering
\caption{Cross-model feature alignment. 
\emph{Best-match sim.} is the mean cosine similarity between each feature in model $A$ and its nearest neighbour in model $B$. For cross-architecture pairs, both search directions ($A{\to}B$ and $B{\to}A$) are shown. \emph{Overall mean} is the mean of the full $50 \times 50$ similarity matrix (a random-alignment baseline; expected $\approx 0$ for unrelated feature sets).}
\label{tab:cross_model_cosine_full}
\resizebox{\textwidth}{!}{%
\begin{tabular}{llcccc}
\toprule
Pair & Attribute & Best-match ($A{\to}B$) & Best-match ($B{\to}A$) & Max sim. & Overall mean \\
\midrule
\multirow{3}{*}{\shortstack[l]{MDLM $\leftrightarrow$ SEDD \\ \footnotesize(diff.\ loss, same arch.)}}
& Sentiment (L5) & \multicolumn{2}{c}{0.079} & 0.122 & --- \\
& Topic (L6)      & \multicolumn{2}{c}{0.085} & 0.122 & --- \\
& Formality (L5)  & \multicolumn{2}{c}{0.084} & 0.125 & --- \\
\midrule
\multirow{3}{*}{\shortstack[l]{MDLM $\leftrightarrow$ Dream \\ \footnotesize(same loss, diff.\ arch.)}}
& Sentiment (L7$\leftrightarrow$L17) & 0.170 & 0.229 & 0.410 & $+$0.007 \\
& Topic (L7$\leftrightarrow$L13)     & 0.166 & 0.259 & 0.414 & $-$0.003 \\
& Formality (L5$\leftrightarrow$L23) & 0.191 & 0.373 & 0.539 & $+$0.005 \\
\midrule
\multirow{3}{*}{\shortstack[l]{MDLM $\leftrightarrow$ LLaDA \\ \footnotesize(same loss, diff.\ arch.)}}
& Sentiment (L7$\leftrightarrow$L20) & 0.133 & 0.146 & 0.335 & $-$0.005 \\
& Topic (L7$\leftrightarrow$L26)     & 0.135 & 0.149 & 0.369 & $+$0.009 \\
& Formality (L5$\leftrightarrow$L26) & 0.142 & 0.178 & 0.471 & $-$0.008 \\
\bottomrule
\end{tabular}%
}
\label{tab:cross_model_feature_alignment}
\end{table}

\paragraph{Key observations.}

\begin{itemize}
    \item \textbf{Different loss $\Rightarrow$ orthogonal features.}
    MDLM and SEDD features show mean best-match cosine of $0.08$--$0.085$ across all attributes---indistinguishable from random alignment in 768-dimensional space (expected random cosine $\approx 1/\sqrt{d} \approx 0.036$, with the slightly higher observed values attributable to both models learning \emph{some} shared structure from the same training data).
    The maximum pairwise similarity is only $0.125$, confirming that no individual feature pair is aligned.

    \item \textbf{Same loss $\Rightarrow$ substantial alignment despite scale gap.}
    Same-loss-family pairs show ${\sim}2\times$ higher alignment.
    MDLM--Dream mean best-match similarity averages $0.23$ (across both directions and all three attributes), with individual feature pairs reaching $0.54$ cosine.
    MDLM--LLaDA averages $0.15$ mean / $0.47$ max---lower than Dream, likely reflecting the greater architectural distance (LLaMA-based vs.\ GPT-2).

    \item \textbf{Directional asymmetry.}
    Dream$\to$MDLM best-match similarity (0.23, 0.26, 0.37) consistently exceeds MDLM$\to$Dream (0.17, 0.17, 0.19) across sentiment, topic, and formality respectively.
    This indicates that Dream's larger SAE ($K{=}14{,}336$, $d{=}3{,}584$) learns more specific features---each Dream feature maps cleanly to a single MDLM counterpart---while MDLM's smaller SAE ($K{=}12{,}288$, $d{=}768$) produces more diffuse features that partially match multiple Dream features.
    For MDLM--LLaDA, the asymmetry is smaller ($\Delta \leq 0.04$), suggesting that LLaDA's features, while in a higher-dimensional space, are not as sharply specialized as Dream's.

    \item \textbf{Overall mean $\approx 0$ confirms no systematic bias.}
    The overall mean similarity (averaging all 2,500 entries in the $50 \times 50$ matrix) is effectively zero for all cross-architecture pairs ($|$mean$| < 0.01$), confirming that the high best-match values reflect genuine feature-level correspondence rather than a global bias in the projection space.

    \item \textbf{Formality $>$ topic $>$ sentiment in alignment strength.}
    Across both cross-architecture pairs, formality features show the highest best-match similarity (MDLM--Dream: 0.28; MDLM--LLaDA: 0.16), followed by topic (0.21; 0.14) and sentiment (0.20; 0.14). This ordering is consistent with the degree of lexical grounding: formality features project onto tight vocabulary clusters of register markers (e.g., contractions, punctuation), topic features onto domain-specific terms (e.g., ``touchdowns'', ``quarterback''), while sentiment features rely on subtler evaluative vocabulary that is less likely to be shared across models.
\end{itemize}

\section{Steering Analysis and Ablations}
\label{app:steering_analysis}

\subsection{Steering Results with Standard Deviations}
\label{app:steering_detail}
The steering results with standard deviations for single attributes are shown in Table~\ref{tab:steering_single_detail} and results with standard deviations for multi-attributes are shown in Table~\ref{tab:steering_multi_detail}. The best operating-point $\alpha$ for each method is indicated as a superscript.

\input{figures_tables/table_steering_results_v2_detail_single}
\input{figures_tables/table_steering_results_v2_detail_multi}

\subsection{Steering Trade-off Curves}
\label{app:tradeoff_curves}

Figures~\ref{fig:tradeoff_mdlm_full}--\ref{fig:tradeoff_dream_full} present the full steering trade-off curves for all four models with all methods. Each figure shows target classifier confidence (top) and LLM-judge fluency (bottom) across the $\alpha$ sweep for all 7 attribute conditions. The main paper (Figure~\ref{fig:tradeoff_mdlm}) shows SAE methods only on MDLM.

\begin{figure}[h]
\centering
\includegraphics[width=\textwidth]{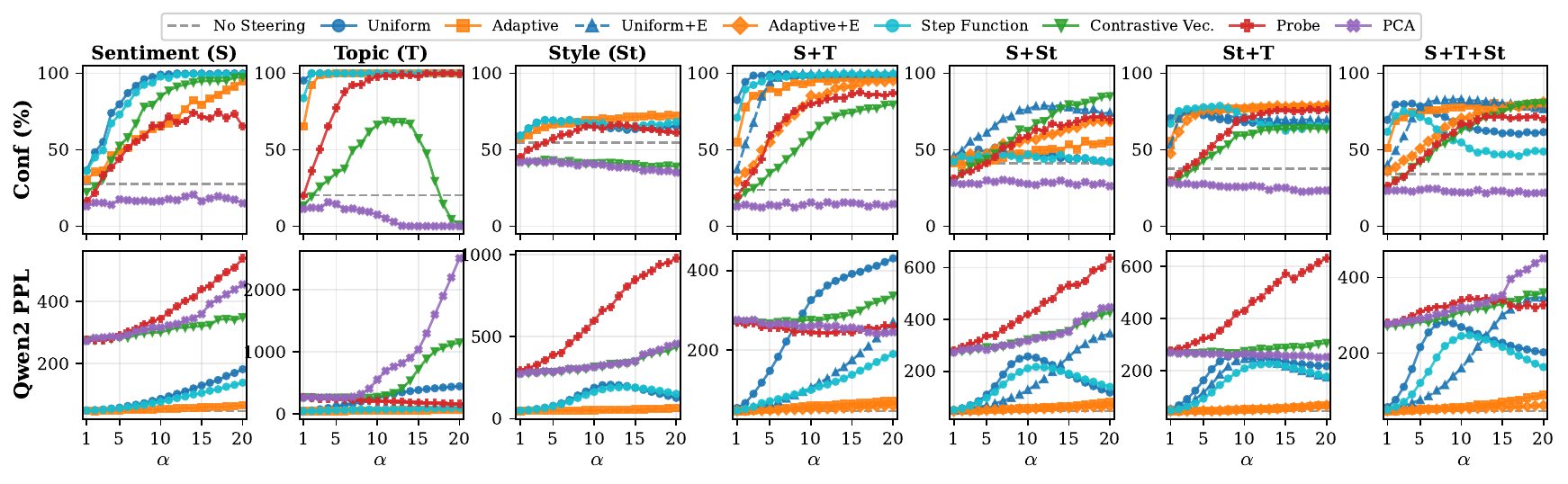}
\caption{MDLM-124M steering trade-off curves (all methods). Top: target confidence vs.\ $\alpha$. Bottom: Qwen2-7B perplexity vs.\ $\alpha$. MDLM shows the clearest Adaptive--Uniform separation: Adaptive matches or exceeds Uniform in steering while preserving lower perplexity.}
\label{fig:tradeoff_mdlm_full}
\end{figure}

\begin{figure}[h]
\centering
\includegraphics[width=\textwidth]{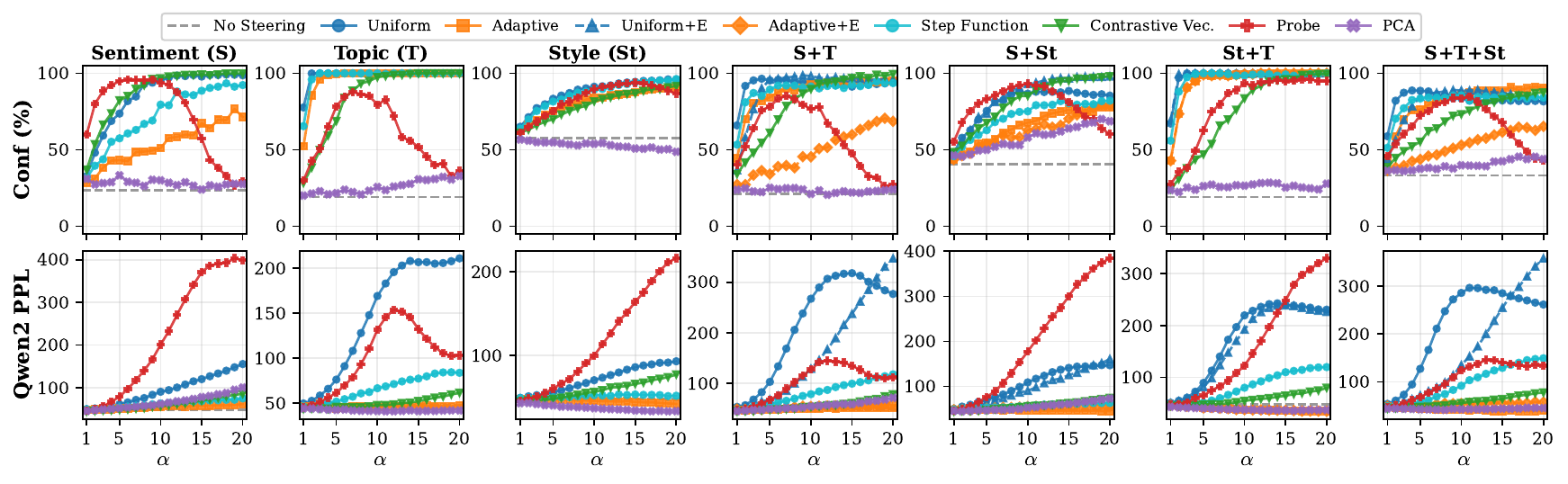}
\caption{SEDD-124M steering trade-off curves (all methods). Top: target confidence vs.\ $\alpha$. Bottom: Qwen2-7B perplexity vs.\ $\alpha$. Contrastive Vectors matches SAE methods on sentiment and topic, consistent with SEDD's distributed feature geometry (\S\ref{sec:temporal_dynamics}).}
\label{fig:tradeoff_sedd_full}
\end{figure}

\begin{figure}[h]
\centering
\includegraphics[width=\textwidth]{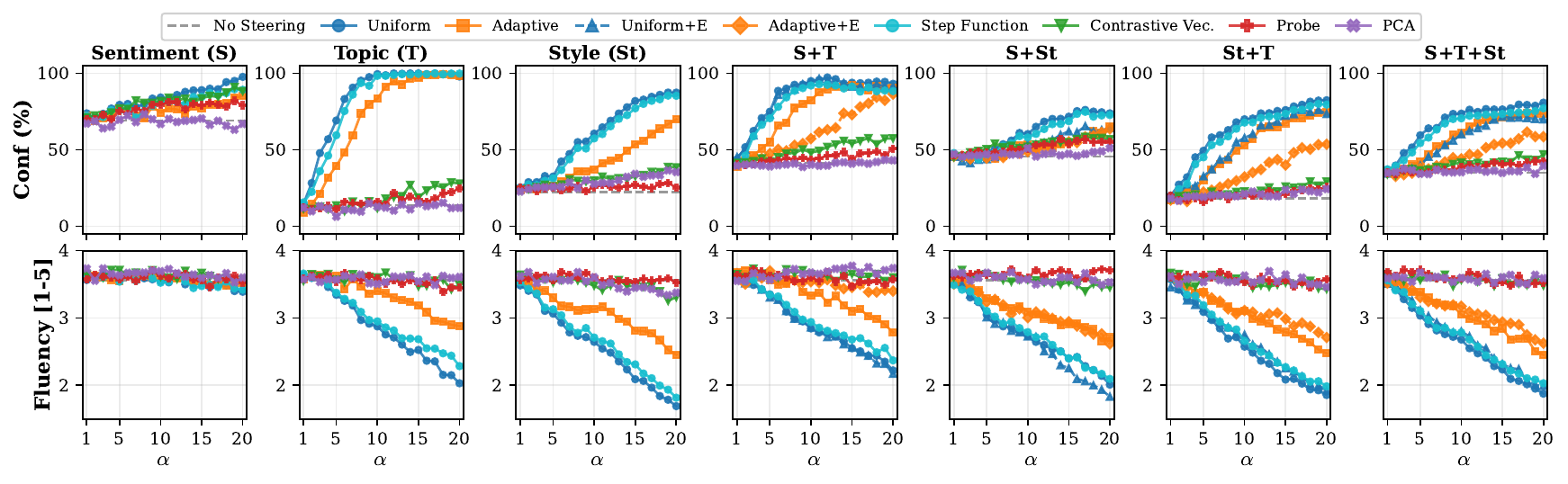}
\caption{LLaDA-8B steering trade-off curves (all methods). Top: target confidence vs.\ $\alpha$. Bottom: LLM-judge fluency vs.\ $\alpha$. SAE methods consistently outperform baselines, with the gap largest for topic and multi-attribute conditions.}
\label{fig:tradeoff_llada_full}
\end{figure}

\begin{figure}[h]
\centering
\includegraphics[width=\textwidth]{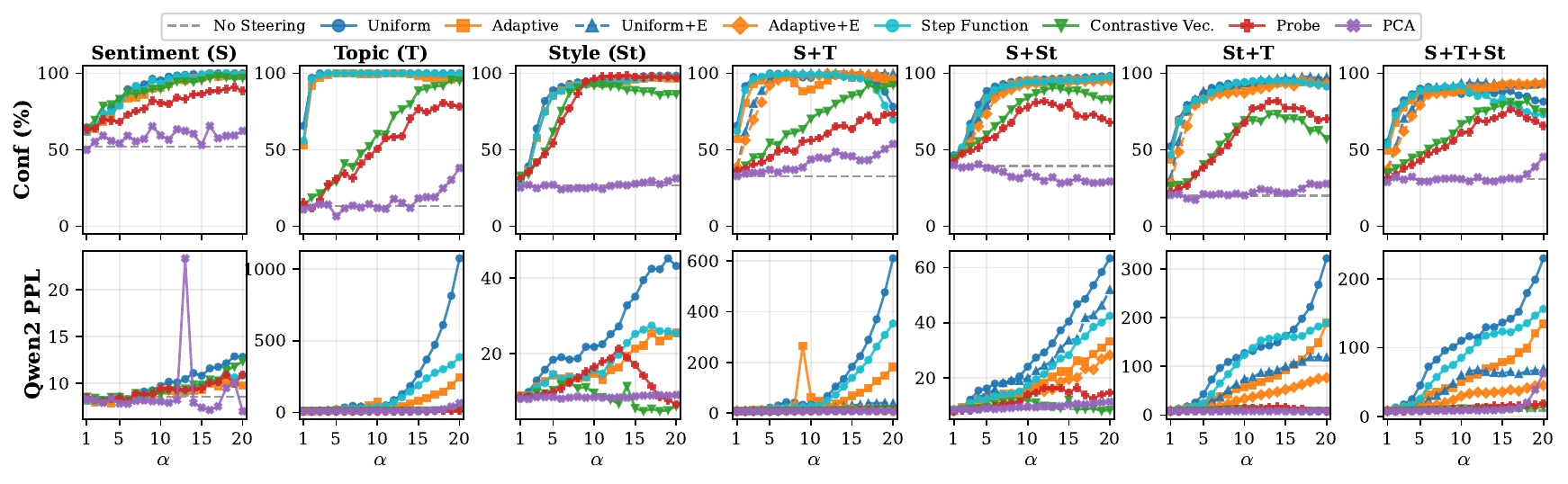}
\caption{DREAM-7B steering trade-off curves (all methods). Top: target confidence vs.\ $\alpha$. Bottom: Qwen2-7B perplexity vs.\ $\alpha$. DREAM achieves near-perfect steering at moderate $\alpha$ but exhibits the sharpest quality degradation, consistent with its concentrated late-block emergence.}
\label{fig:tradeoff_dream_full}
\end{figure}

\subsection{Statistical Analysis of Steering Quality}
\label{app:steering_stats}

We conduct four complementary analyses to quantify the quality advantage of Adaptive steering over Uniform and Step Function baselines.
Throughout, we compare the three non-augmented SAE methods: \textbf{Uniform}, \textbf{Step Function}, and \textbf{Adaptive}.
We exclude the $+E$ variants (Uniform+E, Adaptive+E) to isolate the effect of the steering schedule.
All numbers are computed from per-sample data (200 samples per condition) in the consolidated evaluation files.

\subsubsection{Significance of Quality Advantage}
\label{app:significance}

For each of the 28 conditions (4 models $\times$ 7 attribute combinations), each method selects its best operating point $\alpha^* \in \{1,\ldots,20\}$ that maximizes target classifier confidence subject to the quality gate described in \S\ref{sec:exp_setup} (Qwen2-7B PPL $< 100$ for MDLM, SEDD, and DREAM; LLM-judge fluency $\geq 50\%$ of baseline for LLaDA).
We then compare Adaptive's per-sample quality distribution against the second-best method using a one-sided Mann--Whitney $U$ test ($p < 0.05$).
The quality metric is Qwen2-7B perplexity (lower is better) for MDLM, SEDD, and DREAM, and LLM-judge fluency (higher is better) for LLaDA.

Table~\ref{tab:significance} reports the results.
Adaptive achieves significantly better quality than the second-best method in \textbf{22 of 28 conditions (78.6\%)}: 6/7 for MDLM (86\%), 7/7 for SEDD (100\%), 6/7 for LLaDA (86\%), and 3/7 for DREAM (43\%).
The non-significant cases are MDLM S+St (where Step Function's PPL is marginally lower), LLaDA Sentiment (where both methods produce high fluency, $p = 0.088$), and four DREAM conditions where aggressive steering at high $\alpha$ degrades PPL.

When using LLM-judge fluency as the quality metric for all models, significance drops to 14/28 (50\%) because fluency on small models (MDLM, SEDD) is heavily quantized around integer values, compressing meaningful PPL differences into near-ties.

\begin{table}[h]
\centering
\caption{Significance of Adaptive quality advantage at best operating points. For each condition, we compare Adaptive vs.\ the second-best method (2nd) using a one-sided Mann--Whitney $U$ test. Quality metric: Qwen2-7B PPL for MDLM/SEDD/DREAM (P, lower better), LLM-judge fluency for LLaDA (F, higher better). $^{***}p<0.001$, $^{**}p<0.01$, $^{*}p<0.05$.}
\label{tab:significance}
\tiny
\setlength{\tabcolsep}{3pt}
\begin{tabular}{ll rr l rr r l}
\toprule
Model & Combo & $\alpha_\text{A}$ & $Q_\text{A}$ & 2nd Best & $\alpha_\text{2nd}$ & $Q_\text{2nd}$ & $p$-value & \\
\midrule
MDLM & Sent & 20 & 64.6\textsuperscript{P} & Uniform & 11 & 94.0\textsuperscript{P} & $<$0.001 & $^{***}$ \\
MDLM & Topic & 9 & 54.5\textsuperscript{P} & Uniform & 4 & 97.2\textsuperscript{P} & $<$0.001 & $^{***}$ \\
MDLM & Style & 17 & 60.0\textsuperscript{P} & Step Fn & 4 & 66.2\textsuperscript{P} & $<$0.001 & $^{***}$ \\
MDLM & S+T & 18 & 68.8\textsuperscript{P} & Step Fn & 8 & 87.9\textsuperscript{P} & $<$0.001 & $^{***}$ \\
MDLM & S+St & 20 & 81.2\textsuperscript{P} & Step Fn & 4 & 79.1\textsuperscript{P} & 0.988 & \\
MDLM & St+T & 19 & 67.3\textsuperscript{P} & Uniform & 3 & 91.8\textsuperscript{P} & $<$0.001 & $^{***}$ \\
MDLM & S+T+St & 14 & 65.5\textsuperscript{P} & Step Fn & 3 & 70.3\textsuperscript{P} & $<$0.001 & $^{***}$ \\
\midrule
SEDD & Sent & 19 & 59.3\textsuperscript{P} & Step Fn & 18 & 71.5\textsuperscript{P} & $<$0.001 & $^{***}$ \\
SEDD & Topic & 9 & 44.7\textsuperscript{P} & Uniform & 4 & 65.8\textsuperscript{P} & $<$0.001 & $^{***}$ \\
SEDD & Style & 19 & 42.6\textsuperscript{P} & Step Fn & 20 & 51.8\textsuperscript{P} & $<$0.001 & $^{***}$ \\
SEDD & S+T & 20 & 51.2\textsuperscript{P} & Uniform & 4 & 82.4\textsuperscript{P} & $<$0.001 & $^{***}$ \\
SEDD & S+St & 19 & 45.2\textsuperscript{P} & Step Fn & 20 & 63.3\textsuperscript{P} & $<$0.001 & $^{***}$ \\
SEDD & St+T & 20 & 34.7\textsuperscript{P} & Uniform & 5 & 89.5\textsuperscript{P} & $<$0.001 & $^{***}$ \\
SEDD & S+T+St & 19 & 38.8\textsuperscript{P} & Step Fn & 10 & 91.6\textsuperscript{P} & $<$0.001 & $^{***}$ \\
\midrule
LLaDA & Sent & 20 & 3.52\textsuperscript{F} & Step Fn & 19 & 3.44\textsuperscript{F} & 0.088 & \\
LLaDA & Topic & 18 & 2.94\textsuperscript{F} & Step Fn & 17 & 2.54\textsuperscript{F} & $<$0.001 & $^{***}$ \\
LLaDA & Style & 20 & 2.46\textsuperscript{F} & Step Fn & 19 & 1.94\textsuperscript{F} & $<$0.001 & $^{***}$ \\
LLaDA & S+T & 15 & 3.10\textsuperscript{F} & Step Fn & 11 & 2.85\textsuperscript{F} & $<$0.001 & $^{***}$ \\
LLaDA & S+St & 20 & 2.71\textsuperscript{F} & Uniform & 17 & 2.29\textsuperscript{F} & $<$0.001 & $^{***}$ \\
LLaDA & St+T & 20 & 2.48\textsuperscript{F} & Step Fn & 19 & 2.06\textsuperscript{F} & $<$0.001 & $^{***}$ \\
LLaDA & S+T+St & 18 & 2.69\textsuperscript{F} & Step Fn & 20 & 2.03\textsuperscript{F} & $<$0.001 & $^{***}$ \\
\midrule
DREAM & Sent & 16 & 10.0\textsuperscript{P} & Step Fn & 18 & 11.0\textsuperscript{P} & $<$0.001 & $^{***}$ \\
DREAM & Topic & 12 & 27.0\textsuperscript{P} & Step Fn & 11 & 43.2\textsuperscript{P} & $<$0.001 & $^{***}$ \\
DREAM & Style & 17 & 25.4\textsuperscript{P} & Step Fn & 18 & 25.9\textsuperscript{P} & 0.173 & \\
DREAM & S+T & 14 & 48.5\textsuperscript{P} & Uniform & 5 & 21.8\textsuperscript{P} & 1.000 & \\
DREAM & S+St & 20 & 33.3\textsuperscript{P} & Step Fn & 20 & 42.5\textsuperscript{P} & $<$0.001 & $^{***}$ \\
DREAM & St+T & 15 & 92.6\textsuperscript{P} & Step Fn & 8 & 83.9\textsuperscript{P} & 0.999 & \\
DREAM & S+T+St & 16 & 82.5\textsuperscript{P} & Uniform & 6 & 72.7\textsuperscript{P} & 1.000 & \\
\bottomrule
\end{tabular}
\end{table}

\subsubsection{Matched-$\alpha$ Quality Ordering}
\label{app:matched_alpha}

We test whether the quality ordering $\text{Adaptive} \geq \text{Step Function} \geq \text{Uniform}$ holds at matched steering strength.
For every (model, combo, $\alpha$) triplet where all three methods have data ($n = 560$ points: 4 models $\times$ 7 combos $\times$ 20 $\alpha$ values), we check whether the ordering holds on both quality metrics independently.

Table~\ref{tab:matched_alpha} reports the results.
On LLM-judge fluency, the full ordering holds in \textbf{86.4\%} of cases, with Adaptive $\geq$ Uniform in 98.0\%.
Using the model-specific quality metric (Qwen2-7B PPL for MDLM/SEDD/DREAM, fluency for LLaDA), the full ordering holds in \textbf{88.2\%} of cases.
The ordering is remarkably consistent across models: even the weakest model (DREAM, 79--80\%) shows the predicted trend in the vast majority of conditions.

\begin{table}[h]
\centering
\caption{Matched-$\alpha$ quality ordering. Percentage of (model, combo, $\alpha$) points where the ordering holds. $n=560$ total (4 models $\times$ 7 combos $\times$ 20 $\alpha$).}
\label{tab:matched_alpha}
\small
\begin{tabular}{l cccc}
\toprule
Ordering & \multicolumn{2}{c}{LLMJ Fluency} & \multicolumn{2}{c}{Model-specific} \\
\cmidrule(lr){2-3} \cmidrule(lr){4-5}
 & Count & \% & Count & \% \\
\midrule
Adaptive $\geq$ Step Fn $\geq$ Uniform & 484/560 & 86.4 & 494/560 & 88.2 \\
Adaptive $\geq$ Step Fn & 518/560 & 92.5 & 532/560 & 95.0 \\
Step Fn $\geq$ Uniform & 524/560 & 93.6 & 520/560 & 92.9 \\
Adaptive $\geq$ Uniform & 549/560 & 98.0 & 546/560 & 97.5 \\
\midrule
\multicolumn{5}{l}{\textit{Per-model breakdown (LLMJ fluency, full ordering):}} \\
\quad MDLM & 134/140 & 95.7 & 125/140 & 89.3 \\
\quad SEDD & 122/140 & 87.1 & 139/140 & 99.3 \\
\quad LLaDA & 118/140 & 84.3 & 118/140 & 84.3 \\
\quad DREAM & 110/140 & 78.6 & 112/140 & 80.0 \\
\bottomrule
\end{tabular}
\end{table}

\subsubsection{$\alpha$ Sensitivity}
\label{app:alpha_sensitivity}

We test $\alpha$ sensitivity by fixing one global $\alpha^*$ per (model, method) pair across all 7 conditions, specifically, the $\alpha$ that maximizes mean confidence while passing the quality gate described in \S\ref{sec:exp_setup} for all conditions simultaneously, then measuring the per-condition confidence drop relative to the per-condition optimum.

Table~\ref{tab:alpha_sensitivity} reports the mean and maximum drop in percentage points.
Adaptive is robust: its global $\alpha^*$ sits at 15--20 and incurs only 0.2--1.2~pp mean drop (maximum 3.2~pp).
Uniform is forced to low global $\alpha$ values (2--7 for MDLM/SEDD/DREAM) because higher values violate the quality gate for at least one condition; this low setting under-steers other conditions, producing mean drops of 3.1--9.2~pp and worst-case drops up to 50.9~pp.
Step Function falls in between (mean 1.1--5.6~pp, max 4.5--32.2~pp).
The asymmetry arises because Adaptive concentrates its intervention where attributes form, so increasing $\alpha$ does not degrade generation quality the way uniform-weight intervention does.

\begin{table}[h]
\centering
\caption{$\alpha$ sensitivity: confidence drop (pp) when using a single global $\alpha^*$ across all 7 conditions vs.\ per-condition optimal $\alpha$. $\alpha^*$ is the value that maximizes mean confidence subject to the quality gate passing for all conditions.}
\label{tab:alpha_sensitivity}
\small
\begin{tabular}{l ccc ccc ccc}
\toprule
 & \multicolumn{3}{c}{Uniform} & \multicolumn{3}{c}{Step Function} & \multicolumn{3}{c}{Adaptive} \\
\cmidrule(lr){2-4} \cmidrule(lr){5-7} \cmidrule(lr){8-10}
Model & mean & max & $\alpha^*$ & mean & max & $\alpha^*$ & mean & max & $\alpha^*$ \\
\midrule
MDLM & 9.1 & 50.9 & 2 & 5.6 & 32.2 & 4 & 0.6 & 3.2 & 20 \\
SEDD & 9.2 & 28.4 & 4 & 4.6 & 13.9 & 10 & 0.2 & 0.6 & 19 \\
LLaDA & 1.4 & 2.9 & 19 & 1.1 & 4.5 & 20 & 0.4 & 1.4 & 20 \\
DREAM & 3.1 & 9.1 & 7 & 3.2 & 10.8 & 8 & 1.2 & 2.1 & 15 \\
\bottomrule
\end{tabular}
\end{table}

The rank ordering from Analysis~\ref{app:matched_alpha} is also stable across $\alpha$ ranges.
Table~\ref{tab:alpha_rank_stability} shows that the full quality ordering (Adaptive $\geq$ Step Function $\geq$ Uniform) holds in 75--80\% of points even at the lowest $\alpha$ range [1,3] and strengthens to 86--88\% at the full [1,20] range, confirming that the Adaptive advantage is not an artifact of specific $\alpha$ values.

\begin{table}[h]
\centering
\caption{Rank stability across $\alpha$ ranges. Percentage of matched (model, combo, $\alpha$) points where the full quality ordering (Adaptive $\geq$ Step $\geq$ Uniform) holds, computed separately for LLMJ fluency and model-specific quality.}
\label{tab:alpha_rank_stability}
\small
\begin{tabular}{l cc cc}
\toprule
$\alpha$ range & \multicolumn{2}{c}{LLMJ Fluency} & \multicolumn{2}{c}{Model-specific} \\
\cmidrule(lr){2-3} \cmidrule(lr){4-5}
 & Count & \% & Count & \% \\
\midrule
$[1, 3]$ & 63/84 & 75.0 & 67/84 & 79.8 \\
$[1, 5]$ & 110/140 & 78.6 & 116/140 & 82.9 \\
$[1, 10]$ & 235/280 & 83.9 & 240/280 & 85.7 \\
$[1, 15]$ & 357/420 & 85.0 & 369/420 & 87.9 \\
$[1, 20]$ & 484/560 & 86.4 & 494/560 & 88.2 \\
\bottomrule
\end{tabular}
\end{table}

\subsubsection{Matched-Confidence and Matched-Quality Comparisons}
\label{app:matched_conf}

The previous analyses compare methods at the same $\alpha$, but different methods may achieve different confidence levels at the same $\alpha$.
To control for steerability, we compare quality at \emph{matched target confidence} using an interpolation-based approach.

For each (model, combo) pair, we compute each method's confidence--quality curve across $\alpha \in \{1,\ldots,20\}$.
We identify the overlapping confidence range where all compared methods can operate, then probe at intervals of 2 percentage points within that range.
At each probe, we linearly interpolate each method's quality (between adjacent $\alpha$ values) and compare.

Table~\ref{tab:matched_conf} reports the results.
Using LLM-judge fluency as the quality metric, Adaptive produces better quality than Uniform in \textbf{84.0\%} of matched-confidence probes and better than Step Function in \textbf{72.8\%}.
Using the model-specific metric, the rates are 82.2\% and 71.5\% respectively.
The advantage is consistent across models, with SEDD showing the strongest results (93--100\% Adaptive $>$ Uniform).

\begin{table}[h]
\centering
\caption{Matched-confidence quality comparison. At each confidence level within the overlap range (probed every 2~pp), the percentage of probes where one method produces better quality than the other.}
\label{tab:matched_conf}
\small
\begin{tabular}{l cc cc}
\toprule
Comparison & \multicolumn{2}{c}{LLMJ Fluency} & \multicolumn{2}{c}{Model-specific} \\
\cmidrule(lr){2-3} \cmidrule(lr){4-5}
 & Count & \% & Count & \% \\
\midrule
Adaptive $>$ Uniform & 405/482 & 84.0 & 396/482 & 82.2 \\
Adaptive $>$ Step Fn & 389/534 & 72.8 & 382/534 & 71.5 \\
Step Fn $>$ Uniform & 332/516 & 64.3 & 343/516 & 66.5 \\
\midrule
\multicolumn{5}{l}{\textit{Per-model breakdown (Adaptive $>$ Uniform, LLMJ fluency):}} \\
\quad MDLM & 62/67 & 92.5 & 60/67 & 89.6 \\
\quad SEDD & 104/104 & 100.0 & 89/104 & 85.6 \\
\quad LLaDA & 115/157 & 73.2 & 115/157 & 73.2 \\
\quad DREAM & 124/154 & 80.5 & 132/154 & 85.7 \\
\bottomrule
\end{tabular}
\end{table}

\paragraph{Matched-quality confidence comparison.}
We perform the complementary analysis: at matched quality, does one schedule achieve higher confidence?
The procedure mirrors the matched-confidence analysis above but along the quality axis.
For each (model, condition) pair we sort each method's points by quality, identify the overlapping quality range, and probe at evenly spaced intervals (step size 1.0 for PPL-based models, 0.02 for LLM-judge fluency on LLaDA).
At each probed quality level we linearly interpolate each method's confidence and compare.

Using the model-specific quality metric, Adaptive achieves higher confidence than Uniform in \textbf{46.7\%}, higher than Step Function in \textbf{47.5\%}, and Step Function achieves higher confidence than Uniform in \textbf{44.6\%} of matched-quality probes.
All values are near 50\%, confirming that no method consistently achieves stronger steering at the same quality level.
This is expected: Adaptive's advantage lies in quality preservation---at matched confidence it produces better text, but at matched quality it does not reach higher confidence.
The approximately even split demonstrates that Uniform's additional quality degradation does not translate into stronger attribute control.

\subsection{Feature Count Ablation}
\label{app:feature_count_ablation}

We steer sentiment toward positive using $N \in \{1, 2, 5, 10, 20, 50\}$ features per direction with Uniform steering at $\alpha{=}5$ and report the positive-direction classifier confidences for all three attributes. Table~\ref{tab:nf_ablation} shows that increasing $N$ raises target sentiment confidence but also increases interference on non-target attributes (topic and style). At $N{=}20$, all models achieve strong sentiment control (71--84\%) with modest interference; at $N{=}50$, SEDD formality shifts from 57.7\% baseline to 94.5\% and MDLM formality rises to 67\%. We therefore use $N{=}20$ features throughout.

\begin{table}[h]
\centering
\caption{Feature count ablation: positive-direction sentiment steering at $\alpha{=}5$, varying the number of features $N$. Pos = positive sentiment confidence (\%), Spt = sports topic confidence (\%), Frm = formality confidence (\%).}
\label{tab:nf_ablation}
\resizebox{\textwidth}{!}{%
\small
\begin{tabular}{l c ccc ccc ccc ccc}
\toprule
& & \multicolumn{3}{c}{\textbf{MDLM-124M}} & \multicolumn{3}{c}{\textbf{SEDD-124M}} & \multicolumn{3}{c}{\textbf{DREAM-7B}} & \multicolumn{3}{c}{\textbf{LLaDA-8B}} \\
\cmidrule(lr){3-5} \cmidrule(lr){6-8} \cmidrule(lr){9-11} \cmidrule(lr){12-14}
$N$ & & Pos & Spt & Frm & Pos & Spt & Frm & Pos & Spt & Frm & Pos & Spt & Frm \\
\midrule
\multicolumn{2}{l}{\textit{No Steering}} & 25.5 & 18.6 & 53.3 & 29.0 & 20.3 & 57.7 & 57.7 & 10.5 & 73.5 & 69.1 & 13.9 & 77.8 \\
\midrule
1  & & 41.5 & 20.0 & 51.0 & 29.0 & 23.0 & 61.5 & 69.5 & 7.5 & 82.5 & 73.5 & 7.0 & 81.5 \\
2  & & 50.5 & 23.5 & 50.0 & 37.5 & 26.0 & 66.5 & 70.0 & 8.0 & 82.0 & 72.5 & 9.0 & 80.0 \\
5  & & 65.0 & 22.0 & 46.0 & 53.5 & 24.5 & 77.5 & 78.0 & 6.5 & 85.0 & 75.0 & 8.0 & 83.5 \\
10 & & 75.0 & 26.5 & 51.5 & 65.0 & 20.5 & 80.0 & 80.5 & 5.0 & 90.5 & 79.0 & 7.0 & 83.5 \\
\rowcolor[gray]{0.92}
20 & & 84.0 & 25.0 & 55.5 & 71.5 & 20.0 & 90.0 & 82.5 & 4.0 & 92.5 & 80.0 & 3.5 & 82.0 \\
50 & & 92.0 & 18.0 & 67.0 & 88.5 & 17.5 & 94.5 & 89.5 & 7.5 & 91.5 & 78.0 & 6.0 & 84.5 \\
\bottomrule
\end{tabular}%
}
\end{table}

\subsection{Cross-Attribute Interference}
\label{app:interference}

For each method at its best operating point, we measure cross-attribute interference as the mean absolute percentage-point deviation of non-target attribute classifier confidences from their unsteered baselines. Given $m$ attributes steered simultaneously, the interference for method $M$ is $\frac{1}{|\mathcal{A}\setminus T|}\sum_{a\notin T} |c_a^{(M)} - c_a^{(\text{unsteered})}|$, where $c_a$ is the classifier confidence for attribute $a$. Figure~\ref{fig:interference} reports this metric across all models and steering conditions.

\begin{figure}[h]
\centering
\includegraphics[width=\textwidth]{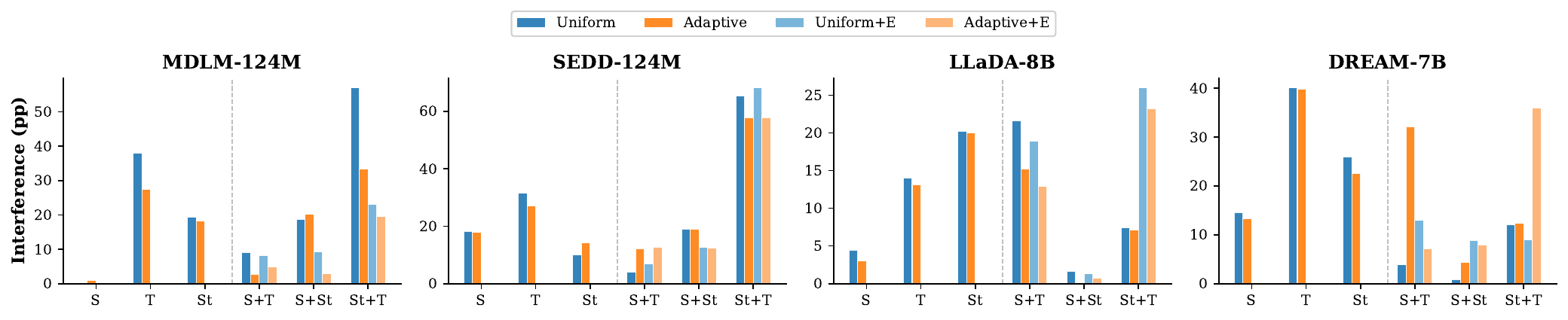}
\caption{Cross-attribute interference at the best operating points: mean absolute pp deviation of non-target attributes from unsteered baselines. On MDLM, Adaptive (orange) produces substantially less interference than Uniform (blue); on other models the gap narrows or reverses, reflecting model-specific temporal dynamics. E-ratio variants (lighter shades) appear only for multi-attribute conditions.}
\label{fig:interference}
\end{figure}

\subsection{Latency Overhead}
\label{app:latency}

Table~\ref{tab:latency} reports per-sample wall-clock times for unsteered generation, SAE steering (uniform schedule, $\alpha{=}5$), and contrastive-vector steering, measured over 20 single-sample generations per method after a warm-up pass.

\begin{table}[h]
\centering
\caption{Generation latency (seconds per sample, batch size 1). SAE overhead is the ratio of SAE steering time to baseline. CV = contrastive vector.}
\label{tab:latency}
\small
\begin{tabular}{lcccccc}
\toprule
\textbf{Model} & \textbf{Steps} & \textbf{Baseline (s)} & \textbf{SAE (s)} & \textbf{SAE Overhead} & \textbf{CV (s)} & \textbf{CV Overhead} \\
\midrule
MDLM (124M)  & 1024 & 25.66 & 30.07 & 1.17$\times$ & 25.73 & 1.00$\times$ \\
SEDD (124M)  & 1024 & 21.49 & 22.55 & 1.05$\times$ & 22.15 & 1.03$\times$ \\
DREAM (7B)   & 256  & 7.52  & 8.41  & 1.12$\times$ & 7.53  & 1.00$\times$ \\
LLaDA (8B)   & 64   & 0.83  & 1.25  & 1.51$\times$ & 0.81  & 0.97$\times$ \\
\bottomrule
\end{tabular}
\end{table}

SAE steering adds 5--51\% overhead depending on model size and step count. The overhead is lowest on SEDD (1.05$\times$), where the 1024-step analytic predictor dominates runtime and the SAE encode--decode ($768 \to 12{,}288 \to 768$) is negligible by comparison. MDLM (1.17$\times$) has a simpler per-step update, so the SAE cost is proportionally larger. For the 7B-scale models, DREAM (1.12$\times$, 256 steps) amortizes the SAE cost across expensive transformer forward passes, while LLaDA (1.51$\times$, 64 steps) has the highest overhead because each of its few steps is fast (${\sim}$13ms), making the per-step SAE cost ($4096 \to 16{,}384 \to 4096$ across 4 layers) a larger fraction. Contrastive-vector steering adds negligible overhead (${\sim}$1.00$\times$) since it requires only a single vector addition per layer. MDLM was benchmarked on an A10G GPU; all other models on A100.

\subsection{Qualitative Steering Examples}
\label{app:qualitative}

Tables~\ref{tab:qual_mdlm}--\ref{tab:qual_dream} show representative generations under triple-attribute steering (S+T+St) for each model, all four SAE steering modes, and baselines (Contrastive Vectors and Prompt where applicable). MDLM and SEDD are steered toward positive sentiment, sports topic, and formality; LLaDA and DREAM toward positive sentiment, sports topic, and informality. Texts are truncated; classifier confidences are from the off-the-shelf classifiers described in \S\ref{app:evaluation}.

\begin{table}[h]
\centering
\caption{MDLM qualitative examples (triple-attribute: positive, sports, formal). S/T = sentiment/topic confidence (\%); Frm = formality confidence (\%).}
\label{tab:qual_mdlm}
\resizebox{\textwidth}{!}{%
\footnotesize
\begin{tabular}{lccccp{8.5cm}}
\toprule
\textbf{Method} & \textbf{S} & \textbf{T} & \textbf{Frm} & \textbf{PPL} & \textbf{Generated Text (truncated)} \\
\midrule
Unsteered & 1 & 0 & 62 & 43 & Its immigration reform bill, making reference again to immigration law Tuesday night{\ldots} The disagreements declared on cases on which judges are working are ``a case of man's versus a judge, the issues that arise that have to be heard by one court.'' \\
\midrule
Uniform & 100 & 100 & 47 & 259 & with also the former rivals side and the stars for the replacement and they was make a breakthrough at the end of 2011, moving to the series win which capped a five-match Premier career{\ldots} \\
Adaptive & 100 & 100 & 50 & 54 & ``I'm really happy, I'm thinking and looking forward to Southampton,'' Thomas this week told Sky Sport. ``I am always hungry to go home.'' The Shearsborough stadium won't have any ``New England'' colours against Coventry{\ldots} \\
Uniform+E & 100 & 100 & 51 & 101 & went back into form, making the Crystal Palace bench on Tuesday night and his hopes of winning the game on Sunday and defeat Manchester United. The former first-time winner and two-time ace culminated on the pre-match session{\ldots} \\
Adaptive+E & 99 & 100 & 40 & 51 & went deeper into future, making it again to take news Tuesday night when his team will win the AFC championship Sunday at Dick Bowl in Spokane. ``I'm going there at 1 p.m.\ tomorrow to try to see what happens,'' Bowl said. \\
\midrule
Contrastive Vec. & 100 & 0 & 37 & 329 & with also the fan t; and feed through for the fans and you with the fan letters to some of her characters, and then they started doing me and fronted and played us those episodes{\ldots} \\
\bottomrule
\end{tabular}%
}
\end{table}

\begin{table}[h]
\centering
\caption{SEDD qualitative examples (triple-attribute: positive, sports, formal). S/T = sentiment/topic confidence (\%); Frm = formality confidence (\%).}
\label{tab:qual_sedd}
\resizebox{\textwidth}{!}{%
\footnotesize
\begin{tabular}{lccccp{8.5cm}}
\toprule
\textbf{Method} & \textbf{S} & \textbf{T} & \textbf{Frm} & \textbf{PPL} & \textbf{Generated Text (truncated)} \\
\midrule
Unsteered & 80 & 0 & 30 & 29 & everybody, that's what happens if you know people are really getting the pressure there is you look at no one else they're busy with that. I'm like things keep on happening every day is that it is thrusting into my life{\ldots} \\
\midrule
Uniform & 99 & 100 & 67 & 312 & complete the third and 2nd all Mexico and scored a one goal that had the European League victories final of ten and scored the final that turned for victory of World title against Brazil USA's Ners Joao{\ldots} \\
Adaptive & 100 & 100 & 61 & 49 & Two games hard-won, Wade again in the fourth quarter and scored again with one tip of Jimmy Wingard's fade-away drive to draw a one minute tie with the Pacers. He scored twice{\ldots} \\
Uniform+E & 100 & 100 & 74 & 135 & great nation is case of it all, although there is one of the most sophisticated players of the sport, Diego Brasil, who was restored for a new level to have with a strong analytical background{\ldots} \\
Adaptive+E & 99 & 100 & 77 & 102 & of Jordi has emerged as the next president in each club and country, the extraordinary result is that Spain has re-established itself{\ldots} In an elegant football football held to the centre of the Catalan club victory{\ldots} \\
\midrule
Contrastive Vec. & 100 & 100 & 85 & 36 & in the bottom of the 10th inning was the night when Babe Ruth tied the National League Division scoreless with a three-run solo third, the first time his career led, but fell foul. The Pirates were 2 ahead of the San Francisco Giants{\ldots} \\
\bottomrule
\end{tabular}%
}
\end{table}

\begin{table}[h]
\centering
\caption{LLaDA qualitative examples (triple-attribute: positive, sports, informal). S/T = sentiment/topic confidence (\%); Inf = informality confidence (\%).}
\label{tab:qual_llada}
\resizebox{\textwidth}{!}{%
\footnotesize
\begin{tabular}{lccccp{8.5cm}}
\toprule
\textbf{Method} & \textbf{S} & \textbf{T} & \textbf{Inf} & \textbf{PPL} & \textbf{Generated Text (truncated)} \\
\midrule
Unsteered & 25 & 0 & 12 & 6 & a great question. I have been thinking about it for a while and I am trying to come up with a solution that would work for me. I have a lot of money and I want to invest it in real estate. \\
\midrule
Uniform & 1 & 100 & 93 & 6 & what you mean by ``after the game'' \;\; lol \;\; but not at the end of the game \;\; but not at the end of the game \;\; but not at the end of the game{\ldots} \\
Adaptive & 100 & 83 & 70 & 5 & a great idea for the game. I think it would be a great way to make the game more fun and competitive by giving the players a big advantage at the start of the game. \\
Uniform+E & 94 & 100 & 72 & 11 & a great way to get a kick out of the game. you can get the ball at the start of the game and then after the game you can get it back and throw it in the end zone. i was a hater in the past. \\
Adaptive+E & 100 & 99 & 75 & 3 & a great way to get to know the game. I think this is a great way to make friends. I think this is a great way to get to know the game. \\
\midrule
Contrastive Vec. & 97 & 100 & 20 & 8 & a good question. It's always difficult to predict the outcome of a match, but it's also difficult to predict the outcome of a game. I think the important thing is to be prepared for the match{\ldots} \\
Prompt & 100 & 2 & 0 & 9 & a good book for people who are interested in sports. It has a lot of information about different sports and how to play them. I like the way it is written because it is easy to understand{\ldots} \\
\bottomrule
\end{tabular}%
}
\end{table}

\begin{table}[h]
\centering
\caption{DREAM qualitative examples (triple-attribute: positive, sports, informal). S/T = sentiment/topic confidence (\%); Inf = informality confidence (\%).}
\label{tab:qual_dream}
\resizebox{\textwidth}{!}{%
\footnotesize
\begin{tabular}{lccccp{8.5cm}}
\toprule
\textbf{Method} & \textbf{S} & \textbf{T} & \textbf{Inf} & \textbf{PPL} & \textbf{Generated Text (truncated)} \\
\midrule
Unsteered & 0 & 0 & 1 & 14 & a good way for beginners to understand that if you have an array that is large enough, you can treat it as a queue. However, when you use an array as a queue, it is easy to reach the capacity of the queue{\ldots} \\
\midrule
Uniform & 100 & 100 & 77 & 144 & a match in 3 final the all match good there time time and they was great team up all good lost and win a match of the game and go won a win of the game they 2 game{\ldots} \\
Adaptive & 100 & 100 & 76 & 37 & a great match. Its a great match i love it and the match was 20-20 and when the match started was 2-3 they played alot of people thought they could be in the back of the match go 4-5{\ldots} \\
Uniform+E & 100 & 100 & 87 & 22 & a good game, but it needs to be more fun and not so boring it is more of a best ball game. 1st game ever. this game was really fun. playa 11. this game was so great{\ldots} \\
Adaptive+E & 100 & 100 & 84 & 21 & a good game, but it needs to be more fun and not so boring it is more of a basket ball game. 1st game ever. this game was good but i liked the first game better. This game was awesome i liked it it was fun and i like it even though it was hard. Great game \\
\midrule
Contrastive Vec. & 100 & 1 & 67 & 28 & a great way to get kids to think about how many days are in a week! I love the whole concept of how many days in a week and then i can ask them how many days in a week{\ldots} \\
Prompt & 100 & 99 & 65 & 11 & a way to do that. It's kind of a good example of how to do it. It's not a good example of how to do it, but I think it's a great example of how to do it{\ldots} \\
\bottomrule
\end{tabular}%
}
\end{table}

On MDLM and SEDD, Uniform steering achieves high target confidence but produces garbled text (PPL $>$250), while Adaptive matches confidence at lower perplexity with coherent output. E-ratio calibration (Uniform+E, Adaptive+E) provides intermediate quality. On LLaDA and DREAM, all SAE modes successfully shift attributes from near-chance baselines while maintaining low perplexity; E-ratio variants achieve the highest simultaneous control across all three attributes. The DREAM unsteered example (formal, technical text with St$=$1\%) contrasts sharply with steered outputs that exhibit clear informality markers (lowercase, casual phrasing, slang).

\section*{Broader Impacts}

Mechanistic interpretability is a foundational tool for AI safety: models whose internal computations can be traced and steered are easier to audit, correct, and align. Our work extends this agenda to discrete diffusion language models, a class that has received less interpretability attention than autoregressive LLMs but is becoming competitive in capability. The framework supports legitimate downstream uses including content moderation, domain adaptation without retraining, and accessibility-oriented control of register or sentiment. The interpretability tools themselves let practitioners diagnose failures and understand where in the denoising trajectory particular attributes commit—information not visible from external behavior alone.

However, the same steering capability that enables beneficial control also enables manipulation: an adversary with white-box access could shift outputs toward biased or deceptive content while preserving fluency, in ways simple prompting cannot. We release no models, datasets, or steering pipelines targeting harmful attributes; trained SAEs and feature sets are intended for research on benign attributes and would need retraining for deployment

%% file: figures_tables/table_feature_vocab_sparsity_all_merged.tex
\begin{table*}[t]
\centering
\caption{Vocabulary grounding of top contrastive SAE features (via $\mathbf{W}_\text{dec} \cdot \mathbf{W}_\text{unembed}^\top$). For each model and attribute, the top-3 positive-$d$ and top-3 negative-$d$ features at the deepest SAE layer are shown.}
\label{tab:feature_vocab_all}
\small
\resizebox{\textwidth}{!}{%
\begin{tabular}{p{0.6cm}p{1cm}rlrl}
\toprule
\textbf{Model} & \textbf{Attribute} & \multicolumn{2}{c}{\textbf{Positive-$d$ (target class)}} & \multicolumn{2}{c}{\textbf{Negative-$d$ (contrast class)}} \\
\cmidrule(lr){3-4} \cmidrule(lr){5-6}
 & & \textbf{Feature ($d$)} & \textbf{Top Tokens} & \textbf{Feature ($d$)} & \textbf{Top Tokens} \\
\midrule
\multirow{9}{*}{\rotatebox[origin=c]{90}{MDLM (L7)}}
 & \multirow{3}{*}{Topic}
   & F6047 ($+$0.60) & midfield, Isles, winger, Trafford, Championship
   & F8014 ($-$0.67) & Securities, industrial, industry, regulatory, commercial \\
 & & F7923 ($+$0.58) & mph, Guinness, mph, Citation, maximum
   & F533 ($-$0.67) & 2030, 2020, average, 2025, 2050 \\
 & & F1000 ($+$0.57) & dyl, anooga, rongh, ickets, Sorceress
   & F6218 ($-$0.61) & Avg, [+, /\$, DAQ, avg \\
\cmidrule(lr){2-6}
 & \multirow{3}{*}{Sentiment}
   & F6624 ($+$0.17) & solid, ocally, Scand, composure, ACTION
   & F11410 ($-$0.30) & !!!!!, ???, !!!!, ……, !!! \\
 & & F5123 ($+$0.16) & DJs, eaturing, magic, creat, aples
   & F10932 ($-$0.27) & my, idiot, NS, replies, CP \\
 & & F6921 ($+$0.15) & explan, overview, Learn, depth, detailing
   & F2643 ($-$0.27) & textures, texture, weird, Illusion, guts \\
\cmidrule(lr){2-6}
 & \multirow{3}{*}{Formality}
   & F1866 ($+$0.45) & reportedly, spokeswoman, watchdog, compensation, commissioner
   & F5013 ($-$0.56) & …, I, !, myself, ?! \\
 & & F8014 ($+$0.37) & Securities, industrial, industry, regulatory, commercial
   & F751 ($-$0.53) & …, ..., [...], …, […] \\
 & & F7544 ($+$0.33) & quished, Instit, Aether, org, Presents
   & F7258 ($-$0.53) & GOODMAN, —, — \\
\midrule
\multirow{9}{*}{\rotatebox[origin=c]{90}{SEDD (L7)}}
 & \multirow{3}{*}{Topic}
   & F606 ($+$0.84) & matchups, rematch, matchup, challenger, clinch
   & F7260 ($-$0.73) & payable, cash, dollars, money, fees \\
 & & F5606 ($+$0.63) & championships, jerseys, championship, playoffs, playoff
   & F1914 ($-$0.66) & GDP, repairs, ppm, gallon, gallons \\
 & & F3948 ($+$0.45) & Featuring, teamed, eaturing, featuring, reunited
   & F9697 ($-$0.57) & liberal, government, neoliberal, privatization, atism \\
\cmidrule(lr){2-6}
 & \multirow{3}{*}{Sentiment}
   & F10956 ($+$0.27) & morrow, ,—, thence, tion, splend
   & F3169 ($-$0.21) & abusers, porn, sexist, moms, Kids \\
 & & F3948 ($+$0.25) & Featuring, teamed, eaturing, featuring, reunited
   & F8644 ($-$0.20) & nobody, nothing, nothing, Nothing, THING \\
 & & F2014 ($+$0.24) & dynam, affordable, unbeat, appreciated, versatile
   & F4439 ($-$0.20) & disgusting, vile, despicable, crap, disgrace \\
\cmidrule(lr){2-6}
 & \multirow{3}{*}{Formality}
   & F12001 ($+$0.63) & [CJK], Participant, ortium, ilaterally, ortment
   & F9490 ($-$0.67) & pie, hell, —, damn, … \\
 & & F4051 ($+$0.58) & Treaty, NATO, treaty, milit, Yugoslavia
   & F11298 ($-$0.61) & ]], ín, tar, eq, › \\
 & & F9697 ($+$0.56) & liberal, government, neoliberal, privatization, atism
   & F7306 ($-$0.55) & §§, Crossref, fmt, ÍÍ, nos \\
\midrule
\multirow{9}{*}{\rotatebox[origin=c]{90}{LLaDA (L26)}}
 & \multirow{3}{*}{Topic}
   & F9143 ($+$1.13) & Para, \zh{增量} (increment), \zh{晗}, \_event, events
   & F5553 ($-$0.80) & Oro, leness, experien, -war, Lah \\
 & & F1834 ($+$1.05) & Historic, atch, \zh{奥斯} (Aus-), stacked
   & F6163 ($-$0.68) & INCLUDING, IllegalStateException, Subscriptions, -analytics \\
 & & F2897 ($+$0.63) & won, awarded, year, \zh{提名} (nomination), win
   & F1518 ($-$0.66) & -producing, \zh{资产负债} (balance sheet), .release, produce \\
\cmidrule(lr){2-6}
 & \multirow{3}{*}{Sentiment}
   & F12180 ($+$0.43) & (always, commended, throughout, PERF, -Sah
   & F5461 ($-$0.44) & Avoid, Avoid, avoided, avoid, AV \\
 & & F2833 ($+$0.38) & recommended, recommend, deserves, recommending, recomm
   & F10262 ($-$0.42) & throughout, whatsoever, anders, OEM, reative \\
 & & F7427 ($+$0.29) & aras, bec, \zh{小小} (tiny), apical, \zh{笃}
   & F12851 ($-$0.42) & \zh{失败} (failure), \zh{失败的} (failed (adj.)), amateur \\
\cmidrule(lr){2-6}
 & \multirow{3}{*}{Formality}
   & F8446 ($+$0.47) & monies, thru, angst, lite, youngster
   & F7446 ($-$1.07) & however, alot, barley, ,D, ect \\
 & & F13972 ($+$0.44) & idency, "Don, cretion, ener, narc
   & F12552 ($-$0.66) & sublicense, fucking, empire, summer, fucked \\
 & & F12180 ($+$0.41) & (always, commended, throughout, PERF, -Sah
   & F9058 ($-$0.61) & \}else, ",\&, ,const, ,\zh{免费}, .\zh{第} \\
\midrule
\multirow{9}{*}{\rotatebox[origin=c]{90}{DREAM (L23)}}
 & \multirow{3}{*}{Topic}
   & F6961 ($+$1.67) & zed, games, -sided, ided, 1
   & F14304 ($-$1.12) & company, angled, business, undergrad, (); \\
 & & F4856 ($+$1.28) & team, pall, Players, players, game
   & F13053 ($-$0.89) & computer, software, computers, Computer, data \\
 & & F7524 ($+$0.91) & team, Team, training, fre, trials
   & F6220 ($-$0.83) & union, union, unions, wright, SHA \\
\cmidrule(lr){2-6}
 & \multirow{3}{*}{Sentiment}
   & F10880 ($+$0.51) & simply, potency, consum, bullish, zen
   & F8145 ($-$0.51) & worse, S, literally, oa, Liter \\
 & & F5417 ($+$0.36) & and, all, each, is, even
   & F13133 ($-$0.34) & 0, 1, prohibits, V, ement \\
 & & F3924 ($+$0.32) & in, also, that, really, is
   & F6642 ($-$0.33) & worse, worst, -bottom, Worst, buffalo \\
\cmidrule(lr){2-6}
 & \multirow{3}{*}{Formality}
   & F4302 ($+$0.59) & , ; … ... \textasciitilde
   & F10886 ($-$1.14) & people, like, all, and, out \\
 & & F4221 ($+$0.58) & Parliamentary, anyhow, programmes, forb, practise
   & F11172 ($-$0.77) & abyte, oppos, imagin, Homeland, igy \\
 & & F5033 ($+$0.47) & cer, bloodstream, vintage, DNA, sap
   & F1918 ($-$0.70) & 1, 2, 0, 4, ' \\
\bottomrule
\end{tabular}%
}
\end{table*}

%% file: figures_tables/table_steering_results_v2_detail_single.tex
\begin{table}[t]
\centering
\caption{Single-attribute steering results with standard deviations. Format: mean$\pm$std. Superscripts on Conf denote the selected $\alpha$. Same quality gate and selection criteria as Table ~\ref{tab:operator}}
\label{tab:steering_single_detail}
\tiny
\setlength{\tabcolsep}{1pt}
\begin{adjustbox}{max width=\textwidth}
\begin{tabular}{l|ccc|ccc|ccc}
\toprule
 & \multicolumn{3}{c}{Sentiment (S)} & \multicolumn{3}{c}{Topic (T)} & \multicolumn{3}{c}{Style (St)} \\
\cmidrule(lr){2-4}\cmidrule(lr){5-7}\cmidrule(lr){8-10}
 & Conf$\uparrow$ & PPL$\downarrow$ & Flu$\uparrow$ & Conf$\uparrow$ & PPL$\downarrow$ & Flu$\uparrow$ & Conf$\uparrow$ & PPL$\downarrow$ & Flu$\uparrow$ \\
\midrule
\multicolumn{10}{l}{\textbf{MDLM-124M}} \\
\quad \textit{No Steering} & $27.5{\pm}39.2$ & $46.0{\pm}10.6$ & $2.32{\pm}0.48$ & $20.3{\pm}39.9$ & $46.0{\pm}10.6$ & $2.32{\pm}0.48$ & $54.8{\pm}17.6$ & $46.0{\pm}10.6$ & $2.32{\pm}0.48$ \\
\quad Uniform & $99.3{\pm}6.6^{11}$ & $94.0{\pm}24.3$ & $2.07{\pm}0.27$ & $99.9{\pm}0.0^{4}$ & $97.2{\pm}16.3$ & $2.00{\pm}0.00$ & $69.4{\pm}8.8^{4}$ & $67.8{\pm}10.9$ & $2.12{\pm}0.32$ \\
\quad Step Function & $99.5{\pm}2.0^{14}$ & $100.0{\pm}20.1$ & $2.06{\pm}0.23$ & $100.0{\pm}0.0^{20}$ & $99.3{\pm}18.1$ & $2.00{\pm}0.00$ & $69.3{\pm}8.8^{4}$ & $66.2{\pm}10.5$ & $2.16{\pm}0.37$ \\
\quad Adaptive & $94.8{\pm}18.9^{20}$ & $64.6{\pm}13.3$ & $2.19{\pm}0.39$ & $99.9{\pm}0.0^{9}$ & $54.5{\pm}15.3$ & $2.03{\pm}0.17$ & $72.7{\pm}8.5^{17}$ & $60.0{\pm}8.7$ & $2.24{\pm}0.43$ \\
\cmidrule(l){1-10}
\quad Contrastive Vec. & $97.4{\pm}11.8^{20}$ & $348.1{\pm}113.8$ & $1.52{\pm}0.50$ & $68.9{\pm}45.7^{11}$ & $296.4{\pm}88.8$ & $1.07{\pm}0.26$ & $43.3{\pm}16.2^{4}$ & $283.3{\pm}85.7$ & $1.44{\pm}0.50$ \\
\quad Probe & $74.4{\pm}36.2^{17}$ & $474.2{\pm}131.6$ & $1.21{\pm}0.41$ & $99.9{\pm}0.1^{19}$ & $166.0{\pm}30.7$ & $1.27{\pm}0.45$ & $66.4{\pm}11.3^{13}$ & $747.5{\pm}212.8$ & $1.16{\pm}0.37$ \\
\quad PCA & $20.7{\pm}36.3^{14}$ & $345.0{\pm}114.7$ & $1.45{\pm}0.50$ & $15.7{\pm}36.1^{4}$ & $268.0{\pm}75.0$ & $1.34{\pm}0.47$ & $42.8{\pm}16.6^{5}$ & $287.1{\pm}82.1$ & $1.37{\pm}0.48$ \\
\midrule
\multicolumn{10}{l}{\textbf{SEDD-124M}} \\
\quad \textit{No Steering} & $23.4{\pm}37.6$ & $48.1{\pm}22.2$ & $2.23{\pm}0.42$ & $19.1{\pm}39.2$ & $48.1{\pm}22.2$ & $2.23{\pm}0.42$ & $57.6{\pm}17.1$ & $48.1{\pm}22.2$ & $2.23{\pm}0.42$ \\
\quad Uniform & $97.2{\pm}13.7^{11}$ & $94.1{\pm}31.4$ & $2.02{\pm}0.14$ & $100.0{\pm}0.0^{4}$ & $65.8{\pm}19.2$ & $1.98{\pm}0.16$ & $96.0{\pm}3.9^{19}$ & $91.4{\pm}14.8$ & $1.99{\pm}0.12$ \\
\quad Step Function & $93.4{\pm}18.4^{18}$ & $71.5{\pm}25.2$ & $2.23{\pm}0.42$ & $100.0{\pm}0.0^{16}$ & $79.7{\pm}15.2$ & $2.01{\pm}0.10$ & $96.2{\pm}3.6^{20}$ & $51.8{\pm}7.8$ & $2.06{\pm}0.23$ \\
\quad Adaptive & $76.7{\pm}37.6^{19}$ & $59.3{\pm}24.1$ & $2.27{\pm}0.45$ & $100.0{\pm}0.0^{9}$ & $44.7{\pm}12.8$ & $2.02{\pm}0.12$ & $90.5{\pm}7.2^{19}$ & $42.6{\pm}7.0$ & $2.22{\pm}0.41$ \\
\cmidrule(l){1-10}
\quad Contrastive Vec. & $99.8{\pm}0.3^{19}$ & $83.4{\pm}22.5$ & $2.02{\pm}0.16$ & $99.9{\pm}0.0^{19}$ & $59.4{\pm}15.2$ & $2.08{\pm}0.31$ & $91.7{\pm}8.5^{20}$ & $77.2{\pm}12.7$ & $2.27{\pm}0.45$ \\
\quad Probe & $95.9{\pm}16.0^{9}$ & $166.5{\pm}49.9$ & $1.88{\pm}0.32$ & $87.8{\pm}31.9^{7}$ & $78.6{\pm}27.9$ & $1.93{\pm}0.32$ & $93.8{\pm}6.9^{15}$ & $163.9{\pm}38.5$ & $1.76{\pm}0.43$ \\
\quad PCA & $33.3{\pm}43.1^{5}$ & $51.9{\pm}12.8$ & $2.13{\pm}0.36$ & $32.9{\pm}46.5^{20}$ & $41.8{\pm}10.8$ & $2.29{\pm}0.48$ & $56.7{\pm}18.4^{1}$ & $42.9{\pm}11.5$ & $2.33{\pm}0.48$ \\
\midrule
\multicolumn{10}{l}{\textbf{LLaDA-8B}} \\
\quad \textit{No Steering} & $69.1{\pm}44.9$ & $7.1{\pm}5.5$ & $3.56{\pm}0.86$ & $13.9{\pm}31.4$ & $7.1{\pm}5.5$ & $3.56{\pm}0.86$ & $22.2{\pm}23.0$ & $7.1{\pm}5.5$ & $3.56{\pm}0.86$ \\
\quad Uniform & $97.7{\pm}13.7^{20}$ & $6.4{\pm}3.4$ & $3.39{\pm}0.66$ & $99.9{\pm}0.1^{15}$ & $5.8{\pm}2.1$ & $2.52{\pm}0.50$ & $87.4{\pm}6.1^{19}$ & $93.3{\pm}1256.0$ & $1.78{\pm}0.47$ \\
\quad Step Function & $89.5{\pm}29.2^{19}$ & $6.5{\pm}5.2$ & $3.44{\pm}0.73$ & $99.9{\pm}0.4^{17}$ & $4.7{\pm}1.6$ & $2.54{\pm}0.52$ & $85.4{\pm}8.7^{19}$ & $4.0{\pm}5.8$ & $1.94{\pm}0.47$ \\
\quad Adaptive & $85.4{\pm}34.2^{20}$ & $6.3{\pm}3.3$ & $3.52{\pm}0.81$ & $99.3{\pm}7.1^{18}$ & $6.0{\pm}3.8$ & $2.94{\pm}0.55$ & $69.8{\pm}25.3^{20}$ & $4.3{\pm}2.6$ & $2.46{\pm}0.70$ \\
\cmidrule(l){1-10}
\quad Contrastive Vec. & $91.0{\pm}27.2^{19}$ & $6.1{\pm}3.2$ & $3.56{\pm}0.80$ & $28.6{\pm}41.9^{19}$ & $6.9{\pm}6.2$ & $3.42{\pm}0.76$ & $38.8{\pm}27.1^{19}$ & $7.1{\pm}6.9$ & $3.25{\pm}0.81$ \\
\quad Probe & $81.9{\pm}36.7^{19}$ & $7.5{\pm}7.7$ & $3.62{\pm}0.84$ & $24.6{\pm}39.4^{20}$ & $7.2{\pm}6.4$ & $3.45{\pm}0.85$ & $28.3{\pm}25.7^{19}$ & $8.5{\pm}23.4$ & $3.59{\pm}0.74$ \\
\quad PCA & $73.2{\pm}42.7^{8}$ & $7.5{\pm}8.1$ & $3.67{\pm}0.93$ & $15.3{\pm}31.9^{18}$ & $6.6{\pm}4.3$ & $3.62{\pm}0.73$ & $36.5{\pm}26.8^{19}$ & $10.0{\pm}24.3$ & $3.34{\pm}0.88$ \\
\quad Prompt & $80.9{\pm}36.5$ & $7.2{\pm}4.5$ & $3.33{\pm}0.80$ & $64.9{\pm}45.6$ & $6.0{\pm}3.9$ & $3.40{\pm}0.78$ & $45.1{\pm}24.9$ & $11.5{\pm}33.2$ & $2.71{\pm}0.82$ \\
\midrule
\multicolumn{10}{l}{\textbf{DREAM-7B}} \\
\quad \textit{No Steering} & $52.2{\pm}48.0$ & $8.6{\pm}3.4$ & $3.75{\pm}0.68$ & $13.3{\pm}29.1$ & $8.6{\pm}3.4$ & $3.75{\pm}0.68$ & $26.9{\pm}22.7$ & $8.6{\pm}3.4$ & $3.75{\pm}0.68$ \\
\quad Uniform & $100.0{\pm}0.0^{19}$ & $12.9{\pm}4.4$ & $3.44{\pm}0.65$ & $100.0{\pm}0.0^{11}$ & $49.2{\pm}16.6$ & $1.01{\pm}0.07$ & $98.5{\pm}1.3^{19}$ & $45.3{\pm}18.5$ & $1.70{\pm}0.46$ \\
\quad Step Function & $100.0{\pm}0.0^{18}$ & $11.0{\pm}3.7$ & $3.59{\pm}0.66$ & $100.0{\pm}0.0^{11}$ & $43.1{\pm}15.1$ & $1.00{\pm}0.00$ & $97.7{\pm}2.2^{18}$ & $25.9{\pm}11.3$ & $1.66{\pm}0.47$ \\
\quad Adaptive & $99.5{\pm}6.8^{16}$ & $10.0{\pm}3.7$ & $3.67{\pm}0.67$ & $100.0{\pm}0.0^{12}$ & $27.0{\pm}30.4$ & $1.19{\pm}0.39$ & $97.0{\pm}3.0^{17}$ & $25.4{\pm}12.8$ & $1.76{\pm}0.42$ \\
\cmidrule(l){1-10}
\quad Contrastive Vec. & $98.0{\pm}12.3^{17}$ & $10.3{\pm}3.9$ & $3.63{\pm}0.66$ & $95.7{\pm}18.2^{19}$ & $15.2{\pm}5.8$ & $2.95{\pm}0.61$ & $92.8{\pm}7.1^{9}$ & $11.0{\pm}8.5$ & $2.31{\pm}0.77$ \\
\quad Probe & $90.8{\pm}26.7^{19}$ & $10.0{\pm}3.6$ & $3.37{\pm}0.66$ & $80.7{\pm}36.1^{18}$ & $10.8{\pm}4.7$ & $3.17{\pm}0.67$ & $98.5{\pm}3.1^{14}$ & $18.9{\pm}16.4$ & $1.84{\pm}0.72$ \\
\quad PCA & $65.6{\pm}45.1^{16}$ & $7.2{\pm}2.9$ & $3.79{\pm}0.64$ & $38.1{\pm}42.4^{20}$ & $64.6{\pm}37.6$ & $1.08{\pm}0.29$ & $31.3{\pm}23.5^{20}$ & $9.0{\pm}3.2$ & $3.66{\pm}0.67$ \\
\quad Prompt & $80.3{\pm}36.6$ & $7.1{\pm}2.8$ & $3.50{\pm}0.66$ & $61.5{\pm}47.1$ & $8.2{\pm}2.6$ & $3.73{\pm}0.59$ & $20.6{\pm}21.5$ & $7.9{\pm}3.2$ & $3.48{\pm}0.75$ \\
\bottomrule
\end{tabular}
\end{adjustbox}
\end{table}

%% file: figures_tables/table_steering_results_v2_detail_multi.tex
\begin{table}[t]
\centering
\caption{Multi-attribute steering results with per-attribute confidence. S: sentiment confidence; T: topic confidence; St: style confidence (formal for MDLM/SEDD, informal for LLaDA/DREAM). Superscripts on the first attribute denote selected $\alpha$. Same quality gate and selection criteria as Table~\ref{tab:operator}.}
\label{tab:steering_multi_detail}
\tiny
\setlength{\tabcolsep}{1pt}
\begin{adjustbox}{max width=\textwidth}
\begin{tabular}{l|cccc|cccc|cccc|ccccc}
\toprule
 & \multicolumn{4}{c}{S+T} & \multicolumn{4}{c}{S+St} & \multicolumn{4}{c}{St+T} & \multicolumn{5}{c}{S+T+St} \\
\cmidrule(lr){2-5}\cmidrule(lr){6-9}\cmidrule(lr){10-13}\cmidrule(lr){14-18}
 & S$\uparrow$ & T$\uparrow$ & PPL$\downarrow$ & Flu$\uparrow$ & S$\uparrow$ & St$\uparrow$ & PPL$\downarrow$ & Flu$\uparrow$ & St$\uparrow$ & T$\uparrow$ & PPL$\downarrow$ & Flu$\uparrow$ & S$\uparrow$ & T$\uparrow$ & St$\uparrow$ & PPL$\downarrow$ & Flu$\uparrow$ \\
\midrule
\multicolumn{18}{l}{\textbf{MDLM-124M}} \\
\quad \textit{No Steering} & $27.5{\pm}39.2$ & $20.3{\pm}39.9$ & $46.0{\pm}10.6$ & $2.32{\pm}0.48$ & $27.5{\pm}39.2$ & $54.8{\pm}17.6$ & $46.0{\pm}10.6$ & $2.32{\pm}0.48$ & $54.8{\pm}17.6$ & $20.3{\pm}39.9$ & $46.0{\pm}10.6$ & $2.32{\pm}0.48$ & $27.5{\pm}39.2$ & $20.3{\pm}39.9$ & $54.8{\pm}17.6$ & $46.0{\pm}10.6$ & $2.32{\pm}0.48$ \\
\quad Uniform & $97.1{\pm}10.7^{3}$ & $99.9{\pm}0.0$ & $89.3{\pm}16.6$ & $1.99{\pm}0.10$ & $32.3{\pm}39.1^{4}$ & $68.1{\pm}8.5$ & $84.9{\pm}13.0$ & $2.02{\pm}0.16$ & $49.5{\pm}5.3^{3}$ & $99.9{\pm}0.0$ & $91.8{\pm}14.7$ & $2.00{\pm}0.07$ & $89.3{\pm}24.4^{2}$ & $99.9{\pm}0.0$ & $49.6{\pm}5.7$ & $75.9{\pm}14.4$ & $1.99{\pm}0.10$ \\
\quad Step Function & $98.3{\pm}6.1^{8}$ & $100.0{\pm}0.0$ & $87.9{\pm}15.0$ & $2.00{\pm}0.00$ & $24.2{\pm}36.5^{4}$ & $67.9{\pm}8.7$ & $79.1{\pm}12.3$ & $2.06{\pm}0.24$ & $55.8{\pm}6.2^{5}$ & $99.9{\pm}0.6$ & $93.2{\pm}15.4$ & $1.98{\pm}0.14$ & $73.0{\pm}37.7^{3}$ & $99.5{\pm}7.0$ & $53.2{\pm}7.2$ & $70.3{\pm}12.6$ & $2.00{\pm}0.07$ \\
\quad Adaptive & $96.1{\pm}14.0^{18}$ & $100.0{\pm}0.0$ & $68.8{\pm}15.2$ & $2.01{\pm}0.10$ & $37.1{\pm}42.6^{20}$ & $74.0{\pm}10.1$ & $81.2{\pm}12.1$ & $2.04{\pm}0.21$ & $56.2{\pm}6.8^{19}$ & $99.9{\pm}0.0$ & $67.3{\pm}11.9$ & $2.00{\pm}0.10$ & $85.1{\pm}29.7^{14}$ & $99.9{\pm}0.0$ & $56.5{\pm}6.6$ & $65.5{\pm}11.8$ & $2.00{\pm}0.10$ \\
\quad Uniform+E & $99.5{\pm}1.9^{9}$ & $100.0{\pm}0.0$ & $94.9{\pm}18.9$ & $2.01{\pm}0.10$ & $78.4{\pm}35.6^{8}$ & $65.4{\pm}11.4$ & $97.8{\pm}15.9$ & $2.01{\pm}0.10$ & $51.7{\pm}6.9^{3}$ & $99.7{\pm}3.1$ & $73.5{\pm}13.5$ & $2.00{\pm}0.07$ & $93.1{\pm}19.5^{6}$ & $98.4{\pm}12.2$ & $51.7{\pm}7.2$ & $85.8{\pm}16.3$ & $2.00{\pm}0.00$ \\
\quad Adaptive+E & $94.2{\pm}18.9^{20}$ & $96.5{\pm}18.4$ & $58.8{\pm}11.8$ & $2.06{\pm}0.25$ & $74.6{\pm}37.8^{18}$ & $62.7{\pm}12.8$ & $62.8{\pm}11.7$ & $2.15{\pm}0.36$ & $58.0{\pm}6.8^{20}$ & $99.9{\pm}0.0$ & $67.3{\pm}12.5$ & $2.00{\pm}0.00$ & $91.2{\pm}23.3^{20}$ & $97.9{\pm}12.9$ & $54.6{\pm}9.0$ & $61.2{\pm}11.5$ & $2.03{\pm}0.17$ \\
\cmidrule(l){1-18}
\quad Contrastive Vec. & $82.6{\pm}31.0^{20}$ & $76.5{\pm}41.6$ & $336.6{\pm}88.3$ & $1.08{\pm}0.28$ & $97.0{\pm}13.6^{20}$ & $72.9{\pm}16.4$ & $427.1{\pm}134.5$ & $1.72{\pm}0.45$ & $28.4{\pm}8.5^{17}$ & $99.9{\pm}0.1$ & $289.2{\pm}68.3$ & $1.47{\pm}0.50$ & $96.5{\pm}12.9^{20}$ & $90.0{\pm}29.7$ & $55.7{\pm}14.2$ & $360.9{\pm}99.7$ & $1.33{\pm}0.47$ \\
\quad Probe & $75.3{\pm}36.1^{16}$ & $99.9{\pm}0.1$ & $251.0{\pm}47.5$ & $1.30{\pm}0.46$ & $70.0{\pm}38.3^{19}$ & $73.0{\pm}13.1$ & $600.3{\pm}211.0$ & $1.26{\pm}0.44$ & $53.0{\pm}11.5^{18}$ & $99.9{\pm}0.0$ & $593.6{\pm}165.9$ & $1.07{\pm}0.26$ & $63.2{\pm}40.6^{19}$ & $100.0{\pm}0.0$ & $52.3{\pm}8.9$ & $320.6{\pm}147.5$ & $1.14{\pm}0.34$ \\
\quad PCA & $20.1{\pm}34.2^{12}$ & $11.2{\pm}30.8$ & $257.8{\pm}78.0$ & $1.29{\pm}0.46$ & $19.9{\pm}34.3^{7}$ & $40.7{\pm}14.7$ & $300.5{\pm}90.9$ & $1.44{\pm}0.50$ & $44.3{\pm}16.4^{1}$ & $12.8{\pm}33.2$ & $270.6{\pm}78.2$ & $1.41{\pm}0.49$ & $18.6{\pm}34.7^{6}$ & $15.4{\pm}35.4$ & $39.6{\pm}14.5$ & $295.3{\pm}93.0$ & $1.43{\pm}0.49$ \\
\midrule
\multicolumn{18}{l}{\textbf{SEDD-124M}} \\
\quad \textit{No Steering} & $23.4{\pm}37.6$ & $19.1{\pm}39.2$ & $48.1{\pm}22.2$ & $2.23{\pm}0.42$ & $23.4{\pm}37.6$ & $57.6{\pm}17.1$ & $48.1{\pm}22.2$ & $2.23{\pm}0.42$ & $57.6{\pm}17.1$ & $19.1{\pm}39.2$ & $48.1{\pm}22.2$ & $2.23{\pm}0.42$ & $23.4{\pm}37.6$ & $19.1{\pm}39.2$ & $57.6{\pm}17.1$ & $48.1{\pm}22.2$ & $2.23{\pm}0.42$ \\
\quad Uniform & $93.0{\pm}19.3^{4}$ & $100.0{\pm}0.0$ & $82.4{\pm}26.9$ & $2.00{\pm}0.00$ & $79.6{\pm}31.1^{9}$ & $96.5{\pm}3.9$ & $98.9{\pm}17.3$ & $2.00{\pm}0.00$ & $71.0{\pm}8.6^{5}$ & $100.0{\pm}0.0$ & $89.4{\pm}20.7$ & $1.99{\pm}0.12$ & $93.2{\pm}19.5^{4}$ & $100.0{\pm}0.0$ & $73.1{\pm}8.2$ & $94.4{\pm}24.5$ & $1.97{\pm}0.17$ \\
\quad Step Function & $85.0{\pm}28.4^{15}$ & $100.0{\pm}0.0$ & $99.3{\pm}19.3$ & $2.00{\pm}0.12$ & $65.7{\pm}37.4^{20}$ & $98.5{\pm}2.1$ & $63.3{\pm}10.2$ & $2.01{\pm}0.10$ & $91.0{\pm}5.8^{13}$ & $99.3{\pm}7.2$ & $98.8{\pm}19.6$ & $2.00{\pm}0.07$ & $69.2{\pm}37.9^{10}$ & $98.5{\pm}11.8$ & $88.9{\pm}6.5$ & $91.6{\pm}21.0$ & $1.97{\pm}0.17$ \\
\quad Adaptive & $90.6{\pm}23.3^{20}$ & $100.0{\pm}0.0$ & $51.2{\pm}12.1$ & $2.01{\pm}0.10$ & $60.5{\pm}42.1^{19}$ & $95.2{\pm}5.1$ & $45.2{\pm}7.1$ & $2.41{\pm}0.49$ & $79.9{\pm}7.7^{20}$ & $100.0{\pm}0.0$ & $34.7{\pm}7.5$ & $2.00{\pm}0.10$ & $91.0{\pm}22.0^{19}$ & $100.0{\pm}0.0$ & $83.0{\pm}7.6$ & $38.8{\pm}8.6$ & $2.01{\pm}0.10$ \\
\quad Uniform+E & $95.7{\pm}13.8^{8}$ & $100.0{\pm}0.0$ & $99.8{\pm}45.7$ & $2.00{\pm}0.10$ & $94.8{\pm}17.9^{11}$ & $93.0{\pm}8.6$ & $98.1{\pm}39.7$ & $2.04{\pm}0.18$ & $66.3{\pm}7.3^{6}$ & $100.0{\pm}0.0$ & $98.6{\pm}21.1$ & $1.98{\pm}0.14$ & $96.4{\pm}13.6^{8}$ & $100.0{\pm}0.0$ & $68.3{\pm}9.0$ & $98.5{\pm}25.5$ & $1.99{\pm}0.10$ \\
\quad Adaptive+E & $78.6{\pm}35.4^{19}$ & $62.6{\pm}48.2$ & $59.6{\pm}25.3$ & $2.17{\pm}0.37$ & $73.9{\pm}38.4^{19}$ & $85.2{\pm}14.1$ & $60.5{\pm}38.6$ & $2.27{\pm}0.45$ & $79.9{\pm}7.7^{20}$ & $100.0{\pm}0.0$ & $34.7{\pm}7.5$ & $2.00{\pm}0.12$ & $76.9{\pm}37.1^{20}$ & $41.8{\pm}49.1$ & $77.3{\pm}15.4$ & $57.5{\pm}22.4$ & $2.14{\pm}0.35$ \\
\cmidrule(l){1-18}
\quad Contrastive Vec. & $99.0{\pm}4.0^{20}$ & $100.0{\pm}0.0$ & $78.1{\pm}19.7$ & $2.00{\pm}0.07$ & $99.7{\pm}0.6^{20}$ & $96.4{\pm}5.1$ & $74.7{\pm}13.5$ & $2.08{\pm}0.28$ & $68.2{\pm}13.3^{20}$ & $99.5{\pm}6.8$ & $79.8{\pm}20.5$ & $2.00{\pm}0.07$ & $97.2{\pm}12.8^{20}$ & $85.5{\pm}34.9$ & $79.8{\pm}11.9$ & $77.4{\pm}17.6$ & $2.02{\pm}0.14$ \\
\quad Probe & $91.6{\pm}25.9^{7}$ & $78.6{\pm}39.8$ & $89.1{\pm}27.0$ & $1.89{\pm}0.31$ & $94.7{\pm}17.0^{10}$ & $91.5{\pm}11.7$ & $177.4{\pm}55.6$ & $1.94{\pm}0.25$ & $57.2{\pm}18.9^{10}$ & $93.4{\pm}24.5$ & $122.6{\pm}37.2$ & $1.84{\pm}0.38$ & $93.5{\pm}20.7^{11}$ & $83.1{\pm}36.8$ & $77.9{\pm}18.6$ & $131.3{\pm}44.7$ & $1.64{\pm}0.48$ \\
\quad PCA & $28.3{\pm}40.5^{5}$ & $21.9{\pm}40.7$ & $49.3{\pm}12.7$ & $2.19{\pm}0.40$ & $59.7{\pm}43.2^{19}$ & $80.4{\pm}14.5$ & $70.9{\pm}18.5$ & $2.21{\pm}0.40$ & $54.5{\pm}18.5^{6}$ & $27.2{\pm}44.0$ & $40.0{\pm}10.2$ & $2.33{\pm}0.49$ & $41.4{\pm}44.1^{17}$ & $33.3{\pm}46.6$ & $61.7{\pm}16.3$ & $44.2{\pm}11.3$ & $2.29{\pm}0.47$ \\
\midrule
\multicolumn{18}{l}{\textbf{LLaDA-8B}} \\
\quad \textit{No Steering} & $69.1{\pm}44.9$ & $13.9{\pm}31.4$ & $7.1{\pm}5.5$ & $3.56{\pm}0.86$ & $69.1{\pm}44.9$ & $22.2{\pm}23.0$ & $7.1{\pm}5.5$ & $3.56{\pm}0.86$ & $13.9{\pm}31.4$ & $22.2{\pm}23.0$ & $7.1{\pm}5.5$ & $3.56{\pm}0.86$ & $69.1{\pm}44.9$ & $13.9{\pm}31.4$ & $22.2{\pm}23.0$ & $7.1{\pm}5.5$ & $3.56{\pm}0.86$ \\
\quad Uniform & $94.7{\pm}20.8^{12}$ & $99.9{\pm}0.3$ & $5.7{\pm}2.5$ & $2.78{\pm}0.41$ & $74.0{\pm}42.5^{17}$ & $78.1{\pm}15.1$ & $4.2{\pm}1.5$ & $2.29{\pm}0.54$ & $99.0{\pm}9.9^{20}$ & $65.8{\pm}17.4$ & $5.6{\pm}2.3$ & $1.85{\pm}0.39$ & $78.8{\pm}35.8^{20}$ & $98.4{\pm}11.9$ & $65.2{\pm}17.3$ & $5.8{\pm}2.2$ & $1.88{\pm}0.36$ \\
\quad Step Function & $87.2{\pm}31.3^{11}$ & $98.4{\pm}10.7$ & $5.9{\pm}5.6$ & $2.85{\pm}0.40$ & $72.2{\pm}43.8^{17}$ & $77.4{\pm}16.4$ & $4.1{\pm}1.5$ & $2.26{\pm}0.51$ & $99.5{\pm}4.4^{19}$ & $60.6{\pm}19.5$ & $3.6{\pm}1.0$ & $2.06{\pm}0.34$ & $73.5{\pm}38.4^{20}$ & $98.6{\pm}11.2$ & $59.3{\pm}20.8$ & $3.8{\pm}1.1$ & $2.03{\pm}0.31$ \\
\quad Adaptive & $87.0{\pm}30.9^{15}$ & $97.0{\pm}15.9$ & $6.8{\pm}4.3$ & $3.10{\pm}0.61$ & $72.0{\pm}43.5^{20}$ & $58.7{\pm}29.9$ & $17.8{\pm}131.9$ & $2.71{\pm}0.64$ & $94.4{\pm}20.6^{20}$ & $55.1{\pm}24.6$ & $4.8{\pm}3.6$ & $2.48{\pm}0.56$ & $77.1{\pm}39.7^{18}$ & $93.8{\pm}22.6$ & $50.6{\pm}28.0$ & $5.5{\pm}5.0$ & $2.69{\pm}0.54$ \\
\quad Uniform+E & $93.5{\pm}22.8^{11}$ & $99.4{\pm}7.1$ & $6.0{\pm}2.6$ & $2.81{\pm}0.43$ & $48.4{\pm}47.8^{17}$ & $82.5{\pm}13.3$ & $100.3{\pm}1259.0$ & $2.06{\pm}0.40$ & $80.5{\pm}34.2^{19}$ & $71.5{\pm}18.5$ & $4.6{\pm}1.8$ & $2.00{\pm}0.38$ & $70.1{\pm}43.6^{15}$ & $82.9{\pm}35.5$ & $61.9{\pm}25.6$ & $4.3{\pm}1.6$ & $2.31{\pm}0.51$ \\
\quad Adaptive+E & $88.9{\pm}30.1^{20}$ & $83.8{\pm}35.5$ & $6.8{\pm}7.1$ & $3.40{\pm}0.71$ & $67.6{\pm}45.2^{20}$ & $59.7{\pm}30.1$ & $18.0{\pm}131.9$ & $2.62{\pm}0.62$ & $53.0{\pm}47.6^{20}$ & $54.1{\pm}29.7$ & $9.1{\pm}34.7$ & $2.71{\pm}0.61$ & $76.8{\pm}40.2^{18}$ & $56.7{\pm}47.4$ & $51.4{\pm}26.4$ & $6.3{\pm}9.3$ & $2.75{\pm}0.56$ \\
\cmidrule(l){1-18}
\quad Contrastive Vec. & $88.1{\pm}30.5^{17}$ & $27.2{\pm}41.1$ & $6.3{\pm}5.0$ & $3.62{\pm}0.77$ & $83.9{\pm}34.0^{17}$ & $33.2{\pm}26.5$ & $8.1{\pm}19.4$ & $3.40{\pm}0.90$ & $22.5{\pm}36.7^{20}$ & $35.9{\pm}26.6$ & $13.3{\pm}59.8$ & $3.43{\pm}0.78$ & $87.7{\pm}31.4^{20}$ & $22.6{\pm}37.9$ & $30.7{\pm}26.5$ & $7.5{\pm}11.8$ & $3.48{\pm}0.75$ \\
\quad Probe & $78.5{\pm}39.6^{20}$ & $22.9{\pm}37.7$ & $6.7{\pm}4.8$ & $3.58{\pm}0.82$ & $82.6{\pm}35.6^{14}$ & $31.5{\pm}24.8$ & $6.7{\pm}3.3$ & $3.71{\pm}0.76$ & $20.3{\pm}36.2^{20}$ & $30.8{\pm}27.1$ & $7.1{\pm}4.9$ & $3.58{\pm}0.83$ & $81.4{\pm}37.2^{19}$ & $19.5{\pm}36.8$ & $26.8{\pm}25.1$ & $7.9{\pm}16.5$ & $3.52{\pm}0.81$ \\
\quad PCA & $71.4{\pm}42.8^{19}$ & $15.4{\pm}30.0$ & $6.4{\pm}4.0$ & $3.71{\pm}0.80$ & $68.7{\pm}44.5^{20}$ & $33.5{\pm}23.8$ & $6.8{\pm}5.4$ & $3.53{\pm}0.87$ & $14.9{\pm}29.7^{20}$ & $33.7{\pm}26.3$ & $6.5{\pm}3.9$ & $3.46{\pm}0.76$ & $69.8{\pm}43.7^{20}$ & $15.6{\pm}29.9$ & $33.2{\pm}25.9$ & $6.7{\pm}5.8$ & $3.59{\pm}0.80$ \\
\quad Prompt & $88.7{\pm}29.2$ & $50.5{\pm}48.6$ & $8.3{\pm}6.9$ & $3.23{\pm}0.75$ & $77.6{\pm}40.3$ & $19.3{\pm}23.9$ & $9.5{\pm}8.3$ & $2.83{\pm}0.85$ & $74.5{\pm}42.0$ & $37.3{\pm}24.0$ & $6.7{\pm}4.0$ & $3.43{\pm}0.87$ & $92.4{\pm}25.2$ & $48.2{\pm}49.3$ & $21.2{\pm}22.8$ & $8.4{\pm}6.2$ & $3.25{\pm}0.96$ \\
\midrule
\multicolumn{18}{l}{\textbf{DREAM-7B}} \\
\quad \textit{No Steering} & $52.2{\pm}48.0$ & $13.3{\pm}29.1$ & $8.6{\pm}3.4$ & $3.75{\pm}0.68$ & $52.2{\pm}48.0$ & $26.9{\pm}22.7$ & $8.6{\pm}3.4$ & $3.75{\pm}0.68$ & $13.3{\pm}29.1$ & $26.9{\pm}22.7$ & $8.6{\pm}3.4$ & $3.75{\pm}0.68$ & $52.2{\pm}48.0$ & $13.3{\pm}29.1$ & $26.9{\pm}22.7$ & $8.6{\pm}3.4$ & $3.75{\pm}0.68$ \\
\quad Uniform & $99.8{\pm}1.1^{5}$ & $99.9{\pm}0.0$ & $21.8{\pm}8.9$ & $2.50{\pm}0.50$ & $100.0{\pm}0.0^{20}$ & $96.5{\pm}2.9$ & $63.3{\pm}20.6$ & $1.55{\pm}0.50$ & $100.0{\pm}0.0^{7}$ & $84.1{\pm}4.6$ & $96.6{\pm}28.1$ & $1.12{\pm}0.32$ & $95.8{\pm}17.4^{6}$ & $100.0{\pm}0.0$ & $77.4{\pm}9.1$ & $72.7{\pm}20.9$ & $1.60{\pm}0.49$ \\
\quad Step Function & $99.7{\pm}1.6^{7}$ & $99.9{\pm}0.0$ & $22.6{\pm}8.7$ & $2.00{\pm}0.19$ & $100.0{\pm}0.0^{20}$ & $94.9{\pm}4.5$ & $42.5{\pm}13.1$ & $1.51{\pm}0.50$ & $99.5{\pm}7.1^{8}$ & $84.3{\pm}6.0$ & $83.9{\pm}26.2$ & $1.18{\pm}0.38$ & $91.3{\pm}22.5^{10}$ & $100.0{\pm}0.0$ & $84.4{\pm}6.8$ & $85.5{\pm}22.5$ & $1.02{\pm}0.12$ \\
\quad Adaptive & $98.1{\pm}9.6^{14}$ & $100.0{\pm}0.0$ & $48.5{\pm}18.5$ & $1.24{\pm}0.43$ & $100.0{\pm}0.0^{20}$ & $93.9{\pm}4.3$ & $33.3{\pm}12.3$ & $1.92{\pm}0.27$ & $100.0{\pm}0.0^{15}$ & $88.3{\pm}4.2$ & $92.6{\pm}30.4$ & $1.07{\pm}0.26$ & $95.7{\pm}15.6^{16}$ & $100.0{\pm}0.0$ & $83.2{\pm}5.9$ & $82.5{\pm}27.5$ & $1.13{\pm}0.34$ \\
\quad Uniform+E & $100.0{\pm}0.0^{11}$ & $100.0{\pm}0.0$ & $17.9{\pm}6.7$ & $2.81{\pm}0.42$ & $100.0{\pm}0.0^{20}$ & $93.2{\pm}3.9$ & $52.0{\pm}19.3$ & $1.79{\pm}0.41$ & $99.5{\pm}7.1^{15}$ & $94.4{\pm}3.9$ & $99.4{\pm}30.7$ & $1.04{\pm}0.20$ & $99.9{\pm}0.2^{13}$ & $100.0{\pm}0.0$ & $81.7{\pm}8.3$ & $65.1{\pm}18.3$ & $1.53{\pm}0.50$ \\
\quad Adaptive+E & $99.9{\pm}0.6^{12}$ & $99.9{\pm}0.1$ & $13.3{\pm}5.3$ & $3.08{\pm}0.45$ & $100.0{\pm}0.0^{20}$ & $90.4{\pm}5.1$ & $28.3{\pm}10.2$ & $1.98{\pm}0.16$ & $97.5{\pm}14.4^{17}$ & $93.2{\pm}4.9$ & $62.8{\pm}24.6$ & $1.53{\pm}0.50$ & $99.4{\pm}7.1^{20}$ & $100.0{\pm}0.0$ & $81.1{\pm}9.9$ & $45.3{\pm}15.1$ & $1.58{\pm}0.49$ \\
\cmidrule(l){1-18}
\quad Contrastive Vec. & $99.9{\pm}0.3^{17}$ & $85.1{\pm}34.2$ & $13.4{\pm}5.4$ & $3.23{\pm}0.68$ & $92.7{\pm}23.2^{13}$ & $89.1{\pm}9.8$ & $9.7{\pm}7.4$ & $2.06{\pm}0.74$ & $53.6{\pm}44.9^{14}$ & $93.0{\pm}6.6$ & $10.8{\pm}8.0$ & $1.98{\pm}0.64$ & $93.1{\pm}19.8^{18}$ & $63.0{\pm}44.7$ & $89.7{\pm}9.0$ & $12.8{\pm}9.3$ & $1.66{\pm}0.64$ \\
\quad Probe & $87.6{\pm}31.3^{20}$ & $59.2{\pm}45.5$ & $12.5{\pm}5.2$ & $3.17{\pm}0.74$ & $69.7{\pm}42.5^{12}$ & $93.5{\pm}9.5$ & $16.3{\pm}10.4$ & $2.73{\pm}0.87$ & $69.6{\pm}41.6^{14}$ & $94.1{\pm}7.4$ & $17.4{\pm}11.2$ & $2.29{\pm}0.84$ & $77.0{\pm}38.9^{16}$ & $56.7{\pm}43.5$ & $95.4{\pm}6.5$ & $16.4{\pm}10.9$ & $2.67{\pm}0.88$ \\
\quad PCA & $63.2{\pm}45.2^{20}$ & $44.3{\pm}45.8$ & $7.6{\pm}3.8$ & $2.82{\pm}0.65$ & $54.0{\pm}47.9^{4}$ & $27.1{\pm}21.2$ & $8.5{\pm}3.4$ & $3.94{\pm}0.66$ & $22.8{\pm}37.7^{18}$ & $32.7{\pm}26.0$ & $7.8{\pm}3.5$ & $3.37{\pm}0.68$ & $32.0{\pm}42.9^{20}$ & $43.3{\pm}43.8$ & $60.9{\pm}24.1$ & $62.9{\pm}46.2$ & $1.11{\pm}0.32$ \\
\quad Prompt & $80.3{\pm}37.1$ & $63.2{\pm}47.1$ & $8.5{\pm}3.6$ & $3.47{\pm}0.68$ & $59.6{\pm}46.2$ & $22.7{\pm}23.3$ & $7.9{\pm}3.8$ & $3.21{\pm}0.67$ & $62.2{\pm}46.6$ & $24.0{\pm}23.3$ & $9.3{\pm}3.4$ & $3.65{\pm}0.62$ & $76.2{\pm}40.3$ & $70.6{\pm}44.4$ & $26.7{\pm}25.1$ & $9.4{\pm}4.4$ & $3.30{\pm}0.70$ \\
\bottomrule
\end{tabular}
\end{adjustbox}
\end{table}